%% file: icml2026/main_ICML.tex
\documentclass{article}

\makeatletter
\def\input@path{{icml2026/}}
\makeatother

\usepackage{microtype}
\usepackage{graphicx}
\usepackage{subcaption}
\usepackage{booktabs} 

\usepackage{hyperref}

 \usepackage[accepted]{icml2026}

\usepackage{amsmath, amsfonts, amssymb, amsthm, bbm,tikz}
\usepackage{appendix} 
\usepackage{mathtools}
\usepackage{empheq}
\usepackage{xcolor}
\usepackage{enumitem}
\usepackage{natbib}

\DeclareFontFamily{U}{rsfs}{\skewchar\font127}
\DeclareFontShape{U}{rsfs}{m}{n}{<-6>rsfs5<6-8>rsfs7<8->rsfs10}{}
\DeclareSymbolFont{rsfs}{U}{rsfs}{m}{n}
\DeclareSymbolFontAlphabet{\mathscr}{rsfs}

\makeatletter
\renewcommand\paragraph{\@startsection{paragraph}{4}{\z@}%
  {0.65\baselineskip}%
  {-0.75em}%
  {\normalfont\normalsize\bfseries}}
\makeatother

\DeclareMathOperator*{\argmax}{argmax}
\DeclareMathOperator*{\argmin}{argmin}

\newcommand{\cD}{\mathscr{D}}

\newcommand{\R}{\mathbb{R}}
\newcommand{\norm}[1]{\left\lVert#1\right\rVert}
\newcommand{\inner}[2]{\left\langle #1,#2\right\rangle}
\newcommand{\ip}[2]{\left\langle #1,#2\right\rangle}
\newcommand{\Hf}{\mathcal{H}_f}
\newcommand{\Hg}{\mathcal{H}_g}
\newcommand{\Hs}{\mathcal{H}^\star}
\newcommand{\X}{\mathcal X}

\newcommand{\W}{\mathcal W}
\newcommand{\Z}{\mathcal Z}
\newcommand*\circled[1]{\tikz[baseline=(char.base)]{
            \node[shape=circle,draw,inner sep=2pt] (char) {#1};}}
\newcommand{\E}{\mathbb{E}}
\newcommand{\Var}{\mathrm{Var}}

\usepackage[capitalize,noabbrev]{cleveref}

\theoremstyle{plain}
\newtheorem{theorem}{Theorem}[section]

\newtheorem{lemma}[theorem]{Lemma}
\newtheorem{corollary}[theorem]{Corollary}
\theoremstyle{definition}

\newtheorem{assumption}[theorem]{Assumption}
\theoremstyle{remark}
\newtheorem{remark}[theorem]{Remark}

\usepackage{authblk}

\usepackage[textsize=tiny]{todonotes}

\icmltitlerunning{Coupled Training with Privileged Information and Unlabeled Data}

\begin{document}
\raggedbottom

\twocolumn[
  \icmltitle{Coupled Training with Privileged Information and Unlabeled Data}



  \icmlsetsymbol{equal}{*}

  \begin{icmlauthorlist}
    \icmlauthor{Jiahao Shi}{one}
    \icmlauthor{Omar Hagrass}{two}
    \icmlauthor{Jason M. Klusowski}{two}
  \end{icmlauthorlist}

  \icmlaffiliation{one}{Department of Electrical and Computer Engineering, Princeton University}
  \icmlaffiliation{two}{Department of Operations Research and Financial Engineering, Princeton University}
    \icmlcorrespondingauthor{Jiahao Shi}{js6028@princeton.edu}
  \icmlkeywords{Machine Learning, ICML}

  \vskip 0.3in
]



\printAffiliationsAndNotice{}  

\begin{abstract}
In many prediction problems, we have extra information during training (for example, measurements that are expensive or slow to collect) that will not be available when the model is deployed. A common strategy is to first train a model that uses all training information, then use its predictions on unlabeled examples to train a second model that only uses the inputs available at test time. However, when the extra training-only information is weak or noisy, this Two-Stage approach can mislead the deployment model and even hurt accuracy. We propose a joint training method that learns the two models together, so the deployment model can benefit from the extra information only when it actually helps, instead of inheriting its mistakes. We provide guarantees that describe when joint training improves prediction accuracy and analyze a simple alternating training algorithm for large, high-dimensional models. Experiments on synthetic data and real-world prediction tasks show that our approach avoids these failures and robustly outperforms standard Two-Stage baselines.
\end{abstract}

\section{Introduction}\label{sec:intro}

Many modern learning settings provide extra information during training that cannot be used at deployment. Alongside the primary features $X$ (e.g., an image or routine clinical measurements) and responses $Y$ (e.g., a continuous outcome such as disease severity or a lab value), training examples may also include privileged features $W$ (e.g., expensive biomarkers, specialist assessments). Because $W$ may be unavailable, costly, or impractical to obtain at test time, the deployed predictor must depend only on $X$. At the same time, a large fraction of the training data may contain $W$ but not the true labels. The goal is therefore to use $W$ to improve training while still producing a final predictor that uses only $X$. This setting is motivated by many applications, including medical domains where reliable labels are expensive \citep{Rajkomar2019MLMedicine, Fri+19}, forecasting and longitudinal studies where intermediate or future measurements are observed only during training \citep{Hogan2004Dropout}, as well as transfer learning \citep{Zhuang2020TransferLearningSurvey} and distribution shift settings \citep{Koh2021WILDS}.

\paragraph{Problem Formulation.} We formalize this setting as follows. Let $\mathcal Z = \mathcal X \times \mathcal W$
and define $Z=(X,W)\in\mathcal Z$. We observe $n$ labeled samples
\[
\cD_L = \{(Z_i, Y_i)\}_{i=1}^n \stackrel{\mathrm{i.i.d.}}{\sim} P_{X,W,Y},
\]
and $m \coloneqq N-n \gg n$ unlabeled samples
\[
\cD_U = \{Z_j\}_{j=n+1}^{N} \stackrel{\mathrm{i.i.d.}}{\sim} P_{X,W},
\]
where $P_{X,W,Y}$ is a distribution on $\mathcal X \times \mathcal W \times \mathbb R$
and $P_{X,W}$ is its marginal over $\mathcal Z$.

Our goal is to estimate a prediction rule that minimizes the expected risk
$$
\E [\ell\left(Y,f(X)\right)],
$$
where $\ell$ is a general loss function. When $\ell$ is the square loss, i.e., $\ell(y,y')=(y-y')^2$, the regression function $$
\mu(x)\coloneqq \E[Y\mid X=x],
$$ uniquely minimizes the squared risk
$
\E\big[(Y-f(X))^2\big]
$ over measurable functions $f:\mathcal X\to\mathbb R$, assuming $\E\big[Y^2\big]<\infty$.

Given this setup, the central question is whether we can use the privileged information $W$ to achieve better performance than training on $X$ alone, while still ensuring that the deployed predictor depends only on $X$.

\paragraph{Two-Stage Procedure.} One common way to leverage privileged data is through a Two-Stage procedure. In the first stage, we estimate a pseudo-response function $\hat g$, which we refer to as the \emph{rich-view model}, using all available labeled samples $\{(Z_i, Y_i)\}_{i=1}^n$ by minimizing the empirical risk
$$
\hat{g} \in \argmin_{g \in \mathcal{G}}\frac{1}{n} \sum_{i=1}^{n} \ell\left(Y_i,g(Z_i)\right),
$$
where $\mathcal{G}$ is a collection of functions defined on $\mathcal{Z}$.
The fitted function $\hat{g}$ is then used to transfer knowledge by imputing pseudo-responses $\hat{Y}_j = \hat{g}(Z_j)$ (or $\hat{Y}_j = \mathbf{1}(\hat{g}(Z_j) \geq 1/2)$ for 0/1 binary classification) for the unlabeled samples $\{Z_j\}_{j=n+1}^{N}$.  
In the second stage, we estimate the target function $\hat{f}$, which we refer to as the \emph{deployment model}, from the combined dataset $\{(X_i, Y_i)\}_{i=1}^n \cup \{(X_j, \hat{Y}_j)\}_{j=n+1}^{N}$ by minimizing
\begin{align*}
\hat{f} \in \argmin_{f \in \mathcal{F}}\frac{1}{N} \Bigg(
&\sum_{i=1}^{n} \ell\left(Y_i,f(X_i)\right) \\
&+ \sum_{j=n+1}^{N} \ell\left(\hat Y_j,f(X_j)\right)
\Bigg),
\end{align*}
where $\mathcal{F}$ is a collection of functions defined on $\mathcal{X}$.
This general strategy was studied in \citet{hou2023surrogate} for specific parametric models and extended to more general nonparametric settings in \citet{xia2024prediction}.

Although it is appealingly simple, this Two-Stage pseudo-labeling strategy can be fragile when the privileged features $W$ provide only weak or noisy information about the response. In such regimes, the first-stage estimator $\hat g$ may produce pseudo-responses that are poorly aligned with the true regression function $\mu$, and these errors are then propagated to the second stage through the imputed labels. As a consequence, the deployment model $\hat f$ may be biased toward artifacts introduced by $\hat g$, leading to degraded performance relative to training on the labeled data alone. This phenomenon, often referred to as \emph{negative transfer}, is especially pronounced when the predictive advantage of $(X,W)$ over $X$ is marginal, or when the privileged variables contain high-dimensional nuisance components that overwhelm their useful signal. In such settings, naive Two-Stage procedures can perform strictly worse than supervised learning using only the labeled samples, despite access to abundant unlabeled data \citep{xia2024prediction}. These limitations motivate approaches that more carefully control the influence of privileged information and adaptively balance the contribution of $W$ during training.

\paragraph{Our Contributions.} We summarize our main contributions as follows.
\begin{itemize}
\item \textbf{Coupled Training Framework.} We propose a unified estimation framework that jointly learns the deployment and rich-view models. By flexibly calibrating the influence of privileged information, our approach effectively mitigates the negative transfer often observed in Two-Stage pseudo-labeling methods.
\item \textbf{Efficient Algorithms.} We develop efficient implementations of the coupled training framework and extend it to high-dimensional settings using a greedy forward selection procedure with provable optimization guarantees.
\item \textbf{Theoretical and Empirical Analysis.} We provide theoretical insights into the conditions under which unlabeled data improves estimation and demonstrate, through synthetic experiments and real-world regression and classification benchmarks, that our method consistently outperforms standard baselines when privileged signals are noisy or weak.
\end{itemize}

\subsection{Pseudo-Response Calibration with Coupled Training} \label{sec:co-training alg}

The main limitation of the Two-Stage approach is that the pseudo-responses
produced by the rich-view model are treated as fixed targets in the second stage, with no
mechanism to correct or attenuate their influence when the privileged signal is weak or noisy.
To address this issue, we propose a coupled training procedure that jointly estimates the
deployment model $f$ and the rich-view model $g$ through the updates below.
The core idea is to allow information to flow in both directions, with $g$ providing
pseudo-responses that enrich the effective sample size for learning $f$, while the current estimate of
$f$ is in turn used to recalibrate $g$ on the unlabeled data, subject to an explicit constraint that limits deviation from the labeled responses. This coupling yields a flexible
regularization of the privileged information and provides a principled mechanism for avoiding
negative transfer. 

We now describe the resulting algorithm and its associated optimization
formulation. For simplicity, we focus on regression, but the same ideas can be adapted for classification, e.g., by treating $\hat Y$ as soft labels and using logistic loss.

\paragraph{Coupled Training Algorithm.}
Initialize with any feasible $g_0 \in \mathcal{G}$ and fix a constraint level $\nu \geq 0$. For $k=1,2,\dots$, alternate until convergence.

1. \textbf{Update $f$.}
Given the current pseudo-responses $\hat Y_j=g_{k-1}(Z_j)$, update $f$ by solving
\[
\begin{aligned}
f_k \in \argmin_{f \in \mathcal{F}} \frac{1}{N}\Bigg(
&\sum_{i=1}^n \ell\left(Y_i,f(X_i)\right)\\
&+
\sum_{j=n+1}^{N} \ell\left(g_{k-1}(Z_j),f(X_j)\right)
\Bigg).
\end{aligned}
\]

2. \textbf{Recalibrate $g$.}
Given the current deployment model $f_k$, update $g$ by solving
\[
g_k \in \argmin_{g \in \mathcal{G}}
\frac{1}{m}\sum_{j=n+1}^{N}\ell\left(g(Z_j),f_k(X_j)\right)
\]
\[
\text{s.t.}\quad
\frac{1}{n}\sum_{i=1}^n \ell\left(Y_i,g(Z_i)\right) \le \nu.
\]

If $\mathcal{F} \subset L^2(\mathcal{X})$ and $\mathcal{G} \subset L^2(\mathcal{Z})$ are convex function classes, the loss $\ell(y, y')$ is jointly convex in its arguments $(y, y')$ (e.g., square loss) and the constraint set is feasible, then the coupled training algorithm gives exact updates for the joint convex optimization problem.
\begin{align}
\label{eq:empirical}
\min_{(f, g) \in \mathcal{F}\times \mathcal{G}}
\;&
\frac{1}{N}
\Bigg(
\sum_{i=1}^{n} \ell\left(Y_i,f(X_i)\right)
\nonumber\\
&\qquad+
\sum_{j=n+1}^{N} \ell\left(g(Z_j),f(X_j)\right)
\Bigg) \nonumber\\
\text{s.t.}\quad
&
\frac{1}{n}\sum_{i=1}^{n}
\ell\left(Y_i,g(Z_i)\right) \le \nu .
\end{align}
In particular, the objective value is nonincreasing along the iterates. Under standard
regularity conditions for these exact updates, every cluster point is a
stationary point, and hence globally optimal by convexity for \eqref{eq:empirical}; see \citet{GRIPPO2000127}.

The implementation of the coupled training algorithm depends on the model class. For linear models with square loss, the penalized joint objective is quadratic in the coefficients of $(f,g)$, so its first-order conditions reduce to a single block linear system. Thus, we solve this system in closed form. For differentiable nonlinear models, such closed-form equations are unavailable, so we compute the same coupled training updates using gradient-based solvers. For tree ensembles, we use the penalized Lagrangian form in \eqref{eq:empirical_dual}, so each alternating update reduces to a standard weighted tree regression. For high-dimensional dictionary models, the updates are computed using the greedy selection routine in Section~\ref{sec:high dim}. See Appendix~\ref{app:additional_experiments} for implementation details. Throughout, ``joint'' training means that $f$ and $g$ are coupled through one objective, not that every experiment uses a single simultaneous optimizer over all parameters.

\begin{figure*}[t]
  \centering
  \begin{subfigure}[t]{0.48\textwidth}
    \centering
    \includegraphics[width=\linewidth]{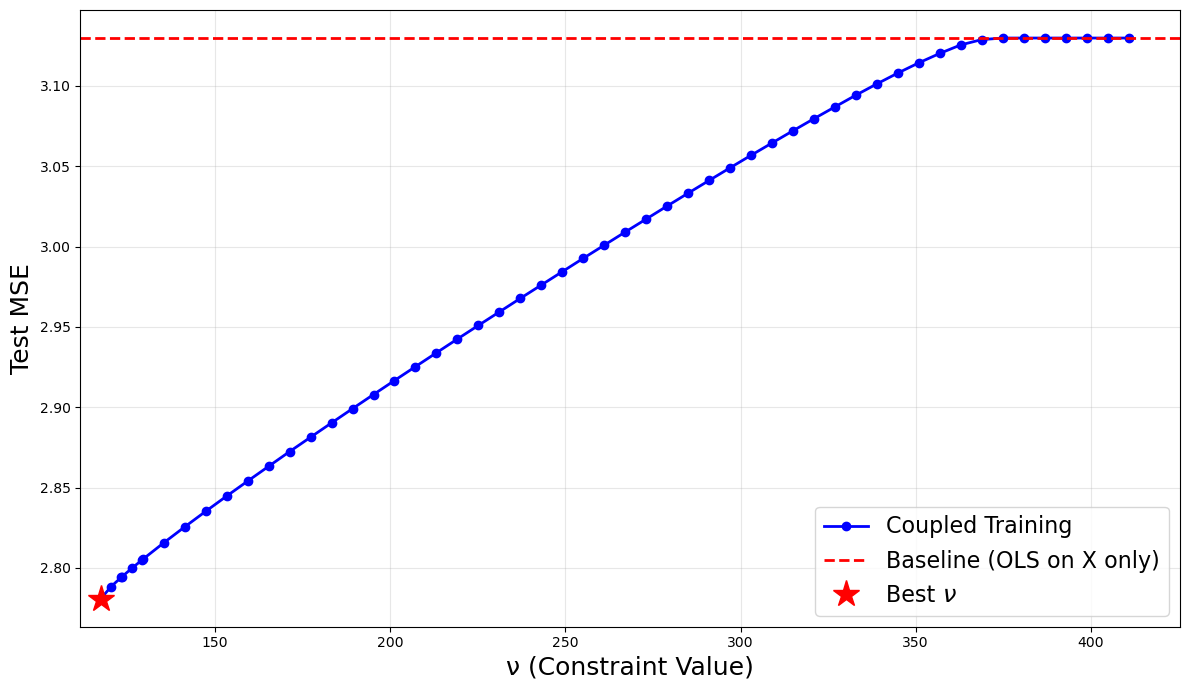}
    \caption{Strong privileged signal (large $\|\theta\|_2$).}
    \label{fig:linear_strong}
  \end{subfigure}
  \hfill
  \begin{subfigure}[t]{0.48\textwidth}
    \centering
    \includegraphics[width=\linewidth]{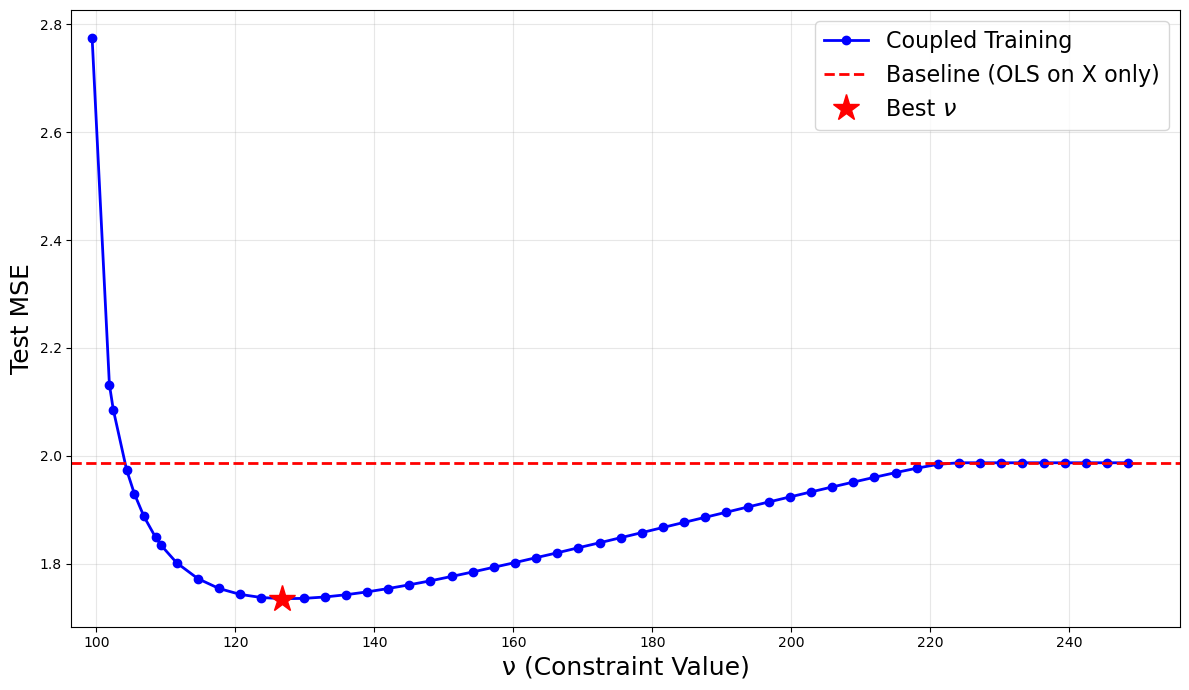}
    \caption{Weak privileged signal (small $\|\theta\|_2$).}
    \label{fig:linear_weak}
  \end{subfigure}
  \caption{\textbf{Linear Gaussian signal strength.} Performance of the proposed method under varying levels of privileged signal strength.}
  \label{fig:linear_signal_strength}
\end{figure*}


Henceforth, we restrict attention to the square loss $\ell(y,y')=(y-y')^2$, which allows for
a tractable analysis. Of course, the proposed coupled training algorithm is not restricted to this choice
of loss. 

We identify any $f:\mathcal{X}\to\mathbb{R}$ with its lift $\tilde f:\mathcal{Z}\to\mathbb{R}$,
defined by $\tilde f(x,w)=f(x)$; below, we simply write $f$ for this lift when no confusion is possible.

\paragraph{Negative Transfer.}

The constraint level $\nu$ governs how much the rich-view model $g$ is allowed to deviate from the labeled responses, and therefore how strongly privileged information can influence the deployment model $f$. In particular, feasibility requires
$$
\nu \geq \min_{g \in \mathcal{G}} \frac{1}{n}\sum_{i=1}^{n} (Y_i - g(Z_i))^2.
$$
If $\mathcal F \subseteq \mathcal G$ (identifying $f\in\mathcal F$ with its lift $\tilde f(x,w)=f(x)$), then choosing $\nu$ large enough that the labeled least-squares fit over $\mathcal F$ is feasible with $g=f$ makes the unlabeled agreement term vanish, and the solution reduces to the labeled least-squares fit over $\mathcal F$. In contrast, choosing $\nu$ close to the minimum labeled risk over $\mathcal G$ forces $g$ to be (approximately) the labeled least-squares fit over $\mathcal G$, yielding behavior akin to the aforementioned Two-Stage pseudo-labeling procedure. Thus, choosing $\nu$ too small can overemphasize noisy pseudo-responses and cause negative transfer, while choosing larger $\nu$ attenuates the influence of $g$ and interpolates back toward supervised learning that ignores $W$.

To illustrate this behavior, we consider a linear model
\[
Y=\beta^\top X+\theta^\top W+\varepsilon,
\]
where $(X,W)$ are jointly Gaussian and $\varepsilon$ is independent mean-zero Gaussian noise.
The strength of the privileged feature is governed by $\|\theta\|_2$, while all other aspects
of the data-generating process are held fixed.

Figure~\ref{fig:linear_signal_strength} illustrates the performance of the proposed method under varying levels of privileged signal strength. In panel~(\subref{fig:linear_strong}), where $\|\theta\|_2$ is large, the privileged variables are highly informative, and the optimal performance is achieved for small values of $\nu$. In this regime, the method behaves like a Two-Stage procedure, leveraging the strong privileged signal to improve estimation accuracy.

In contrast, panel~(\subref{fig:linear_weak}) corresponds to a setting in which the privileged signal is weak. Here, aggressively incorporating pseudo-responses can degrade performance, yielding worse error than the least-squares estimator (OLS) based only on the labeled pairs $\{(X_i,Y_i)\}_{i=1}^n$.
This phenomenon of negative transfer was also noted in \citet[Section~2.4.2]{xia2024prediction} as a limitation of the Two-Stage approach.

Together, these results show that the proposed coupled training framework adaptively moderates the influence of the rich-view model, interpolating between regimes in which privileged information helps or hurts prediction, and thereby yields a robust, principled approach to leveraging privileged information.

\subsection{Related Work}

We place our coupled training framework in context by reviewing connections to learning using privileged information, semi-supervised learning (especially pseudo-labeling), multi-view agreement methods, and related optimization approaches.

\paragraph{Learning Using Privileged Information.} Our setup is closely related to Learning Using Privileged Information (LUPI) \cite{VAPNIK2009544,vapnik2015learning}, where privileged variables $W$ are available during training but the goal is to learn a predictor that depends only on $X$. However, in contrast to our framework, classical LUPI formulations assume fully labeled data and do not leverage unlabeled samples. 

\paragraph{Semi-Supervised Learning.} Learning from both labeled and unlabeled data has been widely studied across statistics and machine learning under the name semi-supervised learning (SSL) \citep{CSZ09}.

Among SSL methods, pseudo-labeling is a particularly prominent approach. In classification, pseudo-labeling was explored in the context of support vector machines by \citet{Joa+99} and later extended to deep learning frameworks by \citet{Lee+13}. In regression, pseudo-responses were constructed via kernel smoothing and imputation in linear models \citep{CC18}, extended to kernel ridge regression \citep{Wan23}, and studied in general settings by \citet{xia2024prediction}. From an inferential perspective, mean estimation with SSL data was investigated by \citet{ZBC19,ZB21}, while \citet{LA17} proposed an ensemble-based method that iteratively assigns pseudo-responses using historical predictions. A parallel line of work is weakly supervised learning, which relies on weak or noisy labels rather than ground truth. The data programming paradigm of \citet{Rat+16,Rat+17} assumes access to labeling functions and was applied to MRI analysis by \citet{Fri+19}. Weak observations were also used for data augmentation by \citet{RJS20}, who established theoretical guarantees under universal central conditions. Closely related is the literature on learning with noisy labels \citep{Nat+13, Son+22}.

In contrast to these approaches, which primarily rely on pseudo-responses or weak supervision, our framework explicitly controls the influence of privileged information within an adaptive self-training scheme.

\paragraph{Multi-View Learning.} The objective in \eqref{eq:empirical} is related to the \emph{co-regularized least-squares} objective of \citet{sindhwani2005coregularization}, which encourages predictors built from different views to agree on unlabeled data. In their multi-view SSL setting, one observes labeled examples $\{(X_i,Y_i)\}_{i=1}^n$ and unlabeled examples $\{X_i\}_{i=n+1}^{N}$, where each input $X_i=(X_i^{(1)},X_i^{(2)})$ comes with two complementary representations $X_i^{(1)}\in\mathcal{X}^{(1)}$ and $X_i^{(2)}\in\mathcal{X}^{(2)}$. The goal is to learn $h=(h^{(1)},h^{(2)})$ so that each component fits the labeled data in its own view, while their predictions are pushed to match on the unlabeled inputs (see \citet[Section~3]{sindhwani2005coregularization}). Related agreement-style objectives also appear in \citet{Ding22}, though their setting is fully supervised and does not use unlabeled data.

In our case, \eqref{eq:empirical} uses a similar kind of agreement term, but for a different reason. We work in a \emph{privileged information} setting where $W$ is available during training but not at deployment, so only $f(X)$ can be used at test time. That makes the coupling inherently asymmetric, with $f$ as the deployment model and $g$ as an auxiliary rich-view model used to pass information to $f$ during training. We also take $g$ to act on the joint input $Z=(X,W)$, rather than on $W$ alone. This lets $g$ capture effects of the privileged variables that depend on $X$ (including $X$--$W$ interactions) and helps translate whatever signal is present in $W$ into improvements in $f$.

\paragraph{Greedy and Dictionary-Based Optimization.}
In Section~\ref{sec:high dim}, we restrict the deployment and rich-view models to
dictionary spans and fit them using an \emph{alternating forward selection} procedure. This connects our method to the literature on greedy approximation
in Hilbert spaces, where algorithms
construct an additive predictor by sequentially selecting atoms that best reduce a squared-loss
objective and enjoy deterministic approximation
guarantees with sublinear objective or approximation error decay \citep{barron2008approximation, DeVoreT96}. It can also be viewed as a block-coordinate variant of classical forward (stepwise) selection in regression \citep[Sec.~3.3.2]{hastie2009esl}.

Our setting differs in that we perform greedy atom selection within each block
while alternating between updating $f$ with $g$ fixed and updating $g$ with $f$ fixed, so the
greedy search is carried out in a product space with interdependent residuals.
Although support selection is combinatorial, we prove that the alternating greedy iterates achieve
a global sublinear decay of the empirical coupled objective (Theorem~\ref{thm:oga-main-1over5}),
extending classical greedy approximation analyses to the coupled privileged information setting.
This is complementary to blockwise optimization analyses for nonconvex problems such as
dictionary learning, which typically focus on reconstruction or coefficient recovery under
structural assumptions \citep{chatterji2017alternating, agarwal2013learning, ruetz2023convergence};
here the dictionaries are fixed and our guarantee controls the empirical comparison term through
an optimization error bound.

\subsection{Notation}

Let $\mathcal{F} \subset L^2(\mathcal{X})$ and $\mathcal{G} \subset L^2(\mathcal{Z})$ be closed, convex function classes. For a nonempty closed convex set $\mathcal{C}$ in a Hilbert space $(\mathcal H,\langle\cdot,\cdot\rangle_{\mathcal{H}})$, let
$\Pi_{\mathcal{C}} h \in \arg\min_{c\in {\mathcal{C}}}\|h-c\|_{\mathcal H}^2$
denote the (metric) projection. When $\mathcal{C}$ is a closed linear subspace, $\Pi_{\mathcal{C}}$ coincides with the orthogonal projector.


 For fixed points $x^N=(x_1,\dots,x_N) \in \mathcal X^N$ and
$z^N=(z_1,\dots,z_N) \in \mathcal Z^N$, where $z_i = (x_i, w_i)$, let $\hat P_N^X$ and
$\hat P_N^Z$ denote the corresponding normalized empirical measures.
For functions $f:\mathcal X\to\mathbb R$ and $g:\mathcal Z\to\mathbb R$, we write
$
\|f\|_{L^p(\hat P_N^X)}
\coloneqq
\left(\frac{1}{N}\sum_{i=1}^N |f(x_i)|^p\right)^{1/p}$, and 
$
\|g\|_{L^p(\hat P_N^Z)} \coloneqq \left(\frac{1}{N}\sum_{i=1}^N |g(z_i)|^p
\right)^{1/p}$. 
For a class $\mathcal F\times\mathcal G$, define the max norm
$
\|(f,g)\|_{1,\infty}
=
\max\{\|f\|_{L^1(\hat P_N^X)},\,\|g\|_{L^1(\hat P_N^Z)}\}.
$
An $\epsilon$-cover of $\mathcal F\times\mathcal G$ with respect to
$\|\cdot\|_{1,\infty}$ is a finite collection $\{(f_1,g_1),\dots,(f_M,g_M)\}
\subset \mathcal F\times\mathcal G$ such that every $(f,g)\in\mathcal F\times\mathcal G$
is within $\epsilon$ of some $(f_j,g_j)$ under this norm, and the minimal such
$M$ is denoted by $\mathcal N_{1, \infty}(\epsilon,\mathcal F\times\mathcal G, z^N)$.

\section{Risk Bounds}

\paragraph{Lagrangian Relaxation.} The constrained problem \eqref{eq:empirical} admits a Lagrangian relaxation, leading to the penalized objective
\begin{align}\label{eq:empirical_dual}
\hat{\mathcal{L}}(f,g;\lambda)
& \coloneqq \frac{1}{N} \Bigg(
\sum_{i\le n} (Y_i - f(X_i))^2 \nonumber\\
&
+ \sum_{j>n}\!(g(Z_j) - f(X_j))^2\!
+\!\lambda\!\sum_{i\le n}\!(Y_i - g(Z_i))^2
\Bigg).
\end{align}

Under the usual convex duality conditions for \eqref{eq:empirical}, for example
feasibility together with a suitable constraint qualification, the constrained problem admits
KKT multipliers. In that case, for a primal solution $(\hat f,\hat g)$ there exists a multiplier
$\tilde\lambda\ge0$ such that $(\hat f,\hat g)$ minimizes the penalized objective
\eqref{eq:empirical_dual} with
$\lambda=(N/n)\tilde\lambda$ and satisfies the complementary-slackness condition
\[
\tilde\lambda
\left(
\frac{1}{n}\sum_{i=1}^n (Y_i-\hat g(Z_i))^2-\nu
\right)=0.
\]
In the rest of the paper we treat the rescaled parameter $\lambda$ as a tuning parameter
and work directly with the penalized formulation.

Note that $\lambda$ plays a role analogous to $\nu$, but in the opposite direction. Smaller $\lambda$ relaxes the effective constraint on $g$ and attenuates the influence of pseudo-responses, while larger $\lambda$ forces $g$ to fit the labeled data more closely and increases the impact of pseudo-labeling. In particular, if $\mathcal{F} \subseteq \mathcal{G}$, then as $\lambda \to 0$, the unlabeled data play no role and the solution reduces to the ordinary least-squares fit based only on the labeled data. As $\lambda \to \infty$, the solution approaches the Two-Stage procedure. In this case, since $\hat{\mathcal{L}}(\hat f, \hat g;\lambda) \leq \hat{\mathcal{L}}(\hat f, \hat f;\lambda)$, we also have the bound $\frac{1}{m}\sum_{j=n+1}^N(\hat Y_j-\hat f(X_j))^2 \leq \frac{\lambda n}{m}\cdot \frac{1}{n}\sum_{i=1}^n(Y_i-\hat f(X_i))^2$, meaning that the (pseudo-) training error of $\hat f$ on the unlabeled features is at most $\lambda n/m$ times the training error of $\hat f$ on the labeled features.

\paragraph{Population Minimizers.} We now turn to a theoretical analysis of the proposed method and derive bounds on the prediction error of the learned deployment model $\hat f$. To this end, we consider the population-level counterpart of the empirical objective in \eqref{eq:empirical_dual}, given by
\begin{align}
\mathcal{L}(f,g;\lambda)
& \coloneqq \frac{n}{N}\E[(Y-f(X))^2] \nonumber \\
&\qquad
+\frac{ m }{N}\E[(g(Z)-f(X))^2] \nonumber \\
&\qquad
+\lambda\frac{n}{N}\E[(Y-g(Z))^2].
\label{eq:pop-loss}
\end{align}
Recall we identify any $f:\mathcal X\to\mathbb R$ with its lift to $\mathcal Z$ via $\tilde f(x,w)=f(x)$. We define the rich-view regression function
$$
\eta(z)\coloneqq \E\big[Y\mid Z=z\big],
$$
which is related to $\mu$ via $\mu(x) = \mathbb{E}[\eta(Z) \mid X=x]$.

\begin{theorem}[Population Minimizers] \label{thm:pop-minimizer}
Any $(f^\star,g^\star)\in \mathcal{F}\times \mathcal{G}$ minimizing \eqref{eq:pop-loss} satisfies the fixed point equations.
\begin{align*}
f^\star&=\Pi_{\mathcal{F}}\left(\frac{n}{N} \mu+\frac{m}{N}\E[g^\star(Z)\mid X]\right), \\
g^\star&=\Pi_{\mathcal{G}}\left(\frac{m}{m+n\lambda} f^\star+\frac{n\lambda}{m+n\lambda}\eta\right).
\end{align*}

Furthermore, if
$\mu \in \mathcal{F}$, $\mu \in \mathcal{G}$, and  $\eta \in \mathcal{G}$, then 
$$
f^\star = \mu, \ \ \text{and}
\quad 
g^\star = \frac{m}{m+n\lambda}\mu + \frac{n\lambda}{m+n\lambda}\eta.
$$
\end{theorem}

\begin{remark}[Interpolation Effect]
The characterization in Theorem~\ref{thm:pop-minimizer} highlights how the rich-view model $g^\star$ interpolates between the deployment target $\mu$ and the rich-view regression function $\eta$. In particular, under the realizability assumptions in Theorem~\ref{thm:pop-minimizer}, as $\lambda \to 0$, the penalty on the rich-view model becomes inactive and the solution satisfies $g^\star \to \mu$, corresponding to a regime in which the privileged information is ignored. On the other hand, as $\lambda$ increases, $g^\star$ places increasing weight on $\eta$, recovering behavior akin to a Two-Stage pseudo-labeling procedure. Intermediate values of $\lambda$ yield an interpolation between these extremes, allowing the method to adapt to the informativeness of the privileged variables.
\end{remark}

\paragraph{Correlation Controlled Risk Bound.} We now study the statistical behavior of the proposed estimator by analyzing its prediction error. Let $(\hat f,\hat g)\in\mathcal{F}\times\mathcal{G}$ denote the solution (or an approximate solution) of the empirical problem \eqref{eq:empirical_dual} and let $(f^\star,g^\star)\in \mathcal{F}\times \mathcal{G}$ denote the minimizer of \eqref{eq:pop-loss}. Our goal is to control the excess risk
\[
\E_{\cD,\,X}\left[(\hat f(X)-\mu(X))^2\right],
\]
where the expectation is over the training data $\cD$ used to construct $\hat f$ and an independent test point $X$.

A key quantity in our analysis is a correlation coefficient that captures the
interaction between the estimation errors of the deployment and rich-view models.
Specifically, define the estimation errors
\[
\hat e_f(X) \coloneqq f^\star(X)- \hat f(X) ,
\qquad
\hat e_g(Z) \coloneqq   g^\star(Z)-\hat g(Z).
\]
We then define
\begin{equation*}
\rho_\star
=
\frac{
    \big|\mathbb{E}_{\cD,\,Z}\left[\hat e_f(X)\,\hat e_g(Z)\right]\big|
}{
    \big(\mathbb{E}_{\cD,\,X}\big[\hat e_f^2(X)\big]\big)^{1/2}
    \big(\mathbb{E}_{\cD,\,Z}\left[\hat e_g^2(Z)\right]\big)^{1/2}
} \in [0, 1],
\end{equation*}
where $Z = (X,W)$,
which measures the alignment between the residuals of the two estimators, with the convention that $\rho_\star=0$ if either denominator term is zero.

The following corollary, which is a direct consequence of the excess loss
decomposition in Theorem~\ref{thm:excess-fgstar}, shows that $\rho_\star$
plays a central role in controlling the prediction error.

\begin{corollary}[Correlation Controlled Risk Bound]\label{cor:gamma-star}
If $\mu \in \mathcal{F}$, $\mu \in \mathcal{G}$ (note that functions on $\mathcal X$ are implicitly lifted to $\mathcal Z$), and  $\eta \in \mathcal{G}$, then any estimator $(\hat f, \hat g) \in \mathcal{F}\times \mathcal{G}$ satisfies
\begin{align*}
 &  \E_{\cD,\,X}\left[(\hat f(X)-\mu(X))^2\right] \nonumber\\ & \qquad\qquad \leq
\frac{
\mathbb{E}_{\cD}\left[\mathcal L(\hat f,\hat g;\lambda)
-
\mathcal L(f^\star,g^\star;\lambda)\right]
}{
\gamma_{n,m,\lambda}(\rho_\star)
},
\end{align*}
where
\[
\gamma_{n,m,\lambda}(\rho_\star)
\coloneqq
1-\frac{m^2}{N(m+n\lambda)}\,\rho_\star^2 \geq \frac{n}{N}.
\]
\end{corollary}

The above bound shows that excess loss in the joint objective translates
into prediction error for $\hat f$, with a constant that degrades as the correlation
$\rho_\star$ between the estimation errors increases.

\begin{remark}[Correlation Score Interpretation] \label{rem: correlation score}
(i) The score $\rho_\star$ measures the alignment between the deployment model and rich-view estimation errors; small $\rho_\star$ yields a sharper conversion from coupled excess loss to deployment risk and corresponds to privileged information contributing variation not already captured by $X$, whereas large $\rho_\star$ limits the gain from incorporating $W$.
In this regime, leveraging the privileged information can substantially improve the accuracy of $\hat{f}$. 
In contrast, when $\rho_\star$ is large, the residuals of $\hat{f}$ and $\hat{g}$ are highly correlated, suggesting that $W$ provides little additional information beyond what is already contained in $X$, and the benefit of incorporating privileged data becomes limited. \\
(ii) By Cauchy--Schwarz and the law of total variance, it can be shown (see
Section~\ref{proof: rem: correlation score} for details) that
\[
\rho_\star^2
\le
\frac{
\mathbb{E}_{\cD,\,X}\left[
\big(\mathbb{E}[\hat e_g(Z)\mid \cD,\,X]\big)^2
\right]
}{
\mathbb{E}_{\cD,\,Z}\left[
\hat e_g^2(Z)
\right]
} \leq 1.
\]
Furthermore, if $\mu \in \mathcal{F}$, $\mu \in \mathcal{G}$, and  $\eta \in \mathcal{G}$, then
$\mathbb{E}[g^\star(Z)\mid X]=\mu(X)$, and hence
\[
\mathbb{E}[\hat e_g(Z)\mid \cD,\,X]
=
\mu(X)-\mathbb{E}[\hat g(Z)\mid \cD,\,X].
\]
Consequently, $\rho_\star$ is small when the conditional rich-view model error
$
\mathbb{E}_{\cD,\,X}\big[
\big(\mathbb{E}[\hat e_g(Z)\mid \cD,\,X]\big)^2
\big]
$
is small relative to the total rich-view model error
$
\mathbb{E}_{\cD,\,Z}\left[
\hat e_g^2(Z)
\right].
$
This corresponds to a regime in which the rich-view model $\hat g$
leverages privileged information $W$ to capture predictive signal about $Y$ that is not
well explained by $X$ alone, thereby reducing the component of the rich-view error that
is aligned with $X$. Hence, $\rho_\star$, and consequently the risk bound, are controlled by a
\emph{multiplicative} relative error, in contrast to \citet{xia2024prediction}, where the risk is
bounded in terms of an \emph{additive} absolute error
$
\mathbb{E}_{\cD,\,X}\big[
\big(\mathbb{E}[\hat e_g(Z)\mid \cD, X]\big)^2
\big].
$
This means that the risk bound degrades more gracefully as $\hat g$ worsens, because it depends on the fraction of rich-view error explainable from $X$ rather than on its absolute size.
\end{remark}

\subsection{Main Risk Bound}
\begin{assumption}
The function classes $\mathcal F$ and $\mathcal G$ are uniformly bounded.
Specifically, without loss of generality, there exists a constant $B\geq 1$ such that
$$
\sup_{(f,g)\in\mathcal F\times\mathcal G}\|(f,g)\|_{\infty}\le B\quad \text{and} \quad
|Y| \le B, \quad \text{a.s.},
$$
where $\|(f,g)\|_{\infty} \coloneqq \max\big\{\|f\|_\infty,\|g\|_\infty\big\}$. \label{assump:boundedness}
\end{assumption}

 \begin{remark}[On the boundedness assumption]
The condition $\sup_{(f,g)\in\mathcal F\times\mathcal G}\|(f,g)\|_{\infty}\le B$ is made for convenience. When $\mathcal F$ and $\mathcal G$ are unbounded but $|Y|\le B$ a.s., the same conclusions apply to the truncated estimators $T_B\hat f$ and $T_B\hat g$, with
$T_B(t) \coloneqq \operatorname{sign}(t)\,(|t|\wedge B)$
following the truncation argument of \citet{barron2008approximation}. Since $|Y|\le B$ a.s., truncation can only decrease the (squared) loss, so no rate is lost.
\end{remark}

We introduce a complexity term that captures the joint richness of the
deployment and rich-view function classes, at scales determined by the sample sizes and the regularization parameter $\lambda$. Define
\begin{align*}
    \mathfrak{E}_{\mathcal{F} \times \mathcal{G}}^{N,\lambda}
\coloneqq 1+\sup_{z^N}\log\Big(
\mathcal{N}_{1, \infty}\big(
\epsilon_N,
\mathcal{F} \times \mathcal{G},
z^N
\big)
\Big),
\end{align*}
with  $\epsilon_N=\frac{1}{1280B(\lambda+1)N}$.

\begin{theorem}[Risk Bound]\label{thm:rate-fstar}
If $\mu \in \mathcal{F}$, $\mu \in \mathcal{G}$,  $\eta \in \mathcal{G}$ and Assumption \ref{assump:boundedness} holds, then any estimator $(\hat f, \hat g) \in \mathcal{F}\times \mathcal{G}$ satisfies
\begin{equation}
\begin{aligned}
&\E_{\cD,\,X}\left[(\hat f(X)-\mu(X))^2\right] \\
&\qquad \leq
\frac{1}{\gamma_{n, m ,\lambda}(\rho_\star)}
\Bigg(
\frac{C_1B^2(\lambda+1)\mathfrak{E}_{\mathcal{F} \times \mathcal{G}}^{N,\lambda}}{N}\\
&\qquad\qquad
+ 2\E_{\cD}\big[\hat\Delta(\hat f, \hat g)\big]
\Bigg),
\end{aligned}
\label{eq:estimation-rate-fstar}
\end{equation}
where $C_1>0$ is a universal constant, and 
$$
\hat \Delta(\hat f, \hat g) \coloneqq \hat{\mathcal{L}}(\hat f,\hat g;\lambda)-\hat{ \mathcal{L}}(f^\star,g^\star; \lambda).
$$
\end{theorem}

\begin{remark}[Empirical Comparison Term]
The term
\(\hat \Delta(\hat f,\hat g)\)
is an empirical comparison term, comparing the empirical coupled objective attained by the estimator $(\hat f,\hat g)$ with that of the population reference pair $(f^\star,g^\star)$.
It is not necessarily nonnegative, since $(f^\star,g^\star)$ need not minimize the empirical objective. In particular, if $(\hat f,\hat g)$ is an exact empirical minimizer of \eqref{eq:empirical_dual}, then
\[
\hat \Delta(\hat f,\hat g)\le 0,
\]
so this term can be dropped from \eqref{eq:estimation-rate-fstar}. For algorithmically constructed estimators, $\hat \Delta(\hat f,\hat g)$ records how well the computed pair competes with $(f^\star,g^\star)$ on the empirical objective; for the alternating forward selection iterates, this comparison is controlled by Theorem~\ref{thm:oga-main-1over5}.
\end{remark}

\begin{remark}[Effect of Correlation Score]
(i) Note that, for any values of $\rho_\star$ and $\lambda$, the quantity
$\gamma_{n,m,\lambda}(\rho_\star)$ satisfies the uniform lower bound
$
\gamma_{n,m,\lambda}(\rho_\star)\ge n/N.
$
As a consequence,
\[
\frac{1}{\gamma_{n,m,\lambda}(\rho_\star)}\cdot\frac{1}{N}
\le \frac{1}{n},
\]
so the $O(1/N)$ term in \eqref{eq:estimation-rate-fstar} is never worse than the
$O(1/n)$ dependence obtained from using only the labeled data.

(ii) When the estimation errors of the deployment and rich-view estimators are highly correlated (i.e., $\rho_\star$ is close to $1$), one has
$
\gamma_{n,m,\lambda}(\rho_\star)
\approx
1-m^2/(N(m+n\lambda))
$.
In particular, when $\lambda = 0$, this becomes $\gamma_{n,m,0}(\rho_\star)\approx n/N$, and hence
\[
\frac{1}{\gamma_{n,m,0}(\rho_\star)}\cdot\frac{1}{N}
\approx
\frac{1}{n},
\]
so the $O(1/N)$ term in \eqref{eq:estimation-rate-fstar} recovers the $O(1/n)$ dependence, obtained when learning solely from $n$ labeled observations. 

(iii)
More generally, $\gamma_{n,m,\lambda}(\rho_\star)$ controls the multiplicative constant in
\eqref{eq:estimation-rate-fstar} through its dependence on $\rho_\star$.
When the deployment and rich-view estimation errors are weakly aligned (small $\rho_\star$),
$\gamma_{n,m,\lambda}(\rho_\star)$ remains close to one, so the $O(1/N)$ term is not significantly inflated.

(iv)
The quantity $\rho_\star$ is introduced only for the analysis, not observed from data, and not used to tune the method. Rather, it summarizes how much of the rich-view error remains explainable from $X$ alone. Constructing a practical proxy for $\rho_\star$ that could guide the choice of $\lambda$ is an interesting direction for future work.

\end{remark}

\section{High-Dimensional Extension}\label{sec:high dim}

In Section~\ref{sec:co-training alg}, we proposed a coupled training algorithm for
jointly estimating the deployment model $\hat f$ and the rich-view model
$\hat g$. In practice, however, directly optimizing over arbitrary functions in
infinite-dimensional Hilbert spaces is computationally inefficient, especially when the
combined inputs $(X,W)$ are high-dimensional. 

To bridge this gap between theory and practice, we adopt a dictionary-based approximation
strategy in which functions are represented as sparse linear combinations of elementary
functions, or \emph{atoms}, drawn from predefined dictionaries. The resulting procedure is efficient and remains theoretically well-justified. In this
section, we introduce a greedy alternating forward selection procedure that incrementally
constructs approximations of $(\hat f,\hat g)$ using finite dictionaries.

Let $\mathcal{D}_f$ and $\mathcal{D}_g$ be finite, symmetric dictionaries such that $\|\psi\|_{L^2(\hat P_N^X)} \leq 1$ for all $\psi \in \mathcal{D}_f$ and $\|\phi\|_{L^2(\hat P_N^Z)} \leq 1$ for all $\phi \in \mathcal{D}_g$.
 For functions
$f \in \mathcal{F} = \mathrm{span}(\mathcal{D}_f)$ and $g \in \mathcal{G} = \mathrm{span}(\mathcal{D}_g)$ of the form
\[
f = \sum_{\psi \in \mathcal{D}_f} c_\psi \psi,
\qquad
g = \sum_{\phi \in \mathcal{D}_g} c_\phi \phi,
\]
we define the associated atomic norms by
\[
\|f\|_{L^1(\mathcal{D}_f)}
\coloneqq
\inf\Bigg\{
\sum_{\psi \in \mathcal{D}_f} |c_\psi|
:\;
f = \sum_{\psi \in \mathcal{D}_f} c_\psi \psi
\Bigg\},
\]
\[
\|g\|_{L^1(\mathcal{D}_g)}
\coloneqq
\inf\Bigg\{
\sum_{\phi \in \mathcal{D}_g} |c_\phi|
:\;
g = \sum_{\phi \in \mathcal{D}_g} c_\phi \phi
\Bigg\}.
\]

\subsection{Alternating Forward Selection Algorithm}
We now describe a greedy algorithm that alternates between updating $f$ and recalibrating
$g$. 

The procedure approximately minimizes $\hat{\mathcal{L}}(f,g;\lambda)$ in~\eqref{eq:empirical_dual} by alternately updating $f$ and $g$ using forward (stepwise) selection \citep[Sec.~3.3.2]{hastie2009esl}.

\paragraph{Alternating Forward Selection Algorithm.}
Initialize $g_0 = 0$. For $k=1,2,\dots$, alternate between the following updates.

1. \textbf{Update $f$.}
Given the current pseudo-responses $\hat Y_j=g_{k-1}(Z_j)$, update $f$ by solving
\begin{align*}
(\psi_k,f_k)\in
\argmin_{\substack{\psi\in \mathcal D_f\\
f\in \mathrm{span}(\psi_1,\dots,\psi_{k-1},\psi)}}
\frac{1}{N}
\Bigg(
\sum_{i=1}^n (Y_i - f(X_i))^2 \\
\qquad +
\sum_{j=n+1}^{N} (g_{k-1}(Z_j)-f(X_j))^2
\Bigg).
\end{align*}

2. \textbf{Recalibrate $g$.}
Given the current deployment model $f_k$, update $g$ by solving
\begin{align*}    
(\phi_k,g_k)\in
\argmin_{\substack{\phi\in \mathcal D_g\\
g\in \mathrm{span}\{\phi_1,\dots,\\
\phi_{k-1},\phi\}}}&
\frac{1}{N} \Bigg(
         \sum_{j=n+1}^{N}
         (g(Z_j) - f_k(X_j))^2 \\
        &\qquad
        + \lambda \sum_{i=1}^{n}
        (Y_i - g(Z_i))^2
    \Bigg).
\end{align*}
An efficient implementation of this procedure is provided in
Appendix~\ref{app:efficient_impl} (see Algorithm~\ref{alg}).
We first establish a convergence rate, then summarize the implementation.

\begin{theorem}[Alternating Forward Selection Convergence Rate]\label{thm:oga-main-1over5}
For all $k\ge1$, and all $f \in \mathrm{span}(\mathcal{D}_f)$ and $g \in \mathrm{span}(\mathcal{D}_g)$, the iterates $(\hat f,\hat g)=(f_k,g_k)$ produced by the Alternating Forward Selection Algorithm satisfy
\begin{align*}
& \hat{\mathcal L}(\hat f, \hat g;\lambda)
\leq \hat{\mathcal L}(f, g;\lambda)
+ C_2\frac{\max\{1,\lambda\}\log (k+1)}{k} \\
&\quad
\times
\big(\|f\|_{L^1(\mathcal{D}_f)}
+\|g\|_{L^1(\mathcal{D}_g)}\big)^2,
\end{align*}
for a universal constant $C_2>0$.
\end{theorem}

This result shows that the greedy alternating procedure gives sublinear control of
the empirical comparison term through its optimization error contribution. After
$k$ steps, each component lies in a span generated by at most $k$ selected atoms.
Thus, the iteration count records both the achieved optimization accuracy and the
complexity of the selected representation along the greedy path.


\subsection{Efficient Implementation} \label{sec: algorithm implementation}

We briefly describe an efficient implementation of the Alternating Forward Selection Algorithm; full derivations and pseudocode are deferred to Appendix~\ref{app:efficient_impl}.
At a high level, each block update (the $f$-step over $\mathrm{span}(\mathcal D_f)$ and the $g$-step over $\mathrm{span}(\mathcal D_g)$) is a least-squares refit over the currently selected atoms.
Rather than resolving these refits from scratch, we maintain incremental QR decompositions for the (weighted) design matrices associated with the selected atoms, enabling $O(Nk)$ updates per iteration, where $k$ is the number of selected atoms per block.

\paragraph{Time complexity.}
Assuming each candidate atom can be evaluated on all $N$ samples in $O(N)$ time, selecting the next atom at iteration $k$ amounts to scanning the remaining candidates in the relevant dictionary and computing their correlations with the current residual after orthogonalization against the selected span (via the maintained QR factors).
This yields a total cost up to step $k$ of
\(
O\big(N (|\mathcal D_f|+|\mathcal D_g|) k\big).
\)
For comparison, a naive version of forward selection over paired atoms $(\psi, \phi)$ (i.e., scanning a product dictionary) scales as
\(
O\big(N |\mathcal D_f||\mathcal D_g| k\big),
\)
which is substantially larger. As shown in Theorem \ref{thm:oga-main-1over5}, this computational improvement comes with only a $\log(k)$ factor difference in the optimization guarantee relative to classical greedy schemes \citet{barron2008approximation,DeVoreT96}.

\section{Experiments}\label{sec:experiments}

\begin{table*}[t]
\centering
\caption{\textbf{Real-data summary.} Lower is better. For Parkinson's, we report
the single-seed subject-level split with $\lambda$ selected by labeled-only CV. For
Bank Marketing, entries are mean holdout Brier scores across 26 outer seeds.}
\label{tab:cv_summary}
\small
\setlength{\tabcolsep}{5pt}
\begin{tabular}{@{}llccccc@{}}
\toprule
Dataset & Metric & Baseline & Two-Stage & Gen. Distill. & SVM+ & Ours \\
\midrule
Parkinson's & MSE   & 77.34  & 83.32  & --     & --     & \textbf{72.60}  \\
Bank Marketing & Brier & 0.0893 & 0.0928 & 0.0907 & 0.1470 & \textbf{0.0881} \\
\bottomrule
\end{tabular}
\end{table*}

\begin{figure*}[h!]
  \centering
  \begin{subfigure}[t]{0.32\textwidth}
    \centering
    \includegraphics[width=\linewidth]{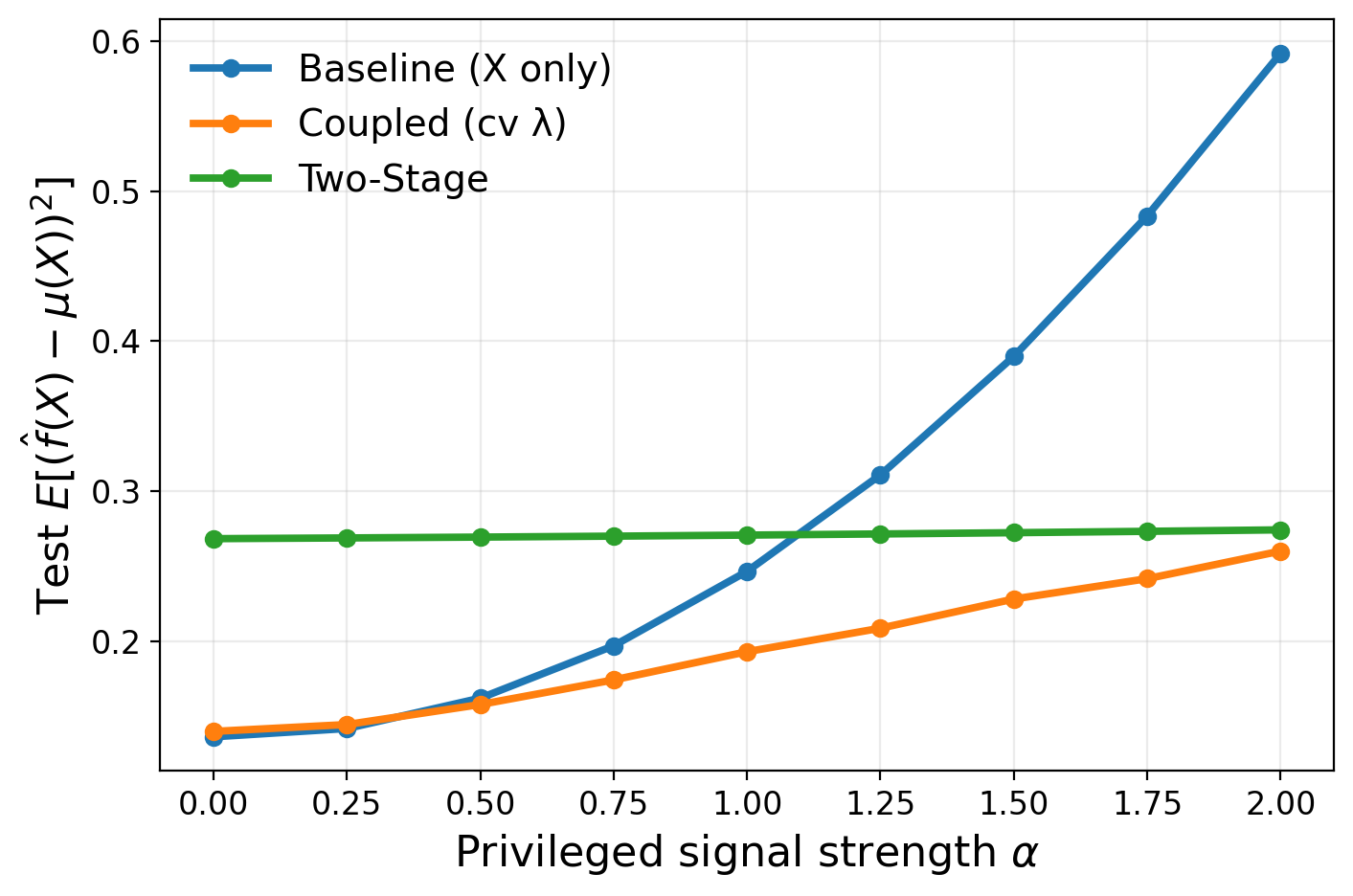}
    \caption{Privileged signal strength $\alpha$.}
    \label{fig:synthetic_signal_sweep}
  \end{subfigure}
  \hfill
  \begin{subfigure}[t]{0.32\textwidth}
    \centering
    \includegraphics[width=\linewidth]{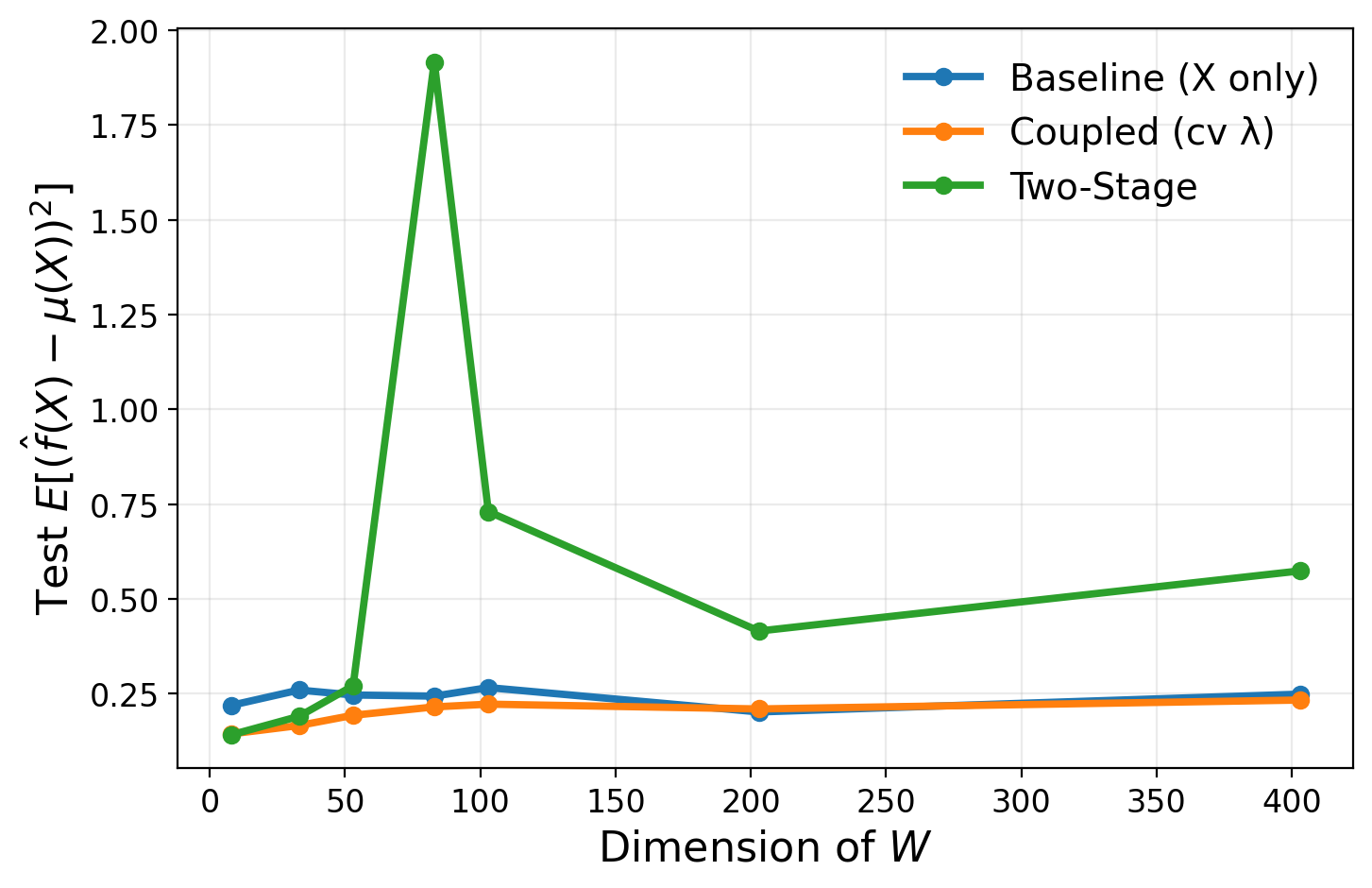}
    \caption{Dimension of $W$.}
    \label{fig:synthetic_wdim_sweep}
  \end{subfigure}
  \hfill
  \begin{subfigure}[t]{0.32\textwidth}
    \centering
    \includegraphics[width=\linewidth]{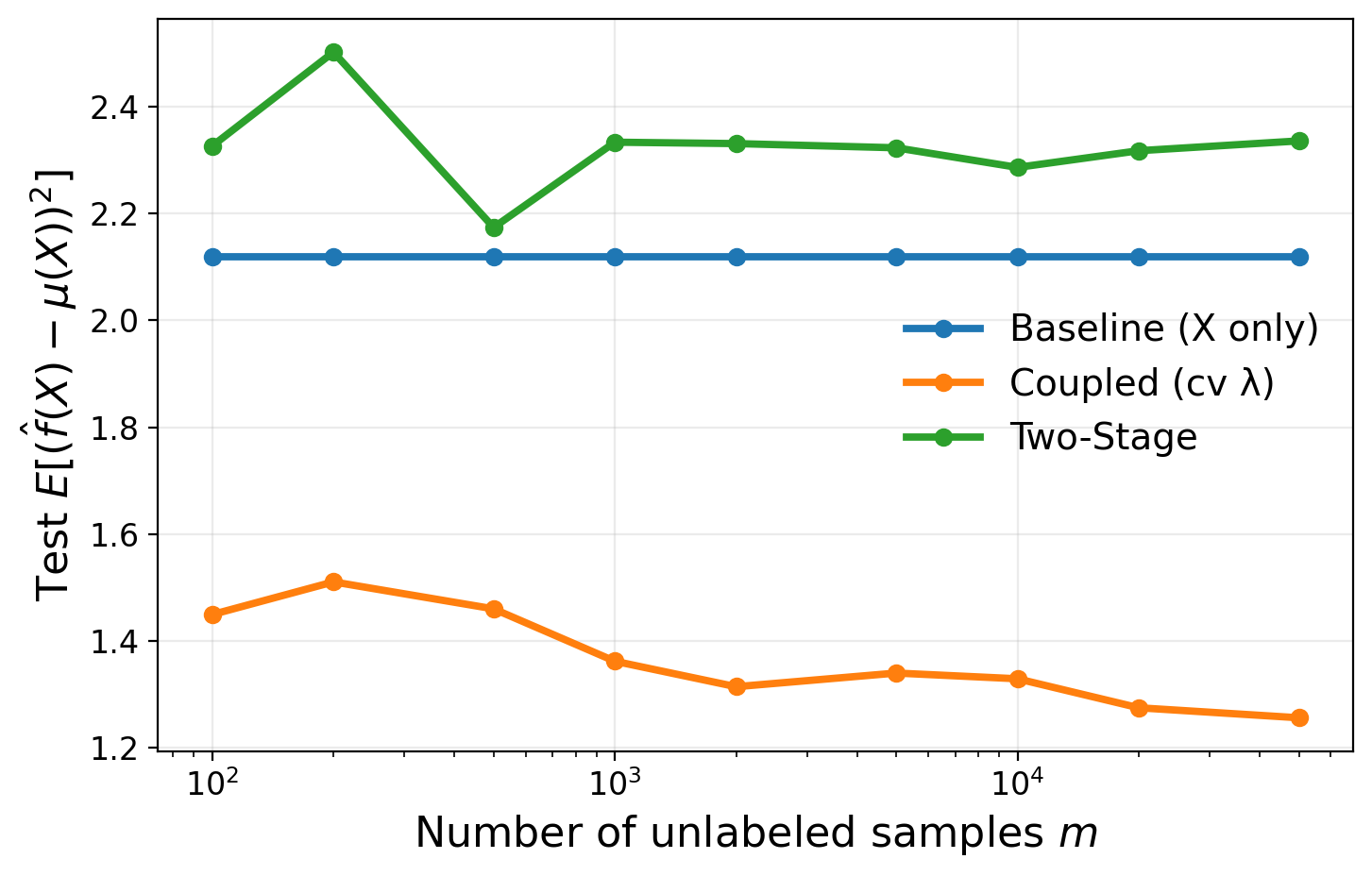}
    \caption{Number of unlabeled samples $m$.}
    \label{fig:synthetic_unlabeled_sweep}
  \end{subfigure}
  \caption{\textbf{Synthetic controls.} Test error is $\E[(\hat f(X)-\mu(X))^2]$. Coupled training adapts to useful privileged signal, is more stable than Two-Stage under nuisance privileged dimensions, and improves with additional unlabeled data.}
  \label{fig:synthetic_controls}
\end{figure*}

We evaluate the coupled objective on controlled synthetic data and two real-world tasks. 
The tuning parameter $\lambda$ interpolates between labeled-only learning on $X$ and the Two-Stage method. Unless otherwise stated, $\lambda$ is selected by cross-validation using only labeled training data. We compare against (i) supervised learning on $\{(X_i,Y_i)\}_{i=1}^n$ using only $X$ (Baseline), and (ii) the Two-Stage procedure from Section~\ref{sec:intro}. The square-loss synthetic controls match our theory, while the logistic synthetic diagnostic, random-forest regression, and Bank Marketing classification experiments are empirical demonstrations beyond the current proof regime.

\paragraph{Synthetic controls.}
We first use a linear model with known target $\mu(X)$ and privileged coordinates $W$ containing both signal and nuisance variation. Figure~\ref{fig:synthetic_controls} shows that the coupled method improves as the privileged signal strengthens, remains stable as nuisance privileged dimensions increase, and improves as the number of unlabeled samples increases. In contrast, Two-Stage pseudo-labeling can degrade when the rich-view teacher overfits nuisance variation.

\paragraph{Synthetic classification diagnostic.}
We also evaluate a binary logistic analogue of the coupled objective, replacing squared-loss term in~\eqref{eq:empirical_dual} by cross-entropy and using soft rich-view pseudo-labels. 
As shown in Figure~\ref{fig:synthetic_classification_ce} in Appendix~\ref{app:additional_experiments}, the test error is minimized at an intermediate value of $\lambda$, while the Two-Stage limit over-transfers rich-view noise.

\paragraph{Real-data benchmarks.}
For regression, we use the UCI Parkinson's Telemonitoring dataset and predict disease severity from voice measurements. Basic perturbation measures form the deployment view $X$, while more complex spectral features form the privileged view $W$. Figure~\ref{fig:parkinson_u_shape} in Appendix~\ref{app:additional_experiments} shows the expected U-shaped dependence on $\lambda$; Table~\ref{tab:cv_summary} reports the same subject-level split with the labeled-only CV choice $\hat\lambda=3$. Due to high variance from subject-level splitting, the main table reports single-seed results for Parkinson's. The larger Bank task averages 26 seeds. Cross-validation (CV) robustness checks for both are provided in Appendix~\ref{app:additional_experiments}.

For classification, we use the Bank Marketing dataset and evaluate by holdout Brier score. Static
demographic, contact, and macroeconomic variables form the deployment view $X$,
while current-call and campaign outcome variables form the privileged view $W$.
We evaluate over 26 stratified outer splits, using $n=200$ labeled examples and
$m=10{,}000$ unlabeled paired $(X,W)$ examples per split. The tuning parameter
$\lambda$ and all baselines are selected by the same fixed 4-fold
cross-validation protocol on the labeled set only. We compare against the
$X$-only Baseline, Two-Stage pseudo-labeling, a squared-loss generalized
distillation baseline, and SVM+ \citep{VAPNIK2009544}. Coupled Training beats
the $X$-only Baseline and generalized distillation in $22/26$ seeds, and beats Two-Stage
and SVM+ in $26/26$ seeds.
Figure~\ref{fig:bank_lambda_curve} in Appendix~\ref{app:additional_experiments} shows a representative split, where the
cross-validated $\hat\lambda$ lands near the interior low-Brier region of the
holdout diagnostic curve.
Table~\ref{tab:cv_summary} summarizes the real-data results. 

Overall, the experiments support the intended role of Coupled Training, exploiting privileged information when useful while attenuating it when direct pseudo-labeling would transfer nuisance signal into the final $X$-only predictor. Further implementation details and diagnostics are deferred to Appendix~\ref{app:additional_experiments}.

\section{Conclusion}

We studied learning with privileged features and unlabeled data when the deployment predictor may only use $X$ at test time. The main difficulty in this setting is negative transfer, where a rich-view predictor trained on $(X,W)$ can be useful when $W$ carries a genuine extra signal, but can also mislead the deployment model when the privileged view is weak or noisy. Our coupled objective addresses this by jointly fitting the deployment model and the rich-view model, rather than treating the rich-view predictor as a fixed teacher.

Our estimator interpolates between labeled-only learning (weak coupling) and Two-Stage pseudo-labeling (strong coupling). This interpolation is reflected both in the population characterization and in the empirical U-shaped dependence on $\lambda$ observed across our synthetic and real data experiments. In addition, for the high-dimensional setting, we developed a dictionary-based AFS procedure and established a sublinear optimization guarantee, which significantly reduces the time complexity.

\paragraph{Limitations.}
Our theory covers square-loss regression with convex function classes, so the random-forest and classification experiments are empirical evidence beyond the present proof setting. The method also assumes paired unlabeled samples with privileged features; extensions to misspecification, distribution shift, semi-paired data, and broader classification guarantees remain future work.

\section*{Acknowledgements}

We extend our special thanks to the anonymous reviewers for their insightful comments and suggestions that greatly enhanced this work. Hagrass gratefully acknowledges financial support from the Schmidt DataX Fund at Princeton University made possible through a major gift from the Schmidt Futures Foundation. Klusowski is grateful for support from the National Science Foundation through NSF CAREER DMS-2239448 and the Alfred P. Sloan Foundation through a Sloan Research Fellowship.

\section*{Impact Statement}

This paper presents work whose goal is to advance the field of machine learning. There are many potential societal consequences of our work, none of which we feel must be specifically highlighted here.

\bibliography{icml2026/refs_ICML}
\bibliographystyle{icml2026/icml2026}

\newpage
\appendix
\onecolumn
\input{icml2026/supplement}

\end{document}

%% file: icml2026/supplement.tex
\section{Efficient Implementation of Alternating Forward Selection}\label{app:efficient_impl}

Recall from \cref{sec:intro} that we observe $n$ labeled and $m$ unlabeled samples with $N=n+m$, and that $\Z=\X\times\W$. Let
$\pi_X:\Z\to\X$ be the canonical projection $\pi_X(x,w)=x$.
Define the finite measures on $\Z$ by
\[
\tilde P^{(1)} \coloneqq \sum_{i=1}^{n} \delta_{Z_i},\qquad
\tilde P^{(2)} \coloneqq \sum_{j=n+1}^{N} \delta_{Z_j},\qquad
\tilde P^{(3)} \coloneqq \lambda\,\tilde P^{(1)},
\]
and normalize by $N$.
\[
\hat P_N^{(\ell)} \coloneqq \frac{1}{N}\tilde P^{(\ell)}, \qquad \ell\in\{1,2,3\}.
\]
Note that $\hat P_N^{(1)}+\hat P_N^{(2)} = \hat P_N^Z$.
We will work exclusively with $Z$-indexed norms by viewing the deployment class as a subspace of functions on $\Z$.
Define the lifted deployment space
\[
\Hf \;\coloneqq\; \bigl\{\, f\in L^2(\hat P_N^Z)\ :\ f \text{ is }\sigma(\pi_X)\text{-measurable} \,\bigr\},
\]
i.e., $f\in\Hf$ implies $f(x,w)$ depends only on $x$ (equivalently, there exists $\bar f:\X\to\R$ such that
$f(x,w)=\bar f(x)$ $\hat P_N^Z$-a.e.).  Let $Q_g \coloneqq \hat P_N^{(2)}+\hat P_N^{(3)}$ and set
\[
\Hg \;\coloneqq\; L^2(Q_g).
\]
We assume the function classes live in these Hilbert spaces, $\mathcal{F}\subset \mathcal{H}_f,\mathcal{G}\subset \mathcal{H}_g$. To keep notation light, we identify the model class $\mathcal{F}$ with its lift to $\mathcal{Z}$, namely, we view each $f\in\mathcal{F}$ as a function $f(x,w)$ that is $\sigma(\pi_X)$-measurable (hence depends only on $x$, i.e., is constant in $w$ given $x$). We use $\mathcal{D}_f$ in the same spirit, and throughout the appendix it refers to a dictionary in the lifted space $\mathcal{H}_f$.
With this convention, for $f\in\Hf$, $g\in\Hg$, the empirical objective can be written as
\[
\hat{\mathcal{L}}(f,g;\lambda)
\;\coloneqq\;
\|Y-f\|_{L^2(\hat P_N^{(1)})}^2
\;+\;
\|f-g\|_{L^2(\hat P_N^{(2)})}^2
\;+\;
\|Y-g\|_{L^2(\hat P_N^{(3)})}^2.
\]

Let
\[
\mathcal K_1 \coloneqq L^2(\hat P_N^{(1)}+\hat P_N^{(2)}),
\qquad
\mathcal K_2 \coloneqq L^2(\hat P_N^{(2)}+\hat P_N^{(3)}),
\qquad
\mathcal K \coloneqq \mathcal K_1\times \mathcal K_2 .
\]
For \(u=(u_1,u_2)\) and \(v=(v_1,v_2)\) in \(\mathcal K\), define
\[
\langle u,v\rangle_\star
=
\langle u_1,v_1\rangle_{L^2(\hat P_N^{(1)})}
+
\langle u_1-u_2,v_1-v_2\rangle_{L^2(\hat P_N^{(2)})}
+
\langle u_2,v_2\rangle_{L^2(\hat P_N^{(3)})}.
\]
This bilinear form is positive semidefinite. Let
\[
\mathcal N_\star
\coloneqq
\{u\in\mathcal K:\|u\|_\star=0\}
\]
and define the quotient Hilbert space
\[
\bar{\mathcal K}\coloneqq \mathcal K/\mathcal N_\star.
\]
We write \([u]\) for the equivalence class of \(u\), although, when no confusion is possible, we suppress brackets and use the same notation for a representative and its quotient class.
The deployment and rich-view model spaces are embedded as
\[
{\Hs}
\coloneqq
\{[(f,g)]: f\in\Hf,\ g\in\Hg\}
\subset \bar{\mathcal K}.
\]
Below, all projections, residuals, spans, and orthogonality statements are taken in
\(\bar{\mathcal K}\) with respect to \(\langle\cdot,\cdot\rangle_\star\).

Let \(Y\) be any extension of the labeled responses to all \(N\) sample points, with
\(Y(Z_i)=Y_i\) for \(i\le n\). Define
\[
T\coloneqq (Y,Y)\in\bar{\mathcal K}.
\]
The choice of \(Y\) on unlabeled samples is irrelevant, since the unlabeled term depends on
\((Y-f)-(Y-g)=g-f\). Hence, for every \(f,g\),
\[
\|T-(f,g)\|_\star^2
=
\|Y-f\|_{L^2(\hat P_N^{(1)})}^2
+
\|f-g\|_{L^2(\hat P_N^{(2)})}^2
+
\|Y-g\|_{L^2(\hat P_N^{(3)})}^2
=
\hat{\mathcal L}(f,g;\lambda).
\]

\begin{algorithm}[t!]
  \caption{Alternating Forward Selection}
  \label{alg}
  \begin{algorithmic}
    \STATE \textbf{Input.} Dictionaries $\mathcal{D}_f^\star$, $\mathcal{D}_g^\star$; initial values $(f_0, g_0) = (0,0)$, $S_0^f=S_0^g=\{0\}$; target $(Y, Y)$.

    \FOR{$k = 1, 2, \ldots$}
        \STATE Compute the residual.
        \[
            r_k \gets (Y, Y) - (f_{k-1}, g_{k-1}).
        \]

        \STATE \textit{Step 1. $f$-selection and projection}
        \STATE Select
        \[
        \psi_k^\star = (\psi_k, 0)
\in
\argmax_{\substack{a\in\mathcal D_f^\star\\
\|\Pi_{S_{k-1}^f}^{\perp}a\|_\star>0}}
\frac{\langle r_k,\Pi_{S_{k-1}^f}^{\perp}a\rangle_\star}
{\|\Pi_{S_{k-1}^f}^{\perp}a\|_\star}.
        \]

        \STATE Update
        \[
            S_k^f \gets \text{span}(\psi_1^\star, \ldots, \psi_k^\star), \]\[
            (f_k, g_{k-1}) \gets (f_{k-1} ,\, g_{k-1})+ \Pi_{S_k^f} r_k,
        \]
        
        \STATE and set
        \[
            r_k^{(g)} \gets r_k - \Pi_{S_k^f} r_k.
        \]

        \STATE \textit{Step 2. $g$-selection and projection}
        \STATE Select
        \[
            \phi_k^\star =(0,\phi_k)\in \argmax_{\substack{b\in\mathcal D_g^\star\\
\|\Pi_{S_{k-1}^g}^{\perp}b\|_\star>0}}
            \frac{\langle r_k^{(g)}, \Pi^{\perp}_{S_{k-1}^g} b \rangle_\star}
                 {\|\Pi^{\perp}_{S_{k-1}^g} b\|_\star}.
        \]

        \STATE Update
        \[
            S_k^g \gets \text{span}(\phi_1^\star, \ldots, \phi_k^\star), \]\[
            (f_k, g_k) \gets (f_k,\, g_{k-1} )+ \Pi_{S_k^g} r_k^{(g)},
        \]
    \ENDFOR
  \end{algorithmic}
\end{algorithm}


Without loss of generality, we suppose that the dictionaries $\mathcal{D}_f$ and $\mathcal{D}_g$ are bounded by $1$ in the \emph{pooled} empirical norm, i.e., 
\[
\sup_{\psi\in\mathcal{D}_f}\|\psi\|_{L^2(\hat P_N^Z)}\le 1,\qquad\sup_{\phi\in\mathcal{D}_g}\|\phi\|_{L^2(\hat P_N^Z)}\le 1.
\]
In the following, we use canonical embeddings to construct the dictionary for the Hilbert space $\Hs$ based on $\mathcal{D}_f$ and $\mathcal{D}_g$.
Define embedded versions.
\[
\mathcal{D}_f^\star \coloneqq \{(\psi,0) : \psi \in \mathcal{D}_f\}, \qquad
\mathcal{D}_g^\star \coloneqq \{(0,\phi) : \phi \in \mathcal{D}_g\}.
\]
Assume the union $\mathcal{D}^\star \coloneqq \mathcal{D}_f^\star \cup \mathcal{D}_g^\star$ is balanced, i.e., whenever $\zeta\in\mathcal{D}^\star$, we have $-\zeta\in\mathcal{D}^\star$. Since $\hat P_N^{(1)}+\hat P_N^{(2)}=\hat P_N^Z$, we have
\[
\|(\psi,0)\|_\star^2=\|\psi\|_{L^2(\hat P_N^{(1)})}^2+\|\psi\|_{L^2(\hat P_N^{(2)})}^2=\|\psi\|_{L^2(\hat P_N^Z)}^2\le 1.
\]
Moreover,
\[
\|(0,\phi)\|_\star^2
=\|\phi\|_{L^2(\hat P_N^{(2)})}^2+\|\phi\|_{L^2(\hat P_N^{(3)})}^2
=\frac{1}{N}\sum_{j=n+1}^{N} \phi(Z_j)^2+\frac{\lambda}{N}\sum_{i=1}^{n} \phi(Z_i)^2
\le \max\{1,\lambda\}\,\|\phi\|_{L^2(\hat P_N^Z)}^2
\le \max\{1,\lambda\}.
\]
Denote $R_\lambda\coloneqq\max\{1,\sqrt{\lambda}\}$. Thus every embedded atom satisfies $\|a\|_\star\le R_\lambda$. 
Throughout this subsection, $\Pi_S$ and $\Pi_S^\perp$ denote the orthogonal projections onto a subspace $S \subset \bar{\mathcal K}$ and its orthogonal complement, respectively, with respect to the inner product $\langle\cdot,\cdot\rangle_\star$.

\section{Implementation Details and Additional Experiments} \label{app:additional_experiments}

\subsection{Parkinson's (Telemonitoring)}

\paragraph{Dataset.} We use the UCI Parkinson’s Telemonitoring dataset to evaluate Coupled Training. The target is \texttt{total\_UPDRS}, which typically requires in-person clinical assessment, so labeled data are limited. Each row is a single voice recording, and each \texttt{subject} may contribute multiple recordings. Since recording conditions at deployment are less controlled, we treat some acoustic descriptors as privileged features available only during training.

\paragraph{Feature split.}
We partition covariates into deployment features $X$ and privileged features $W$ to model the fact that some high-fidelity acoustic descriptors are only reliable in controlled clinical conditions.
\begin{itemize}
    \item \textbf{Deployment features} $X$ (10) are \{\texttt{age, sex, test\_time, Jitter(\%), Jitter(Abs), Jitter:RAP, Jitter:PPQ5, Jitter:DDP, Shimmer, Shimmer(dB)}\}.
    \item \textbf{Privileged features} $W$ (9) are \{\texttt{Shimmer:APQ3, Shimmer:APQ5, Shimmer:APQ11, Shimmer:DDA, NHR, HNR, RPDE, DFA, PPE}\}.
\end{itemize}

Here $W$ contains higher-fidelity acoustic descriptors that are reliable primarily in controlled clinical recordings,
whereas $X$ contains features assumed available at deployment.
We use $W$ only during training to improve the deployment predictor $\hat f(X)$.

\paragraph{Data split.}
To avoid leakage across repeated recordings from the same subject, we split the data at the subject level. We hold out 10\% of subjects as a test set (\texttt{seed}=202). From the remaining subjects, we label 15\% and treat the rest as unlabeled.

\paragraph{Preprocessing.}
All preprocessing is fit on the training set only.
We standardize non-demographic voice features and apply PCA with 12 components (or fewer if limited by dimension),
standardize demographic features (\texttt{age} and \texttt{test time}), and leave \texttt{sex} unchanged.
We also standardize $Y$ using statistics computed from labeled training targets.
\paragraph{Models and hyperparameters.}
We use a random forest regressor with 1400 trees and maximum depth 12 (minimum leaf size 4); remaining settings are kept fixed across methods.

\begin{figure}[t]
  \centering
  \includegraphics[width=0.6\linewidth]{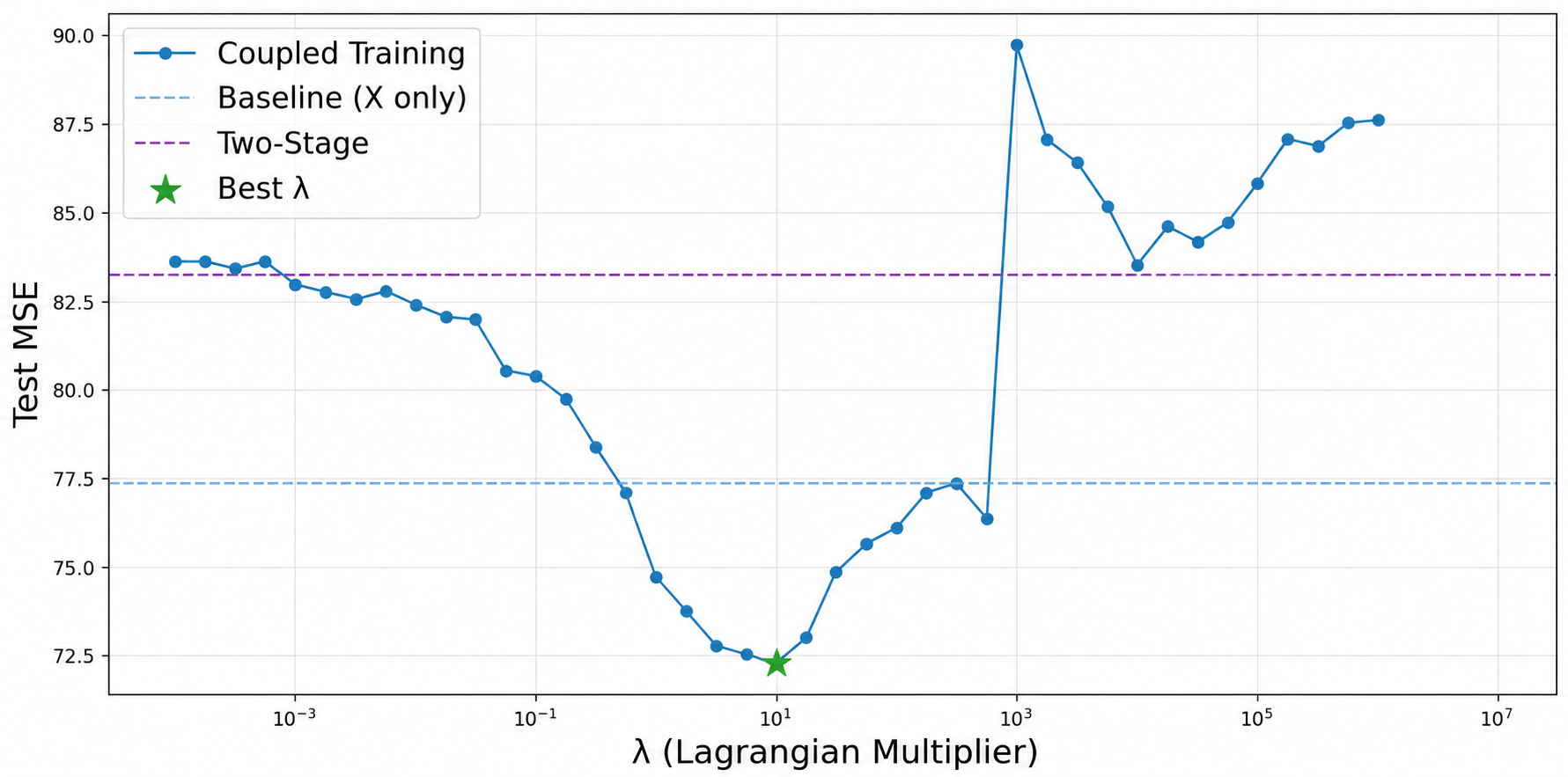}
  \caption{\textbf{Parkinson's dataset.} Test MSE versus $\lambda$.}
  \label{fig:parkinson_u_shape}
\end{figure}

\paragraph{Training procedure.}
Baseline trains $f$ on labeled data only using $X\to Y$.
Two-Stage trains a teacher $g$ on labeled data using $(X,W)\to Y$, pseudo-labels unlabeled data with $\tilde Y_U=g(X_U,W_U)$, then trains a student on $(X_L\cup X_U)\to(Y_L\cup \tilde Y_U)$.
\textbf{Coupled Training} runs the coupled training updates for up to 15 iterations. For each $\lambda$ in a log-spaced grid $\{10^{-4},\dots,10^{6}\}$ with 41 points, we initialize $f$ from the labeled baseline and initialize $g$ by fitting to a constant target $g_0\equiv 0$. In this random-forest experiment we use the penalized Lagrangian form of the square-loss objective,
\[
\sum_{i\in L}(Y_i-f(X_i))^2
+
\sum_{j\in U}(g(X_j,W_j)-f(X_j))^2
+
\lambda\sum_{i\in L}(Y_i-g(X_i,W_i))^2 ,
\]
with all targets standardized using the labeled training responses. Given the current $g$, the $f$-update is a random-forest regression of the targets $(Y_L,g(X_U,W_U))$ on the deployment features $(X_L,X_U)$. Given the current $f$, the $g$-update is a weighted random-forest regression of the targets $(Y_L,f(X_U))$ on the rich-view features $((X_L,W_L),(X_U,W_U))$, using weight $\lambda$ for labeled samples and weight $1$ for unlabeled samples. We employ early stopping based on stabilization of the disagreement on unlabeled data (patience 2).
We then evaluate the resulting $\hat f(X)$ on the held-out test set.

\paragraph{Evaluation.}
All models are trained to predict scaled $Y$; at evaluation, we inversely transform predictions back to the original UPDRS scale, clip to $[0,100]$, and report test mean squared error (MSE).

Figure~\ref{fig:parkinson_u_shape} reports test MSE as a function of $\lambda$; performance is best at an intermediate value (around $10^{1}$).

\paragraph{Cross-validation for $\lambda$.}
We select $\hat\lambda$ using 5-fold cross-validation on the labeled set only, splitting by subject to prevent leakage
(\textsc{GroupKFold}). In each fold, we train Coupled Training using the fixed unlabeled pool $\mathcal D_U$
and evaluate the validation MSE of the deployment model $\hat f_\lambda$ on the held-out labeled fold. 
We take $\hat\lambda\in\arg\min_{\lambda\in\Lambda}\mathrm{MSE}_{\mathrm{val}}(\hat f_\lambda)$,
where $\Lambda=\{10^{-4},\ldots,10^{6}\}$ is a log-spaced grid with 41 points and $(\hat f_\lambda,\hat g_\lambda)$ denotes the models returned by Coupled Training
when run with tuning parameter $\lambda$.
The main-table result uses this CV-selected $\hat\lambda$; because the dataset has few independent subjects after subject-level splitting, we treat this as the primary diagnostic split rather than averaging a high-variance mean.
As a robustness check over seeds $\{202,42,123,999,777,2023,2024,888\}$, CV-selected Coupled Training beats both Baseline and Two-Stage in $6/8$ splits, with wins at seeds $\{202,999,777,2023,2024,888\}$.

\subsection{Bank Marketing}\label{app:bank_marketing}

\begin{figure}[t]
  \centering
  \includegraphics[width=0.6\linewidth]{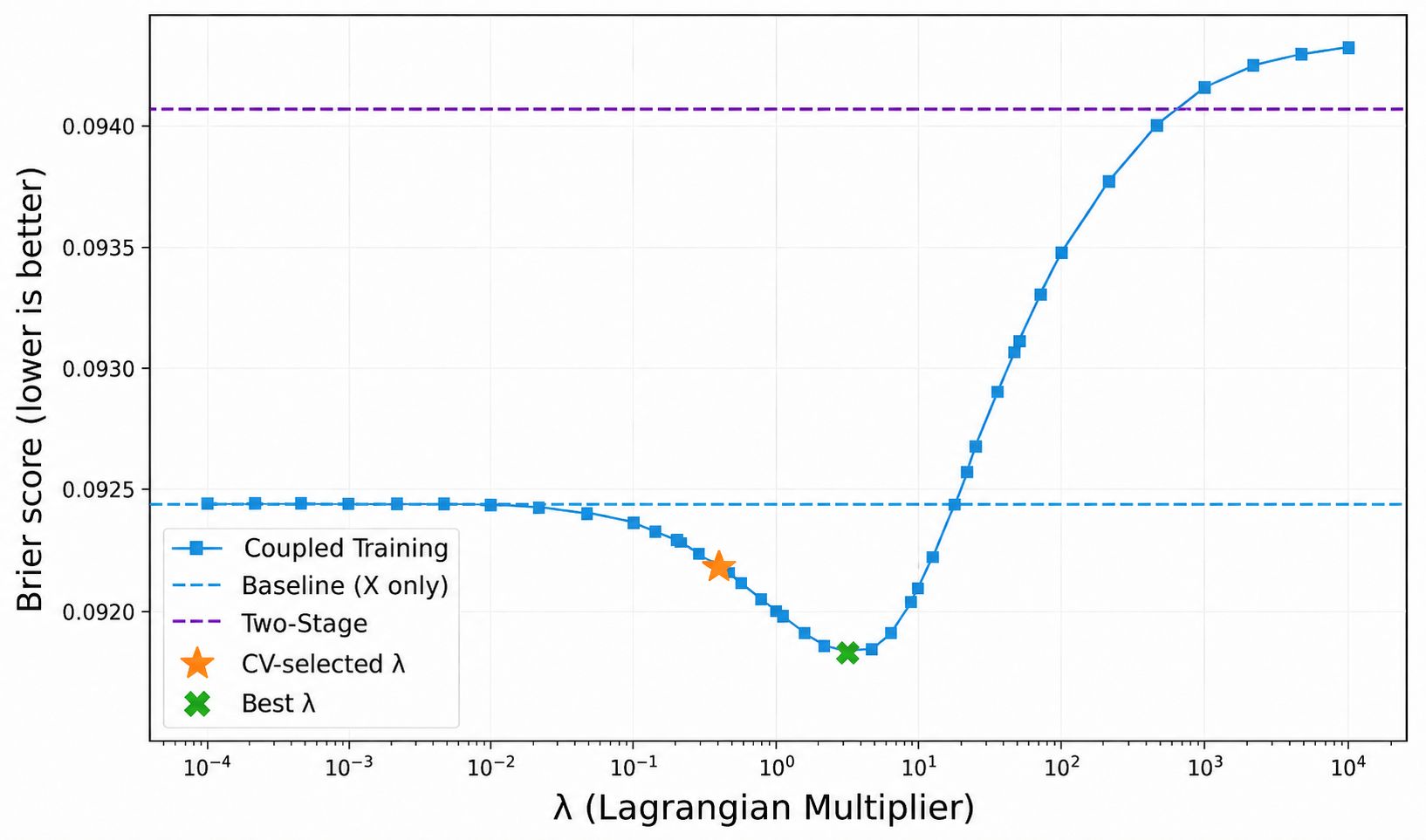}
  \caption{\textbf{Bank Marketing dataset.} Holdout Brier score versus $\lambda$.}
  \label{fig:bank_lambda_curve}
\end{figure}

\paragraph{Dataset.}
We evaluate Coupled Training on the UCI Bank Marketing dataset, a binary
classification task where the goal is to predict whether a client subscribes to
a term deposit. The dataset contains $41{,}188$ examples and has a positive class
rate of approximately $0.113$. We report Brier score, so lower values are better.

\paragraph{Feature split.}
We use a train-time privileged information split motivated by the distinction
between pre-deployment information and information that is tied to the current
campaign call or campaign outcomes.
\begin{itemize}
    \item \textbf{Deployment features} $X$ are
    \{\texttt{age}, \texttt{job}, \texttt{marital}, \texttt{education},
    \texttt{default}, \texttt{housing}, \texttt{loan}, \texttt{contact},
    \texttt{month}, \texttt{day\_of\_week}, \texttt{emp.var.rate},
    \texttt{cons.price.idx}, \texttt{cons.conf.idx}, \texttt{euribor3m},
    \texttt{nr.employed}\}.
    \item \textbf{Privileged features} $W$ are
    \{\texttt{duration}, \texttt{pdays}, \texttt{previous}, \texttt{poutcome}\}.
\end{itemize}
The final deployed predictor uses only $X$. The privileged view $W$ is used only
during training.

\paragraph{Data split.}
For each outer seed in
\[
\{2034,2035,\ldots,2059\},
\]
we use a stratified random $80/20$ train--test split. From the training pool, we
sample $n=200$ labeled examples stratified by class and $m=10{,}000$ unlabeled
examples from the remaining training examples. In each split, the labeled set
contains $23$ positives and $177$ negatives.

\paragraph{Preprocessing.}
All preprocessing is fitted on the training pool only. Numeric features are
standardized, categorical features are one-hot encoded with unknown categories
ignored at test time, and the transformed design matrices are dense. The test
set is never used for preprocessing, hyperparameter selection, or model fitting.

\paragraph{Training-time corruption of privileged features.}
To model noisy or imperfect train-time privileged information, we corrupt $W$
only on training rows. The fixed corruption preset is
\texttt{mix\_plus\_noise}, where a fraction $0.35$ of the training rows has its
privileged feature block row-mixed, then featurewise dropout with probability
$0.10$ is applied, followed by Gaussian noise with standard deviation $0.15$.
The same corruption protocol is used for Coupled Training, Two-Stage, and
squared-loss generalized distillation whenever those methods use $W$.

\paragraph{Coupled model and primary baselines.}
The $X$-only Baseline, Two-Stage, squared-loss generalized distillation, and
Coupled Training models all use linear squared-loss predictors with an
unpenalized intercept. The fixed ridge parameters for Coupled Training and
Two-Stage are $\alpha_f=300$ for the deployment model and $\alpha_g=0.01$ for
the rich-view model. The agreement weight on unlabeled examples is $1$, and the
Two-Stage pseudo-label weight is also $1$.

Let $\bar X$ denote $X$ augmented with an intercept and let
$\bar Z$ denote $(X,W)$ augmented with an intercept. For each $\lambda$, the
coupled model solves
\[
\min_{\beta,\gamma}
\sum_{i\in L} (Y_i-\bar X_i^\top\beta)^2
+
\sum_{j\in U}(\bar X_j^\top\beta-\bar Z_j^\top\gamma)^2
+
\lambda\sum_{i\in L}(Y_i-\bar Z_i^\top\gamma)^2
+
\alpha_f\|\beta_{-0}\|_2^2
+
\alpha_g\|\gamma_{-0}\|_2^2,
\]
where the subscript $-0$ excludes the intercept.

\paragraph{Cross-validation for $\lambda$.}
We select $\lambda$ by fixed 4-fold stratified cross-validation on the
$200$ labeled examples only. The unlabeled examples are used for training inside
each fold, but their labels are never used for selection. The grid is
\[
\Lambda
=
\operatorname{unique}\Big(
\operatorname{logspace}(-4,4,25)
\cup
\operatorname{logspace}(-1,2,21)
\Big).
\]
After selecting $\hat\lambda$, we refit on the full labeled set and the same
unlabeled pool, then evaluate the $X$-only deployment predictor on the held-out
test set.

\paragraph{Squared-loss generalized distillation.}
We include a same-loss generalized-distillation baseline. The teacher is a
ridge-regularized linear predictor trained on a privileged train-time view
$Z$, where $Z$ is selected from either $W$ alone or $[X,W]$. For a teacher ridge
parameter $\alpha_T$, the teacher solves
\[
\hat\gamma
\in
\arg\min_{\gamma}
\sum_{i\in L}(Y_i-\bar Z_i^\top\gamma)^2
+
\alpha_T\|\gamma_{-0}\|_2^2 .
\]
It produces soft teacher targets
\[
q_i=\bar Z_i^\top\hat\gamma .
\]
The student is an $X$-only ridge model trained by soft squared-loss
distillation.
\[
\min_{\beta}
\sum_{i\in L}(Y_i-\bar X_i^\top\beta)^2
+
a_L\sum_{i\in L}(q_i-\bar X_i^\top\beta)^2
+
a_U\sum_{j\in U}(q_j-\bar X_j^\top\beta)^2
+
\alpha_S\|\beta_{-0}\|_2^2 .
\]
There is no hard pseudo-label term. We select the teacher view, teacher ridge
parameter, student ridge parameter, and soft-distillation weights by the same
fixed 4-fold labeled-only CV protocol used for the coupled model. The narrowed
search grid is
\[
Z\in\{W,[X,W]\},\qquad
\alpha_T\in\{0.003,0.01,0.03\},\qquad
\alpha_S\in\{90,300,900\},
\]
and
\[
(a_L,a_U)
\in
\{(0,0.5),(0,1),(0.25,0.5),(0.25,1),(0.5,0.5),(0.5,1)\}.
\]
This grid contains the ordinary squared-loss Two-Stage baseline as the special
case
\[
Z=[X,W],\qquad
\alpha_T=\alpha_g=0.01,\qquad
\alpha_S=\alpha_f=300,\qquad
(a_L,a_U)=(0,1).
\]

\paragraph{Vapnik SVM+.}
We also compare against Vapnik SVM+, a labeled-only privileged information
baseline. SVM+ uses $W$ only for the labeled training examples and predicts with
$X$ only at test time. Since classical SVM+ is not a semi-supervised method, it
does not use the unlabeled pool for learning. We tune
\[
C\in\{0.1,1,10\},
\qquad
\gamma\in\{0.1,1,10\}
\]
by the same fixed 4-fold labeled-only CV protocol, then refit on all
$200$ labeled examples. We use the classical SVM+ limit $\Delta=\infty$ and
Platt calibration with calibration parameter $C=1$ to obtain probabilities for
Brier-score evaluation.

\paragraph{Evaluation.}
For all methods, the final test-time predictor is an $X$-only predictor. We
clip predicted probabilities to $[10^{-6},1-10^{-6}]$ and report holdout Brier
score. All model choices are made without using holdout labels.

\paragraph{Results.}
Table~\ref{tab:bank_marketing_details} reports the 26-seed comparison. The
CV-selected coupled model obtains mean holdout Brier score $0.0881$, improving
over the $X$-only Baseline, Two-Stage, squared-loss generalized distillation,
and Vapnik SVM+. Squared-loss generalized distillation improves over Two-Stage,
with mean Brier score $0.0907$ compared with $0.0928$, but remains worse than
Coupled Training. The mean Brier-score gain of Coupled Training is $0.0011$
against the Baseline, $0.0047$ against Two-Stage, $0.0026$ against squared-loss
generalized distillation, and $0.0589$ against Vapnik SVM+. Coupled Training
beats the Baseline in $22/26$ seeds, Two-Stage in $26/26$ seeds, squared-loss
generalized distillation in $22/26$ seeds, and Vapnik SVM+ in $26/26$ seeds.
The mean regret of the CV-selected $\hat\lambda$ relative to the per-seed
holdout-best $\lambda$ is $0.00125$ Brier score, with median regret $0.00050$.

\begin{table}[t]
\centering
\caption{\textbf{Bank Marketing detailed comparison.} Means and confidence
intervals are computed across 26 outer seeds. ``Coupled gain'' is the mean
Brier-score improvement of the CV-selected coupled model over the corresponding
method; positive values favor Coupled Training.}
\label{tab:bank_marketing_details}
\small
\setlength{\tabcolsep}{4.5pt}
\begin{tabular}{@{}lcccc@{}}
\toprule
Method & Uses unlabeled? & Mean Brier & 95\% CI & Coupled gain \\
\midrule
Baseline ($X$ only)          & No  & 0.0893 & [0.0888, 0.0897] & 0.0011 \\
Two-Stage                   & Yes & 0.0928 & [0.0918, 0.0939] & 0.0047 \\
Gen. Distill. (squared)     & Yes & 0.0907 & [0.0894, 0.0921] & 0.0026 \\
Vapnik SVM+                 & No  & 0.1470 & [0.1221, 0.1720] & 0.0589 \\
Coupled, CV-selected $\lambda$ & Yes & \textbf{0.0881} & [0.0874, 0.0889] & -- \\
\bottomrule
\end{tabular}
\end{table}

\subsection{PneumoniaMNIST Experiment with AFS-Based Dictionary Selection}\label{subsec:pneumoniamnist}
\paragraph{Dataset.}
We use Algorithm~\ref{alg} as a dictionary-selection stage on PneumoniaMNIST, a binary classification task with $28\times 28$ grayscale images.
We use the flattened image as deployment features $X\in\mathbb R^{784}$ and construct privileged features $W\in\mathbb R^{18}$ available only during training, with the first coordinate of $W$ a noisy version of $Y$ and the remaining ones random linear projections of $X$.
The label $Y\in\{0,1\}$ indicates pneumonia.
\paragraph{Feature split. }
Let $X\in\R^{N\times 784}$ be pixel features scaled to $[0,1]$. We generate $W\in\R^{N\times 18}$ by
(i) $W_{:,1}=Y+\varepsilon$ with i.i.d.\ $\varepsilon\sim\mathcal{N}(0,0.3^2)$, and
(ii) $W_{:,2:18}=XP$ where $P\in\R^{784\times 17}$ has i.i.d.\ entries $\mathcal{N}(0,0.1^2)$ (fixed given the seed).

\paragraph{Data split.}
We use a single stratified split with $n=100$ labeled points, $m=2000$ unlabeled points, and $n_{\mathrm{test}}=1500$ test points (\texttt{seed}=42).
\paragraph{Preprocessing.}
For all linear/kernel methods, we standardize both $X$ and $W$ using \texttt{StandardScaler} fit on labeled$\cup$unlabeled data only
(the test set is never used for fitting). We also feed standardized $W$ to the separately trained rich-view teachers described below.

\paragraph{AFS selection stage and dictionaries.}
For each of 10 log-spaced values of $\lambda$ from $10^{-2}$ to $10^{2}$,
we run Algorithm~\ref{alg} in $\Hs=\Hf\times\Hg$
(Section~\ref{sec: algorithm implementation}) for $K=70$ iterations with squared loss.
The internal iterates $(f_k,g_k)$ drive the successive atom selections. After iteration $K$,
we retain only the selected deployment span $S_K^f$ for the post-selection refit below;
neither $f_K$ nor $g_K$ is used as the final deployment predictor or as the refit teacher.
We consider the following dictionary constructions.
\begin{itemize}\itemsep2pt
\item \textbf{Linear AFS (random-feature dictionary).}
Embedded dictionaries use fixed Gaussian random projections with sizes
$\texttt{dict\_size\_f}=\texttt{dict\_size\_g}=2048$ for $f$ and $g$ respectively.
\item \textbf{Kernelized AFS (RBF kernel dictionary).}
Centers are all labeled points plus up to $500$ randomly chosen unlabeled points.
Kernel bandwidths $\gamma_f$ (for $X$) and $\gamma_g$ (for $(X,W)$) are initialized by the median heuristic
(computed on up to $600$ samples), and tuned by stratified $5$-fold CV on labeled data, maximizing AUROC over
$\gamma\in\gamma_0\cdot\{0.25,0.5,1,2,4\}$ and $\alpha\in\{10^{-3},10^{-2},10^{-1},1,10\}$.
\end{itemize}

\paragraph{Post-selection teacher and refit.}
Separately from the selection-stage iterate $g_K$, we fit a labeled-only rich-view teacher
$g_{\mathrm T}$ using $(X_L,W_L)\to Y_L$ (ridge in the linear case; kernel ridge in the
kernel case with tuned $\alpha_g$). We form pseudo-targets
$\tilde Y_U=g_{\mathrm T}(X_U,W_U)$ and fit the final deployment predictor
$\hat f_{\mathrm{refit}}$ by ridge regression with $\alpha_{\mathrm{refit}}=10^{-3}$ on
$(X_L\cup X_U)\to(Y_L\cup\tilde Y_U)$, restricted to $S_K^f$.
Thus, Algorithm~\ref{alg} determines the deployment dictionary span, whereas this separate
ridge refit determines the reported coefficients. Theorem~\ref{thm:oga-main-1over5}
concerns the selection-stage iterates, not the post-selection refit or its test metrics.

\paragraph{Training procedure.}
Baseline trains $f$ on labeled data only using $X\to Y$ (ridge for linear; kernel ridge with tuned $(\gamma_f,\alpha_f)$ for kernel).
Two-Stage uses the same labeled-only rich-view teacher construction to pseudo-label the
unlabeled data, but fits its deployment student in the full model class rather than restricting
it to the AFS-selected span $S_K^f$.
For reference, we also report a CNN baseline (image-only student) and a CNN Two-Stage teacher--student pipeline where the teacher fuses
image features with standardized $W$, pseudo-labels unlabeled data (threshold $0.5$), and retrains the image-only student;
we use Adam with learning rate $10^{-3}$, batch size $32$, and $30$ epochs.

\paragraph{Evaluation.}
For the AFS-based methods, the reported deployment score is $\hat f_{\mathrm{refit}}(X)$.
Models output real-valued scores; we apply a sigmoid to obtain probabilities and report test AUROC (primary),
along with accuracy at threshold $0.5$ and probability MSE against $\{0,1\}$ targets.

\paragraph{Results.}
Figure~\ref{fig:pneumoniamnist} shows the test AUROC of the post-selection refits as a function of $\lambda$. Coupled Training with AFS-based selection performs better than the Baseline and Two-Stage methods for many $\lambda$ values, with the strongest results at intermediate $\lambda$. The effect is especially clear for linear models. For a subset of $\lambda$ choices, linear AFS-based selection even outperforms both CNN Baseline and Two-Stage.

\begin{figure}[t]
  \centering
  \includegraphics[width=0.6\linewidth]{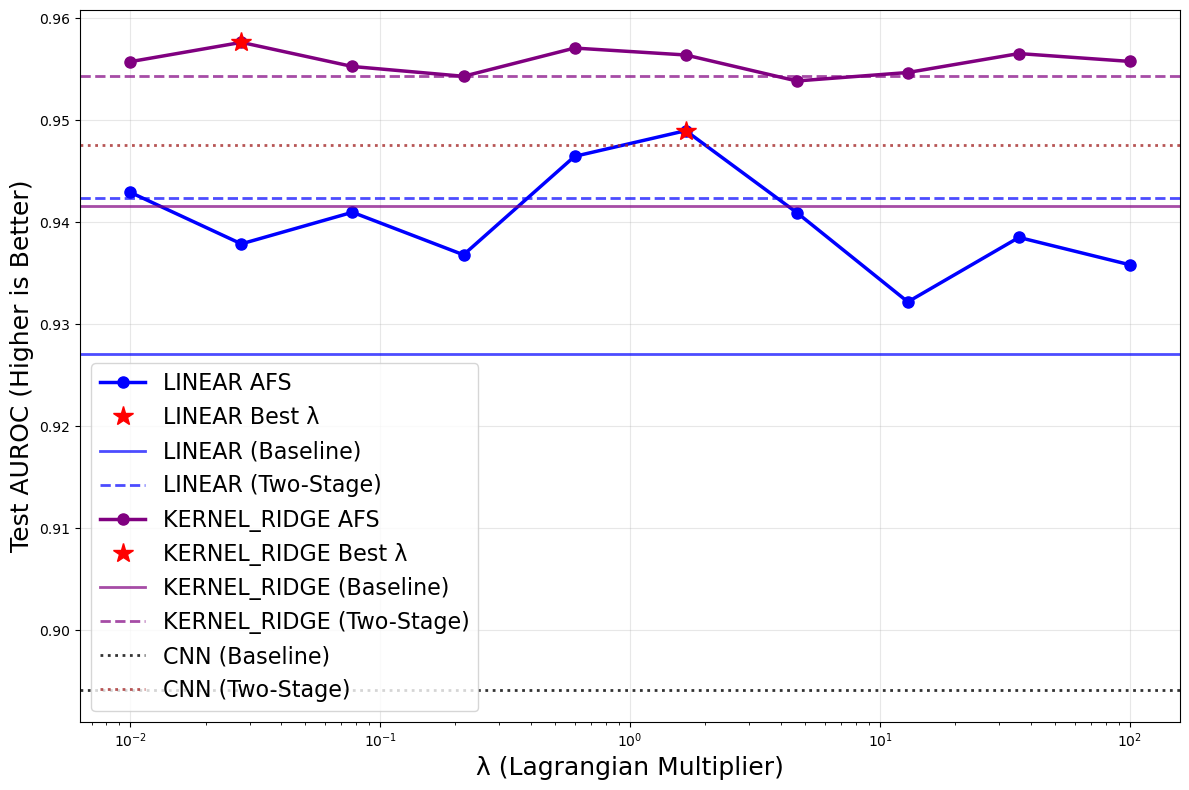}
  \caption{\textbf{PneumoniaMNIST.} Test AUROC of the post-selection deployment refit versus $\lambda$. Algorithm~\ref{alg} is used only to select the deployment dictionary span.}
  \label{fig:pneumoniamnist}
\end{figure}

\subsection{Controlled Synthetic Linear Experiments}\label{app:controlled_synthetic}

\paragraph{Dataset.}
We generate controlled synthetic data to isolate three effects that are central to the behavior of Coupled Training, namely the strength of the privileged signal, the amount of nuisance variation in the privileged view, and the number of unlabeled samples. The data-generating process separates the deployment signal from the privileged signal. Specifically, we sample
\[
X\sim \mathcal N(0,I_{d_X}),\qquad
H\sim \mathcal N(0,I_q),\qquad
V\sim \mathcal N(0,I_{d_{\mathrm{noise}}}),
\]
where $H$ contains the informative latent privileged factors and $V$ contains nuisance privileged coordinates. We then form
\[
W_{\mathrm{sig}}
=
\rho_{XW}XA+\sqrt{1-\rho_{XW}^2}\,H,
\qquad
W=(W_{\mathrm{sig}},V),
\]
where the columns of $A$ are normalized. The response is generated as
\[
Y=X^\top\beta+\alpha H^\top\theta+\varepsilon,
\qquad
\varepsilon\sim\mathcal N(0,\sigma^2).
\]
Here $\alpha$ controls the strength of the privileged signal. Since $H$ is not available at deployment, the deployment target is
\[
\mu(X)=X^\top\beta.
\]
Thus, unlike ordinary test MSE against noisy labels, this construction lets us evaluate the theory-aligned estimation error of the learned deployment predictor.

\paragraph{Feature split.}
The deployment view is $X\in\mathbb R^{d_X}$ and the privileged view is $W\in\mathbb R^{q+d_{\mathrm{noise}}}$. The first $q$ coordinates of $W$ are correlated with the latent privileged signal, while the remaining $d_{\mathrm{noise}}$ coordinates are nuisance dimensions. This design allows us to test whether the rich-view model helps when the privileged signal is informative, and whether it becomes harmful when the privileged view contains many irrelevant coordinates.

\paragraph{Data split.}
For each random seed, we independently sample labeled, unlabeled, and test sets from the same ground-truth model. Unless otherwise stated, we use $d_X=10$, $q=3$, $\rho_{XW}=0.7$, and $\sigma=1$. The signal-strength and nuisance-dimension sweeps use $n=100$ labeled samples, $m=20{,}000$ unlabeled samples, and $n_{\mathrm{test}}=10{,}000$ test samples. The unlabeled-sample-size sweep uses a smaller labeled budget, $n=40$, while varying $m$.

\paragraph{Training procedure.}
All methods use linear predictors with an intercept. Baseline fits the deployment model $f$ using only labeled pairs $(X,Y)$. Two-Stage first fits a rich-view teacher $g$ using labeled pairs $((X,W),Y)$, pseudo-labels the unlabeled examples, and then fits an $X$-only student on the union of labeled and pseudo-labeled examples. Coupled Training solves the linear penalized objective for each value of $\lambda$ in a logarithmic grid. For numerical stability, we use a small ridge penalty in all linear solves.

\paragraph{Hyperparameter selection.}
For the summary sweeps in the main text, $\lambda$ is selected by cross-validation on the labeled set only. The unlabeled samples are used during training but their labels are never used for selecting $\lambda$. We sweep $\lambda$ over a log-spaced grid from $10^{-4}$ to $10^4$.

\paragraph{Evaluation.}
For the controlled sweeps, we report the theory-aligned test estimation error
\[
\mathbb E_{\mathrm{test}}\big[(\hat f(X)-\mu(X))^2\big],
\]
where $\mu(X)=X^\top\beta$ is known from the data-generating process. This removes irreducible label noise from the evaluation and directly measures how well the deployment predictor estimates the target regression function. We average results over multiple random seeds and report mean curves.

\paragraph{Controlled sweeps.}
The signal-strength sweep varies $\alpha$ while holding the nuisance dimension fixed. The nuisance-dimension sweep varies $d_{\mathrm{noise}}$ while holding the privileged signal strength fixed. The unlabeled-sample-size sweep varies $m$ while keeping the labeled budget fixed. These three sweeps correspond respectively to the three panels in Figure~\ref{fig:synthetic_controls}, showing that Coupled Training improves as the privileged signal becomes useful, remains more stable than Two-Stage when nuisance privileged dimensions grow, and benefits from additional unlabeled paired $(X,W)$ samples.

\subsection{Synthetic Binary Classification Diagnostic}\label{app:synthetic_classification_ce}

\paragraph{Dataset.}
We also run a controlled binary classification diagnostic to test the cross-entropy analogue of Coupled Training. The synthetic generator produces deployment features
$X\in\mathbb R^{5}$ and privileged features $W\in\mathbb R^{40}$, with correlated deployment and privileged views. Labels are sampled from a Bernoulli model whose logit depends on both the deployment signal and the privileged signal, with additive logit noise. The default setting uses correlation strength $0.95$, $X$ scale $1.0$, $W$ scale $1.05$, logit-noise standard deviation $0.70$, and unlabeled mean parameter $1.0$.

\paragraph{Feature split.}
The deployment view is $X\in\mathbb R^5$ and the privileged view is $W\in\mathbb R^{40}$. The final predictor must use only $X$ at test time. The privileged block $W$ is available only during training and is intentionally higher-dimensional than $X$, so the experiment tests whether Coupled Training can benefit from privileged information without fully inheriting rich-view noise.

\paragraph{Data split.}
For each random seed, we generate independent labeled, unlabeled, and test samples. We use
\[
n=50,\qquad m=3000,\qquad n_{\mathrm{test}}=6000,
\]
and average results over seeds $\{0,1,2,3,4\}$. The labels of the unlabeled samples are discarded during training and are used only by the data generator.

\paragraph{Models and hyperparameters.}
All methods use linear logistic predictors with intercepts. Let
\[
p_f(x)=\sigma(\bar x^\top\beta),
\qquad
p_g(x,w)=\sigma(\bar z^\top\gamma),
\]
where $\bar x$ and $\bar z$ include intercepts and $\sigma(t)=(1+e^{-t})^{-1}$.
The deployment model uses ridge parameter $\alpha_f=10^{-4}$, and the rich-view model uses ridge parameter $\alpha_g=10^{-1}$. The $X$-only baseline is trained for $500$ gradient steps with learning rate $0.05$. The rich-view teacher and the Two-Stage student are each trained for $700$ gradient steps with learning rate $0.03$.

\paragraph{Training procedure.}
Baseline trains the deployment model using only labeled pairs $(X_L,Y_L)$. Two-Stage first trains a rich-view teacher using labeled pairs $((X_L,W_L),Y_L)$, pseudo-labels the unlabeled examples with teacher probabilities $p_g(X_U,W_U)$, and then trains an $X$-only student on the union of labeled and pseudo-labeled examples.

For Coupled Training, we use the cross-entropy analogue of the alternating updates. For $a\in[0,1]$ and $p\in(0,1)$, define
\[
\mathrm{CE}(a,p)=-a\log p-(1-a)\log(1-p).
\]
Given $g$, the deployment update minimizes
\[
\sum_{i\in L}\mathrm{CE}\big(Y_i,p_f(X_i)\big)
+
\sum_{j\in U}\mathrm{CE}\big(p_g(X_j,W_j),p_f(X_j)\big)
+
\frac{\alpha_f}{2}\|\beta_{-0}\|_2^2,
\]
where the subscript $-0$ excludes the intercept. Given $f$, the rich-view update minimizes
\[
\sum_{j\in U}\mathrm{CE}\big(p_f(X_j),p_g(X_j,W_j)\big)
+
\lambda\sum_{i\in L}\mathrm{CE}\big(Y_i,p_g(X_i,W_i)\big)
+
\frac{\alpha_g}{2}\|\gamma_{-0}\|_2^2.
\]
We run $5$ outer coupled iterations. Each $f$-update and each $g$-update uses $150$ gradient steps with learning rate $0.02$.

\paragraph{Hyperparameter sweep.}
For this diagnostic, $\lambda$ is swept rather than selected by cross-validation. We use $14$ log-spaced values from $10^{-2}$ to $10^{3}$. Thus, the best point in the figure should be interpreted as the best point in the displayed diagnostic sweep, not as a validation-selected tuning parameter.

\paragraph{Evaluation.}
The final test-time predictor is the deployment model $p_f(X)$, which uses only $X$. We classify by thresholding at $1/2$ and report test error
\[
\mathbb P_{\mathrm{test}}\big(\mathbf 1\{p_f(X)\ge 1/2\}\ne Y\big).
\]
All reported curves are averaged over the five random seeds.

\paragraph{Results.}
Figure~\ref{fig:synthetic_classification_ce} shows that the cross-entropy coupled model has an interior optimum in $\lambda$. The best coupled model occurs around $\lambda\approx 5$ and improves slightly over the labeled $X$-only baseline, while the Two-Stage limit performs worse. This mirrors the regression experiments, where moderate coupling can use privileged information, whereas overly strong transfer can propagate rich-view noise into the deployment model.

\begin{figure}[t]
  \centering
  \includegraphics[width=0.6\linewidth]{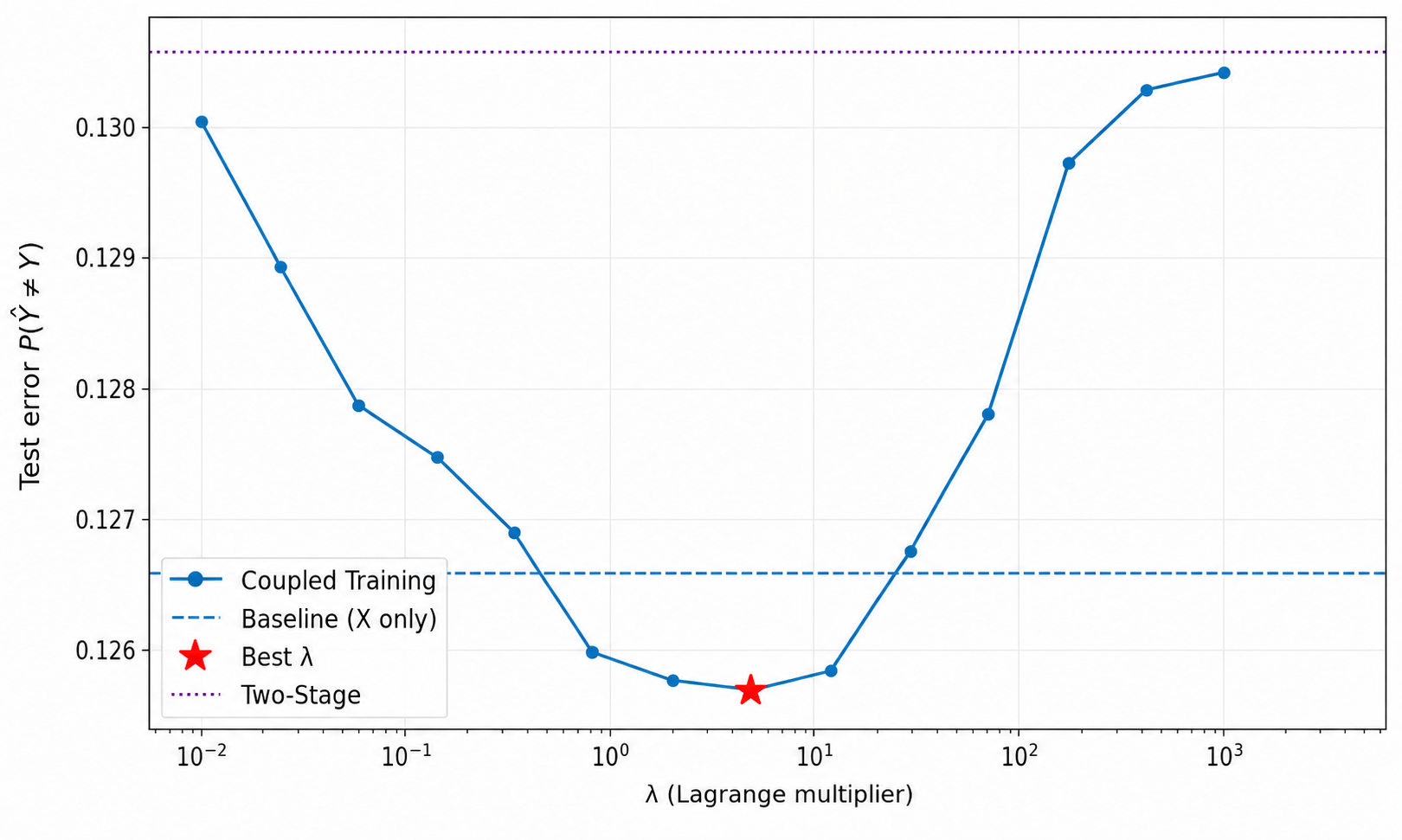}
  \caption{\textbf{Synthetic binary classification diagnostic.} Test $0$--$1$ error for the cross-entropy analogue of Coupled Training as a function of $\lambda$, averaged over seeds $\{0,1,2,3,4\}$.}
  \label{fig:synthetic_classification_ce}
\end{figure}

 \newpage

\section{Proofs}

In what follows, let 
$a = \frac{n}{N}$, $b = \frac{m}{N}$, and $c = b + a\lambda = \frac{m + n\lambda}{N}$, 
and define
\[
s_f(z) \coloneq \frac{b\,f(x) + a\lambda\,\eta(z)}{c}, 
\qquad 
t_g(x) \coloneq a\,\mu(x) + b\,\E[g(Z) \mid X = x].\]

For functions $f \in \mathcal{F}$ and $g \in \mathcal{G}$, we define the (squared) $L^2$ norms
$\|f\|_X^2 = \mathbb{E}[f^2(X)]$ and $\|g\|_Z^2 = \mathbb{E}[g^2(Z)]$. Define the inner products $\langle\cdot, \cdot\rangle_X$ and $\langle\cdot, \cdot\rangle_Z$ analogously.

For fixed points $s^n = (s_1,\dots,s_n)$ in a measurable space $\mathcal S$, let $\hat{P}_n$ denote the empirical measure with respect to these points.
 For a function $h:\mathcal{S}\to\mathbb{R}$, we write $\|h\|_{L^p(\hat{P}_n)} = \big(\frac{1}{n}\sum_{i=1}^n |h(s_i)|^p\big)^{1/p}$. An $L^p$-$\epsilon$-cover of a class $\mathcal{C}$ with respect to $\|\cdot\|_{L^p(\hat P_n)}$ is a finite collection $\{h_1,\dots,h_M\}\subset \mathcal{C}$ such that every $h\in\mathcal{C}$ is within $\varepsilon$ of some $h_j$ under this norm, and the minimal such $M$ is the empirical covering number $\mathcal{N}_p(\epsilon,\mathcal{C},s^n)$.

\subsection{Proof of Theorem \ref{thm:pop-minimizer}}

For a fixed $f \in \mathcal{F}$, let $d$ denote constants that do not depend on $g$ (and may change from line to line). Then, we can write the population loss as follows
\begin{align}
    \mathcal L(f,g;\lambda) &= b\E[(g(Z)-f(X))^2] + a \lambda \E[(Y-g(Z))^2] + d \nonumber\\
    &= b \left(\E g^2(Z) - 2\E g(Z)f(X) + \E f^2(X)\right) + a\lambda\left(\E Y^2-2\E[Yg(Z)]+ \E g^2(Z)\right) +d \nonumber\\
    & =c \E g^2(Z) -2 \E[g(Z)(b f(X)+a \lambda Y)] + d \nonumber\\
    &= c \E g^2(Z) -2 \E\left[g(Z)\E[(b f(X)+a \lambda Y) \mid Z] \right] + d \nonumber\\
    &= c \left( \E g^2(Z)-2 \E[g(Z)s_f(Z)]  \right) +d \nonumber \\
    &=
c\left(
\E[g^2(Z)]-2\E[g(Z)s_f(Z)]+\E[s_f^2(Z)]
\right)+d \\
    &= c\norm{g-s_f}_{Z}^2 + d. \label{eq:loss fixed f}
\end{align}

Thus, for fixed $f\in \mathcal{F}$, the best response over $\mathcal{G}$ is the metric projection
\begin{equation} \label{eq: best_g_fixed_f}
g^\mathcal{G}_{f,\lambda}=\argmin_{g \in \mathcal{G}}\mathcal L(f,g;\lambda) =\argmin_{g\in \mathcal{G}}\norm{g-s_f}_{Z}^2 = \Pi_{\mathcal{G}}s_f.
\end{equation}

For a fixed $g \in \mathcal{G}$,  let $d$ denote constants that do not depend on $f$ (and may change from line to line). Then using same approach to derive \eqref{eq:loss fixed f}, we get
\begin{align}
      \mathcal L(f,g;\lambda) &= a \E[(Y-f(X))^2] + b\E[(g(Z)-f(X))^2] + d \nonumber \\
      &=(a+b) \E f^2(X) - 2 \E \left[f(X) \E[(b g(Z)+a Y) \mid X] \right] + d \nonumber \\
      &= \norm{f-t_g}_X^2 + d. 
      \label{eq: loss fixed g}
\end{align}
Thus, for fixed $g\in \mathcal{G}$, the best response over $\mathcal{F}$ is
\begin{equation} \label{eq:best_f_fixed_g}
f^\mathcal{F}_{g}=\argmin_{f \in \mathcal{F}}\mathcal L(f,g;\lambda) =\argmin_{f\in \mathcal{F}}\norm{f-t_g}_{X}^2 = \Pi_{\mathcal{F}}t_g.    
\end{equation}

Hence, the desired result follows from \eqref{eq: best_g_fixed_f} and \eqref{eq:best_f_fixed_g}.

Furthermore, if $\mu \in \mathcal{F}$ and $\mu \in \mathcal{G}$ (identifying $\mu$ with its lift $\tilde{\mu}(x,w) = \mu(x)$ on $\mathcal{Z}$), and  $\eta \in \mathcal{G}$, set
\[
f_0\coloneqq\mu,
\qquad
g_0\coloneqq\frac{b}{c}\mu+\frac{a\lambda}{c}\eta.
\]
The convexity of $\mathcal G$ gives $g_0\in\mathcal G$, and
$\E[g_0(Z)\mid X]=\mu(X)$. Consequently,
$s_{f_0}=g_0$ and $t_{g_0}=f_0$. For any
$(f,g)\in\mathcal F\times\mathcal G$, write $u=f-f_0$ and $v=g-g_0$.
Using the conditional-mean identities for $\mu$ and $\eta$, together with
$b(g_0-f_0)=a\lambda(\eta-g_0)$, a direct expansion gives
\[
\mathcal L(f,g;\lambda)-\mathcal L(f_0,g_0;\lambda)
=a\|u\|_X^2+b\|v-u\|_Z^2+a\lambda\|v\|_Z^2\ge0.
\]
Thus, $(f_0,g_0)$ is a global minimizer. Since $a,b>0$, equality forces
$u=v=0$, so the minimizer is unique and
\[
f^\star=\mu,
\qquad
g^\star=\frac{m}{m+n\lambda}\mu
+\frac{n\lambda}{m+n\lambda}\eta.
\]

\medskip
Next, before proving Theorem~\ref{thm:rate-fstar}, we present a sequence of intermediate results that will be useful for its proof.

\begin{theorem}\label{thm:excess-fgstar}
The following identity decomposes the excess population loss around
$(f^\star,g^\star)$.
\begin{align}
\mathcal L(\hat f, \hat g;\lambda)-\mathcal L(f^\star,g^\star;\lambda)
&= \|u\|_X^2 + c\|v\|_Z^2 - 2b\langle u,v\rangle_Z\nonumber\\
&\quad + \inner{\nabla_f \mathcal L(f^\star,g^\star;\lambda)}{u}_X
+ \inner{\nabla_g \mathcal L(f^\star,g^\star;\lambda)}{v}_Z, \label{eq:exact-star}
\end{align}
where 
\[
u\coloneq \hat f-f^\star\in \mathcal{F}-f^\star,\qquad v\coloneq \hat g-g^\star\in \mathcal{G}-g^\star.
\]
\end{theorem}
\begin{proof}
    We start by decomposing the loss $\mathcal L(\hat f,\hat g;\lambda)$ into three terms.
    $$\mathcal L(\hat f,\hat g;\lambda) = \circled{1} + \circled{2} + \circled{3},$$
where
    \begin{align*}
        \circled{1} = 
a\E[(Y-\hat f(X))^2], \ \ 
\circled{2}=  b\E[(\hat g(Z)-\hat f(X))^2], \ \ \text{and} \ \ 
 \circled{3} =a\lambda\E[(Y-\hat g(Z))^2].
    \end{align*}
Writing $\hat f= f^\star+u$ and $\hat g= g^\star+ v$, we can expand each term as follows.
\begin{align*}
    \circled{1} &= a\left(\E[(Y-f^\star(X))^2] + \E u^2(X) -2\E[(Y-f^\star(X))u(X)]\right) \\
    &= a\E[(Y-f^\star(X))^2]+a\norm{u}_X^2 -2\E[u(X)(\E(Y\mid X)-f^\star(X))]\\
    &= a\E[(Y-f^\star(X))^2]+a\norm{u}_X^2-2a\inner{\mu-f^\star}{u}_X.
\end{align*}
\begin{align*}
    \circled{2} &= b\left( \E[(g^\star(Z)-f^\star(X))^2]+\E[(v(Z)-u(X))^2] + 2\E[(g^\star(Z)-f^\star(X))(v(Z)-u(X))] \right) \\
    &= b \E[(g^\star(Z)-f^\star(X))^2]+ b \norm{v}_Z^2  
    +b \norm{u}_X^2-2b\inner{u}{v}_Z \\ & \quad + 2b\inner{g^\star-f^\star}{v}_Z-2b\inner{\tilde{g}^\star-f^\star}{u}_X,
\end{align*}
   where $\tilde{g}^\star(x)=\E[g^\star(Z) \mid X=x]$. 
\begin{align*}
    \circled{3} &= a\lambda\E[(Y-g^\star(Z))^2]+a\lambda \norm{v}_Z^2-2a\lambda\E[(Y-g^\star(Z))v(Z)] \\ 
    & =a\lambda\E[(Y-g^\star(Z))^2]+a\lambda \norm{v}_Z^2 -2a\lambda \inner{\eta-g^\star}{v}_Z.
\end{align*}
Using these expansions, observe that
\begin{align*}
    \mathcal L(\hat f,\hat g;\lambda)-\mathcal L(f^\star,g^\star;\lambda) & = (a+b)\norm{u}^2_X+ c\norm{v}_Z^2-2b\inner{u}{v}_Z \\ & \qquad+\inner{2(f^\star-b \tilde{g}^\star-a \mu)}{u}_X + \inner{2(cg^\star-b f^\star-a\lambda\eta)}{v}_Z \\
    &= \norm{u}^2_X+ c\norm{v}_Z^2-2b\inner{u}{v}_Z \\ & \qquad+\inner{2(f^\star-t_{g^\star})}{u}_X + \inner{2c(g^\star-s_{f^\star})}{v}_Z \\
    &= \norm{u}^2_X+ c\norm{v}_Z^2-2b\inner{u}{v}_Z \\ & \qquad+\inner{\nabla_f \mathcal L(f^\star,g^\star;\lambda)}{u}_X
+ \inner{\nabla_g \mathcal L(f^\star,g^\star;\lambda)}{v}_Z,
\end{align*}
where the last equality follows from \eqref{eq:loss fixed f} and \eqref{eq: loss fixed g}, which give
\begin{align*}
    \nabla_g \mathcal L(f,g;\lambda)=2c(g-s_f),    \qquad 
\nabla_f \mathcal L(f,g;\lambda)&=2(f-t_g). 
\end{align*}

\end{proof}


\subsection{Proof of Corollary~\ref{cor:gamma-star}}

Since $(f^\star,g^\star)$ minimizes $\mathcal L$ over the closed convex set $\mathcal{F}\times \mathcal{G}$, the variational inequalities give
\[
\inner{\nabla_f \mathcal L(f^\star,g^\star;\lambda)}{u}_X\geq 0,\qquad
\inner{\nabla_g \mathcal L(f^\star,g^\star;\lambda)}{v}_Z\geq 0
\]
for all $u\in \mathcal{F}-f^\star$, $v\in \mathcal{G}-g^\star$.

Thus, dropping the nonnegative linear terms in \eqref{eq:exact-star} yields
\begin{equation}
\mathcal L(\hat f,\hat g;\lambda)-\mathcal L(f^\star,g^\star;\lambda)
\geq \|u\|_X^2 + c\|v\|_Z^2 - 2b\langle u,v\rangle_Z.
\label{eq:quad-lb}
\end{equation}

The result follows by taking expectation with respect to $\cD$ of both sides \eqref{eq:quad-lb}, then we can further lower bound the right hand by completing the square of the last two terms as done below. 
\begin{align*}
    \E_{\cD}&\norm{u}^2_X+ c\E_{\cD}\norm{v}_Z^2-2b\E_{\cD}\inner{u}{v}_Z\\
    &\ge \E_{\cD}\norm{u}^2_X+ c\E_{\cD}\norm{v}_Z^2-2b\left|\E_{\cD}\inner{u}{v}_Z\right|\\
    &= \E_{\cD}\norm{u}^2_X+ c\E_{\cD}\norm{v}_Z^2-2b\rho_\star\sqrt{\E_{\cD}\norm{u}^2_X \E_{\cD}\norm{v}^2_Z}\\
    &=  \E_{\cD}\norm{u}^2_X +c\left(\E_{\cD}\norm{v}^2_Z-\frac{2b \rho_\star}{c}\sqrt{\E_{\cD}\norm{u}^2_X \E_{\cD}\norm{v}^2_Z}+\frac{b^2 \rho_\star^2}{c^2}\E_{\cD}\norm{u}^2_X\right) -\frac{b^2\rho_\star^2}{c}\E_{\cD}\norm{u}^2_X \\
    &= \gamma_{n, m ,\lambda}(\rho_\star)\E_{\cD}\norm{u}^2_X+c\left(\sqrt{\E_{\cD}\norm{v}^2_Z}-\frac{b \rho_\star}{c}\sqrt{\E_{\cD}\norm{u}^2_X}\right)^2 \\ 
    & \geq \gamma_{n, m ,\lambda}(\rho_\star)\E_{\cD}\norm{u}^2_X.
\end{align*}

\medskip
\bigskip
 
Throughout, we use uppercase letters (e.g., $X_i,Z_i,Y_i$) for random variables and lowercase letters (e.g., $x_i,z_i,y_i$) for their realizations.
We write
$x^{n}=(x_1,\ldots,x_n)$ and $z^{n}=(z_1,\ldots,z_n)$ for the labeled sample locations, and
$x^{m}=(x_{n+1},\ldots,x_N)$ and $z^{m}=(z_{n+1},\ldots,z_N)$ for the unlabeled ones (where $m=N-n$).
We also use $x^{N}=(x_1,\ldots,x_N)$ and $z^{N}=(z_1,\ldots,z_N)$ to denote the full collection of sample points.
Define
\begin{align*}
k_i(f,g)
&\coloneq
(f(x_i)-y_i)^2+\lambda (g(z_i)-y_i)^2
-\bigl[(f^\star(x_i)-y_i)^2+\lambda (g^\star(z_i)-y_i)^2\bigr],\\
p_j(f,g)
&\coloneq
(f(x_j)-g(z_j))^2-(f^\star(x_j)-g^\star(z_j))^2.
\end{align*}
For notational convenience, throughout the remainder of the text we write $\sum_i$ for
$\sum_{i=1}^n$ and $\sum_j$ for $\sum_{j=n+1}^N$. We now state a result adapted from \citet[Theorem~11.4]{gyorfi2002distribution}.

\medskip

\begin{theorem}\label{thm:modified_thm11.4}
Assume that \ref{assump:boundedness} holds. 

Then 
\begin{align*}
    &\mathbb{P}\Bigg\{ \exists (f,g)\in \mathcal{F}\times\mathcal{G}:\E[\hat {\mathcal L}(f,g;\lambda)]-\E[\hat{ \mathcal L}(f^{\star},g^{\star};\lambda)] - \frac{1}{N }\left(\sum\limits_i k_i+\sum\limits_j p_j\right) \\ & \qquad \qquad \ge \epsilon\left(\alpha+\beta+\E[\hat{\mathcal L}(f,g;\lambda)]-\E[\hat{\mathcal L}(f^{\star},g^{\star};\lambda)] \right) \Bigg\} \\
    &\le 10\mathfrak M_{\mathcal{F},\mathcal{G}}^ {n,m,\lambda}\exp \left(-\dfrac{\epsilon^2(1-\epsilon)\alpha N}{512B^2(\lambda+1)}\right),
\end{align*}
where $\alpha,\beta>0,$ $0<\epsilon < 1,$ and
\begin{align*}
    \mathfrak M_{\mathcal{F},\mathcal{G}}^ {n,m,\lambda}\coloneqq & \max_{z^n}\mathcal{N}_1 \left(\dfrac{\epsilon(\alpha+\beta) }{640Ba},\mathcal{F},x^n\right)\cdot\mathcal{N}_1 \left(\dfrac{\epsilon(\alpha+\beta) }{640B(\lambda+1)a},\mathcal{G},z^n\right)+\\&\max_{z^m}\mathcal{N}_1 \left(\dfrac{\epsilon(\alpha+\beta) }{640Bb},\mathcal{F},x^m\right)\cdot\mathcal{N}_1 \left(\dfrac{\epsilon(\alpha+\beta) }{640Bb},\mathcal{G},z^m\right)+\\&\max_{z^N}\mathcal{N}_1\left(\dfrac{\epsilon\beta  }{40B},\mathcal{F},x^N\right)\cdot \mathcal{N}_1\left(\dfrac{\epsilon\beta  }{40B(\lambda+1)},\mathcal{G},z^N\right).
\end{align*}

\end{theorem}

\begin{proof}
\raggedbottom
Throughout this proof, for $(f,g)\in\mathcal F\times\mathcal G$, we interpret
$f$ and $g$ pointwise, writing
$
f \equiv f(X)$, $g \equiv g(Z)$, 
so that expressions such as $(f-Y)^2$, $(g-Y)^2$, and $(f-g)^2$ are understood as
$(f(X)-Y)^2$, $(g(Z)-Y)^2$, and $(f(X)-g(Z))^2$, respectively.

For a sample $(X,Z,Y)$, define the $*$-norm
\[
\|(f,g)-(Y,Y)\|_*^2
\coloneq
a(f-Y)^2 + a\lambda (g-Y)^2 + b(f-g)^2.
\]

Denote
\[
A_{f,g}
\coloneq \|(f,g)-(Y,Y)\|_*^2
  - \|(f^\star,g^\star)-(Y,Y)\|_*^2
= a\,k(f,g) + b\,p(f,g),
\]
where
$
k(f,g)
= (f-Y)^2 + \lambda(g-Y)^2
  - \bigl[(f^\star-Y)^2 + \lambda(g^\star-Y)^2\bigr]$, and 
$p(f,g)
= (f-g)^2 - (f^\star-g^\star)^2.
$

By the definition above, our goal is to bound
\begin{equation}
\mathbb{P} \left\{\exists (f,g)\in \mathcal{F}\times \mathcal{G}:\E[A_{f,g}]-\frac{1}{N}\sum_{i} k_i - \frac{1}{N }\sum_{j} p_j\ge \epsilon \left(\alpha+\beta+\E[A_{f,g}]\right)   \right\} \label{eq:goal to bound}    
\end{equation}
Observe that when $N \le \dfrac{128( \lambda+1)B^2}{\epsilon^2\left(\alpha+\beta\right)}$,
we have
$$10\mathfrak M_{\mathcal{F},\mathcal{G}}^ {n,m,\lambda}\exp \left(-\dfrac{\epsilon^2(1-\epsilon)\alpha N}{512B^2(\lambda+1)}\right)\ge 15\exp \left(-\dfrac{\epsilon^2(1-\epsilon)(\alpha+\beta)N}{512B^2(\lambda+1)}\right)\ge 15\exp \left(-\dfrac{128}{512}\right)> 1,$$ thus it is sufficient to only consider the case when $N > \dfrac{128( \lambda+1)B^2}{\epsilon^2\left(\alpha+\beta\right)}$.

To bound \eqref{eq:goal to bound}, we divide the proof into several steps.

\bigskip
\noindent\textbf{Step 1. Symmetrization}
\\
\raggedbottom
Let $\cD'_L=\{(X'_i, W'_i, Y'_i)\}_{i=1}^n \stackrel{\text{i.i.d.}}{\sim} P_{X,W,Y}$,
and
$\cD'_U=\{Z'_j = (X'_j, W'_j)\}_{j=n+1}^{N} 
  \stackrel{\text{i.i.d.}}{\sim} P_{X,W},
$
be a \emph{ghost} sample independent of $\cD_L$ and $\cD_U$, respectively. 
Denote 
\begin{align*}
k'_i
&= (f(X'_i)-Y'_i)^2 
   + \lambda\,(g(Z'_i)-Y'_i)^2
   - (f^\star(X'_i)-Y'_i)^2 
   - \lambda\,(g^\star(Z'_i)-Y'_i)^2, \\
p'_j
&= (f(X'_j)-g(Z'_j))^2
   - (f^\star(X'_j)-g^\star(Z'_j))^2,
\end{align*}
and define the original and ghost collections of labeled and unlabeled summands
\[
\cD = \{\, k_i,\, p_j : 1 \le i \le n,\; n+1 \le j \le N \,\},
\qquad
\cD' = \{\, k'_i,\, p'_j : 1 \le i \le n,\; n+1 \le j \le N \,\}.
\]
We will show that 
\begin{align}
    &\mathbb{P}\Bigl\{\exists (f,g)\in \mathcal{F}\times \mathcal{G}\colon
\mathbb{E}[A_{f,g}]
-\dfrac{1}{N}\sum_{i} k_i
-\dfrac{1}{N}\sum_{j} p_j
\ge \epsilon\bigl(\alpha+\beta+\mathbb{E}[A_{f,g}]\bigr)\Bigr\}
\nonumber\\
&\le \tfrac{8}{7}\;
\mathbb{P}\Bigl\{\exists (f,g)\in \mathcal{F}\times \mathcal{G}\colon
\dfrac{1}{N}\Bigl(
\sum_{i}\bigl(k_i'-k_i\bigr)
+\sum_{j}\bigl(p_j'-p_j\bigr)
\Bigr)
\ge \tfrac{\epsilon}{2}\bigl(\alpha+\beta+\mathbb{E}[A_{f,g}]\bigr)\Bigr\} \label{eq:symmetrization}
\end{align}
To prove this, consider
\begin{align}
&\Var \left(
  \frac{1}{N}\Bigl(\sum_{i} k_i' + \sum_{j} p_j'\Bigr)
\right) \nonumber= \frac{1}{N ^2}  
   \Var \left(\sum_{i} k_i' + \sum_{j} p_j'\right)\le \frac{1}{N^2}\Bigl(n  \Var(k_1') +  m   \Var(p_1')\Bigr) \nonumber\\[0.25em]
&\le \frac{a}{N }  
\E \left[
  \Bigl((f-y)^2 + \lambda (g-y)^2 - (f^\star-y)^2 - \lambda (g^\star-y)^2\Bigr)^2
\right] + \frac{b}{N }  
\E \left[
  \Bigl((f-g)^2 - (f^\star-g^\star)^2\Bigr)^2
\right]\nonumber\\[0.25em]
&\le \frac{1}{N }  
\E \Bigl[
  a\bigl((f+f^\star-2y)(f-f^\star) + \lambda (g+g^\star-2y)(g-g^\star)\bigr)^2 \nonumber\\
&\qquad
  +   b\bigl((f-g+f^\star-g^\star)(  f-g-f^\star+g^\star)\bigr)^2
\Bigr] \nonumber\\
&\stackrel{(*)}{\le} \frac{1}{N }  
\E \Bigl[
  a\bigl((f+f^\star-2y)^2 + \lambda (g+g^\star-2y)^2\bigr)
   \bigl((f-f^\star)^2 + \lambda (g-g^\star)^2\bigr) \nonumber\\
&\qquad
  +   b\bigl((f-g+f^\star-g^\star)(  f-g-f^\star+g^\star)\bigr)^2
\Bigr]\nonumber\\[0.25em]
&\le\frac{16( \lambda+1)B^2}{N }   
\E \Bigl[
  \bigl(a\bigl((f-f^\star)^2+\lambda(g-g^\star)^2\bigr)
     + b\bigl(f-f^\star+g^\star-g\bigr)^2\bigr)
\Bigr]\nonumber\\[0.25em]
&= \frac{16( \lambda+1)B^2}{N }  
\E \left[\bigl\|(f,g)-(f^\star,g^\star)\bigr\|_*^2\right]\nonumber\\
&\stackrel{(\dag)}\le \frac{16( \lambda+1)B^2}{N }  
\E \left[\|(f,g)-(Y,Y)\|_*^2-\|(f^\star ,g^\star)-(Y, Y)\|_*^2\right] \nonumber\\
&\le  \frac{16( \lambda+1)B^2}{N }  
\E \left[A_{f,g}\right], \label{eq:variance bound}
\end{align}
where $(*)$ follows by applying Cauchy–Schwarz, and $(\dag)$ holds since $(f^\star,g^\star)$ minimizes $\E\|(f,g)-(Y,Y)\|_*^2=\mathcal L(f,g;\lambda)$.
Thus,
\begin{equation}
\E \left[A_{f,g}\right] \ge \frac{1}{16( \lambda+1)B^2} \left(a\E[k^2(f,g)]+b\E[p^2(f,g)]\right)  \label{eq: expectation lower bound}  
\end{equation} 
Let $\mathcal E$ denote the event on the left-hand side of
\eqref{eq:symmetrization}. For every fixed
$(f,g)\in\mathcal F\times\mathcal G$, Chebyshev's inequality and
\eqref{eq:variance bound} give
\begin{align*}
&\mathbb P_{\cD'}\left\{
\E[A_{f,g}]-\frac1N\left(\sum_i k_i'+\sum_j p_j'\right)
>\frac{\epsilon}{2}\bigl(\alpha+\beta+\E[A_{f,g}]\bigr)
\right\}\\
&\quad\le
\frac{64(\lambda+1)B^2\E[A_{f,g}]}
{N\epsilon^2\bigl(\alpha+\beta+\E[A_{f,g}]\bigr)^2}\\
&\quad\le
\frac{16(\lambda+1)B^2}{N\epsilon^2(\alpha+\beta)}
<\frac18,
\end{align*}
where the last two inequalities use $\E[A_{f,g}]\ge0$ and the lower bound on
$N$ above.

Let $\mathcal H$ be the event that there exists the same
$(f,g)\in\mathcal F\times\mathcal G$ satisfying both
\[
\E[A_{f,g}]-\frac1N\left(\sum_i k_i+\sum_j p_j\right)
\ge\epsilon\bigl(\alpha+\beta+\E[A_{f,g}]\bigr)
\]
and
\[
\E[A_{f,g}]-\frac1N\left(\sum_i k_i'+\sum_j p_j'\right)
\le\frac{\epsilon}{2}
\bigl(\alpha+\beta+\E[A_{f,g}]\bigr).
\]
For every realization $\cD\in\mathcal E$, the conditional section
$\mathcal H_{\cD}$ contains the ghost-good event corresponding to any witness
of $\mathcal E$. Hence
\[
\mathbb P(\mathcal H\mid\cD)\ge\frac78
\quad\text{on }\mathcal E,
\qquad
\mathbb P(\mathcal H)\ge\frac78\mathbb P(\mathcal E).
\]
On $\mathcal H$, the same pair $(f,g)$ satisfies
\[
\frac1N\left(
\sum_i(k_i'-k_i)+\sum_j(p_j'-p_j)
\right)
\ge\frac{\epsilon}{2}
\bigl(\alpha+\beta+\E[A_{f,g}]\bigr).
\]
Therefore \eqref{eq:symmetrization} follows.

\bigskip
\noindent\textbf{Step 2. Population to empirical}

Define
\[
\widehat Q(f,g)
\coloneq \frac{1}{N}\left(\sum_i k_i^2+\sum_j p_j^2\right),
\qquad
\widehat Q'(f,g)
\coloneq \frac{1}{N}\left(\sum_i (k_i')^2+\sum_j (p_j')^2\right),
\]
and set
\[
Q(f,g)\coloneq \E[ak^2+bp^2],
\qquad
\tau\coloneq B^2(\lambda+1)(\alpha+\beta)=\tau_\alpha+\tau_\beta,
\qquad
\tau_\alpha\coloneq B^2(\lambda+1)\alpha,
\quad
\tau_\beta\coloneq B^2(\lambda+1)\beta.
\]

Next, we show that
\begin{align}
         &\mathbb{P} \left\{\exists (f,g)\in \mathcal{F}\times \mathcal{G}: \dfrac{1}{N } \left(\sum\limits_{i} (k_i' - k_i)+ \sum\limits_{j} (p_j'-p_j)\right)\ge \dfrac{\epsilon}{2} \left(\alpha+\beta+\E\Bigl[A_{f,g}\Bigr]\right)\right\} \nonumber\\
         &\le\mathbb{P} \Bigg\{\exists (f,g)\in \mathcal{F}\times \mathcal{G}:\dfrac{1}{N } \left(\sum\limits_{i} (k_i' - k_i)+ \sum\limits_{j} (p_j'-p_j)\right)\ge \dfrac{\epsilon}{2} \left(\alpha+\beta\right) \nonumber\\&\qquad \qquad + \dfrac{\epsilon}{64(\lambda+1)B^2}\dfrac{1-\epsilon}{1+\epsilon} \left(\widehat Q(f,g)+\widehat Q'(f,g)-\dfrac{2\epsilon}{1-\epsilon}\tau\right)\Bigg\}\nonumber \\& + 8\E \left[\mathcal{N}_1 \left(\dfrac{\epsilon\tau}{80aB^2(\lambda+1)},\mathcal{G}_k,z^n\right)+\mathcal{N}_1 \left(\dfrac{\epsilon\tau}{80bB^2(\lambda+1)},\mathcal{G}_p,z^m\right)\right]\exp \left(-\dfrac{\epsilon^2\tau N }{480B^4(\lambda+1)^2}\right) \label{eq: pop to empirical}.
\end{align}

To prove \eqref{eq: pop to empirical}, we first present the following lemma.

\raggedbottom

\begin{lemma}\label{lem:pop to empirical}
Let $\mathcal{G}_k=\{k(f,g): (f,g)\in \mathcal{F}\times \mathcal{G}\}, \ \mathcal{G}_p=\{p(f,g):(f,g)\in \mathcal{F}\times \mathcal{G}\}$. $(z^n,y^n),z^m$ are the samples used to generate $k_i,p_j$ respectively. Then
\begin{align*}
    &\mathbb{P} \left\{\exists (f,g)\in \mathcal{F}\times \mathcal{G}:\dfrac{\widehat Q(f,g)-Q(f,g)}{\tau+\widehat Q(f,g)+Q(f,g)} > \epsilon\right\}\\&\le 4\E \left[\mathcal{N}_1 \left(\dfrac{\epsilon\tau}{80aB^2(\lambda+1)},\mathcal{G}_k,z^n\right)+\mathcal{N}_1 \left(\dfrac{\epsilon\tau}{80bB^2(\lambda+1)},\mathcal{G}_p,z^m\right)\right]\exp \left(-\dfrac{\epsilon^2\tau N }{480B^4(\lambda+1)^2}\right).
\end{align*}  
\end{lemma} 
\begin{proof}
Write
\[
P_n k^2\coloneq\frac1n\sum_i k_i^2,
\qquad
P_m p^2\coloneq\frac1m\sum_j p_j^2.
\]
Since $\widehat Q=aP_nk^2+bP_mp^2$ and
$Q=a\E[k^2]+b\E[p^2]$, the event in the lemma is contained in the union of
\begin{align*}
&\left\{\exists(f,g):
\frac{P_nk^2-\E[k^2]}
{\tau/(2a)+P_nk^2+\E[k^2]}>\epsilon\right\},\\
&\left\{\exists(f,g):
\frac{P_mp^2-\E[p^2]}
{\tau/(2b)+P_mp^2+\E[p^2]}>\epsilon\right\}.
\end{align*}
Indeed, if neither block event occurs, multiplying the two inequalities by
$a$ and $b$ and adding them gives
$\widehat Q-Q\le\epsilon(\tau+\widehat Q+Q)$.

Introduce the nonnegative squared classes
\[
\mathcal H_k\coloneq\{k^2:k\in\mathcal G_k\},
\qquad
\mathcal H_p\coloneq\{p^2:p\in\mathcal G_p\}.
\]
Under Assumption~\ref{assump:boundedness},
\[
|k|\le 4B^2(\lambda+1),
\qquad
|p|\le 4B^2,
\]
and hence
\[
|k^2-\widetilde k^2|
\le 8B^2(\lambda+1)|k-\widetilde k|,
\qquad
|p^2-\widetilde p^2|
\le 8B^2|p-\widetilde p|.
\]
The two squared classes have envelopes $16B^4(\lambda+1)^2$ and
$16B^4$, respectively. For a nonnegative class $\mathcal H$ bounded by $L$,
an i.i.d.\ sample of size $r$, and $\nu>0$,
\citet[Theorem 11.6]{gyorfi2002distribution} gives
\begin{align*}
&\mathbb P\left\{\exists h\in\mathcal H:
\frac{P_rh-\E h}{\nu+P_rh+\E h}>\epsilon\right\}\\
&\quad\le4\E\mathcal N_1\left(
\frac{\epsilon\nu}{5},\mathcal H,S^r\right)
\exp\left(-\frac{3\epsilon^2\nu r}{40L}\right).
\end{align*}
If necessary, this form follows by rescaling the class. Apply it to
$\mathcal H_k$ with $\nu=\tau/(2a)$ and to $\mathcal H_p$ with
$\nu=\tau/(2b)$. Since $n=aN$ and $m=bN$, the lemma probability is at most
\begin{align*}
&4\E\mathcal N_1\left(\frac{\epsilon\tau}{10a},
\mathcal H_k,z^n\right)
\exp\left(-\frac{3\epsilon^2\tau N}
{1280B^4(\lambda+1)^2}\right)\\
&\quad+4\E\mathcal N_1\left(\frac{\epsilon\tau}{10b},
\mathcal H_p,z^m\right)
\exp\left(-\frac{3\epsilon^2\tau N}{1280B^4}\right).
\end{align*}
The preceding Lipschitz bounds pull these squared-class covers back to
$\mathcal G_k$ and $\mathcal G_p$. Since $\lambda+1\ge1$ and
$3/1280\ge1/480$, weakening both terms to a common radius and exponent gives
\begin{align*}
&\mathbb{P} \left\{\exists (f,g)\in \mathcal{F}\times\mathcal{G}:
\frac{\widehat Q(f,g)-Q(f,g)}
{\tau+\widehat Q(f,g)+Q(f,g)} > \epsilon\right\}\\
&\quad\le 4\E \left[
\mathcal{N}_1 \left(
\frac{\epsilon\tau}{80aB^2(\lambda+1)},\mathcal{G}_k,z^n\right)
+\mathcal{N}_1 \left(
\frac{\epsilon\tau}{80bB^2(\lambda+1)},\mathcal{G}_p,z^m\right)
\right]
\exp \left(-\frac{\epsilon^2\tau N}{480B^4(\lambda+1)^2}\right).
\end{align*}
\end{proof}

Next, define the events
\begin{align*}
\mathcal{S}_1
&\coloneq
\Biggl\{
\begin{aligned}
\exists (f,g)\in\mathcal F\times\mathcal G:
&\frac{1}{N}\Bigl(\sum_i (k_i'-k_i)+\sum_j (p_j'-p_j)\Bigr)\ge \frac{\epsilon}{2}\Bigl(\alpha+\beta+\E[A_{f,g}]\Bigr)
\end{aligned}
\Biggr\},\\[2mm]
\mathcal{S}_2
&\coloneq
\Bigl\{\forall (f,g)\in\mathcal F\times\mathcal G:
\widehat Q(f,g)-Q(f,g)
\le \epsilon\bigl(\tau+\widehat Q(f,g)+Q(f,g)\bigr)\Bigr\},\\[2mm]
\mathcal{S}_3
&\coloneq
\Bigl\{\forall (f,g)\in\mathcal F\times\mathcal G:
\widehat Q'(f,g)-Q(f,g)
\le \epsilon\bigl(\tau+\widehat Q'(f,g)+Q(f,g)\bigr)\Bigr\}.
\end{align*}

We have the set inclusion $
\mathcal{S}_1 \subset (\mathcal{S}_1\cap \mathcal{S}_2\cap \mathcal{S}_3)\ \cup\ \mathcal{S}_2^c\ \cup\ \mathcal{S}_3^c
$, hence
\begin{align*}
\mathbb{P}(\mathcal{S}_1)
&\le \mathbb{P}(\mathcal{S}_1\cap \mathcal{S}_2\cap \mathcal{S}_3)+\mathbb{P}(\mathcal{S}_2^c)+\mathbb{P}(\mathcal{S}_3^c) \\ &\le \mathbb{P}(\mathcal{S}_1\cap \mathcal{S}_2\cap \mathcal{S}_3)+2\,\mathbb{P}(\mathcal{S}_2^c) \\
&\le\mathbb{P} \Bigg\{\exists (f,g)\in \mathcal{F}\times \mathcal{G}: \dfrac{1}{N } \left(\sum\limits_{i} (k_i' - k_i)+ \sum\limits_{j} (p_j'-p_j)\right)\ge \dfrac{\epsilon}{2} \left(\alpha+\beta\right)\\& \qquad + \dfrac{\epsilon}{64(\lambda+1)B^2}\dfrac{1-\epsilon}{1+\epsilon} \left(\widehat Q(f,g)+\widehat Q'(f,g)-\dfrac{2\epsilon}{1-\epsilon}\tau\right)\Bigg\}\\ &+ 8\E \left[\mathcal{N}_1 \left(\dfrac{\epsilon\tau}{80aB^2(\lambda+1)},\mathcal{G}_k,z^n\right)+\mathcal{N}_1 \left(\dfrac{\epsilon\tau}{80bB^2(\lambda+1)},\mathcal{G}_p,z^m\right)\right]\exp \left(-\dfrac{\epsilon^2\tau N }{480B^4(\lambda+1)^2}\right),
\end{align*}

where the last inequality uses Lemma~\ref{lem:pop to empirical} and the fact that $\mathcal{S}_2$ implies
\[
\E \left[A_{f,g}\right]\stackrel{(*)}{\ge}\frac{1}{16( \lambda+1)B^2} Q(f,g) \ge \frac{1}{16( \lambda+1)B^2}\dfrac{1-\epsilon}{1+\epsilon} \left(\widehat Q(f,g) -\dfrac{\epsilon}{1-\epsilon}\tau\right)
\]
Similarly for $\cD'$, $\mathcal{S}_3$ implies 
\[
\E \left[A_{f,g}\right]\stackrel{(*)}{\ge} \frac{1}{16( \lambda+1)B^2} Q(f,g) \ge \frac{1}{16( \lambda+1)B^2}\dfrac{1-\epsilon}{1+\epsilon} \left(\widehat Q'(f,g) -\dfrac{\epsilon}{1-\epsilon}\tau\right),
\]
where $(*)$ follows by \eqref{eq: expectation lower bound}. Therefore \eqref{eq: pop to empirical} holds.

\noindent\textbf{Step 3. Rademacher variables.}
Let $\{U_\ell\}_{\ell=1}^N$ be independent Rademacher variables,
independent of both the original and ghost samples.
Set
$c_\epsilon\coloneq
\dfrac{\epsilon}{64B^2(\lambda+1)}\dfrac{1-\epsilon}{1+\epsilon}$.
Then
\begin{align}
       &\mathbb{P} \Bigg\{\exists (f,g)\in \mathcal{F}\times \mathcal{G}: \dfrac{1}{N } \left(\sum\limits_{i} (k_i' - k_i)+ \sum\limits_{j} (p_j'-p_j)\right)\ge \dfrac{\epsilon}{2} \left(\alpha+\beta\right) \nonumber \\  &\qquad+ c_\epsilon \left(\widehat Q(f,g)+\widehat Q'(f,g)-\dfrac{2\epsilon}{1-\epsilon}\tau\right)\Bigg\}
    \nonumber \\ &\le  2\mathbb{P} \Bigg\{\exists (f,g)\in \mathcal{F}\times \mathcal{G}: \dfrac{1}{N } \left| \sum\limits_{i} U_i k_i + \sum\limits_{j}U_j p_j \right| \ge \dfrac{\epsilon}{4} \left(\alpha+\beta\right) \nonumber\\&\qquad+ c_\epsilon \left(\widehat Q(f,g) -\dfrac{\epsilon}{1-\epsilon}\tau\right)\Bigg\} \label{eq: rademacher},
\end{align}
For brevity, write
\begin{align*}
R_{\cD}(f,g)
&\coloneq\frac{\epsilon}{4}(\alpha+\beta)
+c_\epsilon\left(\widehat Q(f,g)-\frac{\epsilon}{1-\epsilon}\tau\right),\\
T_{\cD}(f,g;U)
&\coloneq\frac1N\left(\sum_iU_i k_i+\sum_jU_jp_j\right),
\end{align*}
and define $R_{\cD'}$ and $T_{\cD'}$ analogously. Independently swapping the
two observations within each original--ghost pair shows that the left-hand
probability in \eqref{eq: rademacher} equals
\[
\mathbb P\left\{\exists(f,g):
T_{\cD'}(f,g;U)-T_{\cD}(f,g;U)
\ge R_{\cD}(f,g)+R_{\cD'}(f,g)\right\},
\]
because $\widehat Q+\widehat Q'$ is invariant under these swaps. By the
triangle inequality and a union bound, this is at most
\begin{align*}
&\mathbb P\left\{\exists(f,g):
|T_{\cD}(f,g;U)|\ge R_{\cD}(f,g)\right\}\\
&\quad+\mathbb P\left\{\exists(f,g):
|T_{\cD'}(f,g;U)|\ge R_{\cD'}(f,g)\right\}.
\end{align*}
The two probabilities are equal by identical distribution, which proves
\eqref{eq: rademacher}.

\bigskip
\noindent\textbf{Step 4. Covering numbers.} For each realization $\cD$,
let $\Gamma_\delta=\Gamma_\delta(\cD)$ be a smallest subset of
$\mathcal{F}\times\mathcal{G}$ such that,
for every $(f,g)\in\mathcal{F}\times\mathcal{G}$, there exists $(\tilde f,\tilde g)\in\Gamma_\delta$
satisfying
\[
\frac{1}{N}\left(
\sum_i \left|k_i(f,g)-k_i(\tilde f,\tilde g)\right|
+
\sum_j \left|p_j(f,g)-p_j(\tilde f,\tilde g)\right|
\right)
\le \delta.
\]

For notational simplicity, we write $\tilde k_i$ and $\tilde p_j$ in place of
$k_i(\tilde f,\tilde g)$ and $p_j(\tilde f,\tilde g)$, respectively. Define
\[
\widetilde Q\coloneq\frac{1}{N}\left(\sum_i\tilde k_i^2+\sum_j\tilde p_j^2\right).
\]
Conditionally on $\cD$, we show that
\begin{align}
&\mathbb{P}_U \left\{
\begin{aligned}
&\exists (f,g)\in \mathcal{F}\times \mathcal{G}: \dfrac{1}{N }
\left| \sum\limits_{i} U_i k_i + \sum\limits_{j}U_j p_j \right|\ge \dfrac{\epsilon}{4} \left(\alpha+\beta\right)
+c_\epsilon\left(\widehat Q(f,g)-\dfrac{\epsilon}{1-\epsilon}\tau\right)
\end{aligned}
\,\middle|\,\cD\right\} \nonumber\\
&\le  \left| \Gamma_{\frac{\epsilon\beta}{5}} \right|
\max _{(\tilde f,\tilde g)\in \Gamma_{\frac{\epsilon\beta}{5}}}
\mathbb{P}_U \left\{
\begin{aligned}
&\dfrac{1}{N } \left| \sum\limits_{i} U_i \tilde k_i
+ \sum\limits_{j}U_j \tilde p_j \right|
 \ge \dfrac{\epsilon}{4} \alpha
+c_\epsilon\left(\widetilde Q-\dfrac{\epsilon}{1-\epsilon}\tau_\alpha\right)
\end{aligned}
\,\middle|\,\cD\right\}. \label{eq: covering numbers}
\end{align}

Indeed, conditionally on $\cD$, we have
\begin{align*}
&\mathbb{P}_U \left\{
\begin{aligned}
&\exists (f,g)\in \mathcal{F}\times \mathcal{G}: \dfrac{1}{N }
\left| \sum\limits_{i} U_i k_i + \sum\limits_{j}U_j p_j \right|\ge \dfrac{\epsilon}{4} \left(\alpha+\beta\right)
+c_\epsilon\left(\widehat Q(f,g)-\dfrac{\epsilon}{1-\epsilon}\tau\right)
\end{aligned}
\,\middle|\,\cD\right\}\\
&\stackrel{(*)}{\le}
\left| \Gamma_\delta \right| \max _{(\tilde f,\tilde g)\in \Gamma_\delta}
\mathbb{P}_U \Bigg\{
\begin{aligned}
&\dfrac{1}{N } \left| \sum\limits_{i} U_i \tilde k_i
+ \sum\limits_{j}U_j \tilde p_j \right| + \delta\ge \dfrac{\epsilon}{4} \left(\alpha+\beta\right)
+c_\epsilon\left(\widetilde Q
-8B^2(\lambda+1)\delta
-\dfrac{\epsilon}{1-\epsilon}\tau\right)
\end{aligned}
\,\Biggm|\,\cD\Bigg\}
    \\
&\stackrel{(\dag)}{\le}
\left| \Gamma_{\frac{\epsilon\beta}{5}} \right|
\max _{(\tilde f,\tilde g)\in \Gamma_{\frac{\epsilon\beta}{5}}}
\mathbb{P}_U \left\{
\begin{aligned}
&\dfrac{1}{N } \left| \sum\limits_{i} U_i \tilde k_i
+ \sum\limits_{j}U_j \tilde p_j \right| \ge \dfrac{\epsilon}{4} \alpha
+c_\epsilon\left(\widetilde Q-\dfrac{\epsilon}{1-\epsilon}\tau_\alpha\right)
\end{aligned}
\,\middle|\,\cD\right\},
\end{align*}  
where $(*)$ follows by observing that 
\begin{align*}
    &\dfrac{1}{N } \left| \sum\limits_{i} U_i k_i + \sum\limits_{j}U_j p_j \right|- \dfrac{1}{N } \left| \sum\limits_{i} U_i \tilde k_i + \sum\limits_{j}U_j \tilde p_j \right| \le   \dfrac{1}{N } \left| \sum\limits_{i} U_i (k_i-\tilde k_i) + \sum\limits_{j}U_j (p_j-\tilde p_j) \right|\\&\le\dfrac{1}{N }\left(\sum\limits_{i} \left| k_i(f,g)-k_i(\tilde f,\tilde g) \right| +\sum\limits_{j} \left| p_j(f,g)-p_j(\tilde f,\tilde g) \right|\right)
    \le\delta,
\end{align*} 
and 
\begin{align*}
   &\dfrac{1}{N }\left(\sum\limits_i \tilde k_i^2+\sum\limits_j\tilde p_j^2\right)-\dfrac{1}{N }\left(\sum\limits_i k_i^2+\sum\limits_j p_j^2\right)\\
   &\qquad\le \dfrac{1}{N }\left(\sum\limits_i (\tilde k_i-k_i)(\tilde k_i+k_i)+\sum\limits_j (\tilde p_j-p_j)(\tilde p_j+p_j)\right) \\
   &\qquad\le\dfrac{8B^2(\lambda+1)}{N}\left(\sum\limits_{i} \left| k_i(f,g)-k_i(\tilde f,\tilde g) \right| +\sum\limits_{j} \left| p_j(f,g)-p_j(\tilde f,\tilde g) \right|\right)\le8B^2(\lambda+1)\delta.
\end{align*}
$(\dag)$ follows by setting $\delta=\dfrac{\epsilon\beta}{5}$, using $\tau=\tau_\alpha+\tau_\beta$, and observing that
\begin{align*}
    &\dfrac{\epsilon\beta}{4}-\dfrac{\epsilon\beta}{5}-c_\epsilon \left(8B^2(\lambda+1)\dfrac{\epsilon\beta}{5} +\dfrac{\epsilon}{1-\epsilon}\tau_\beta\right)  \\
    &=\dfrac{\epsilon\beta}{20}-\dfrac{\epsilon^2(1-\epsilon)}{40(1+\epsilon)} \beta -\dfrac{\epsilon^2}{64(1+\epsilon)}\beta \ge 0.
\end{align*}

\bigskip
\noindent\textbf{Step 5. Bernstein's inequality.}
Conditionally on $\cD$, for every fixed
$(\tilde f,\tilde g)\in\Gamma_{\epsilon\beta/5}(\cD)$,
Bernstein's inequality gives
\begin{align}
&\mathbb{P}_U \left\{
\begin{aligned}
&\dfrac{1}{N } \left| \sum\limits_{i} U_i \tilde k_i
+ \sum\limits_{j}U_j \tilde p_j \right|\ge \dfrac{\epsilon}{4} \alpha
+c_\epsilon\left(\widetilde Q-\dfrac{\epsilon}{1-\epsilon}\tau_\alpha\right)
\end{aligned}
\,\middle|\,\cD\right\} \nonumber\\
&\qquad\le  2\exp\left(
 -\frac{N\epsilon^{2}(1-\epsilon)\alpha}
       {512(\lambda+1)B^{2}}\right). \label{eq: bernstein bound}
\end{align}
To verify this bound, define
\begin{align*}
\sigma^2&\coloneq\dfrac{1}{N }\left(\sum\limits_i \tilde k_i^2+\sum\limits_j\tilde p_j^2\right)= \dfrac{1}{N }\Var\left( \sum\limits_{i} U_i \tilde k_i + \sum\limits_{j}U_j \tilde p_j \right), \\
A_1&\coloneq\dfrac{\epsilon}{4} \alpha- \dfrac{\epsilon^2}{64(1+\epsilon)} \alpha, \ \text{and} \quad A_2\coloneq\dfrac{\epsilon}{64(\lambda+1)B^2}\dfrac{1-\epsilon}{1+\epsilon}.
\end{align*}
By Bernstein's inequality and the independence of the Rademacher variables,
we have
\begin{align*}
    &\mathbb{P}_U \Bigg\{\dfrac{1}{N } \left| \sum\limits_{i} U_i \tilde k_i + \sum\limits_{j}U_j \tilde p_j \right|  \ge \dfrac{\epsilon}{4} \alpha+c_\epsilon\left(\widetilde Q-\dfrac{\epsilon}{1-\epsilon}\tau_\alpha\right)\,\Biggm|\,\cD\Bigg\}\\&= \mathbb{P}_U \left\{\dfrac{1}{N } \left| \sum\limits_{i} U_i \tilde k_i + \sum\limits_{j}U_j \tilde p_j \right|  \ge A_1 + A_2 \sigma^2\,\middle|\,\cD\right\}\\&\le  2\exp \left(-\dfrac{N (A_1 + A_2 \sigma^2)^2}{2\sigma^2 +\dfrac{8}{3} B^2(\lambda+1) (A_1 + A_2 \sigma^2)}\right)= 2\exp \left(-\dfrac{3NA_2}{8B^2(\lambda+1)}\cdot \dfrac{(A_1/A_2 +  \sigma^2)^2}{ \left(\dfrac{3}{4B^2A_2(\lambda+1)}+1\right) \sigma^2 +\dfrac{A_1   }{A_2}}\right)\\
  &\stackrel{(\dag)}{\le}  2\exp \left(-\dfrac{3NA_2}{8B^2(\lambda+1)}\cdot \frac{192\,\epsilon(1+\epsilon)(\lambda+1)B^2(16+15\epsilon)\alpha}
{\bigl(\epsilon(1-\epsilon)+48(1+\epsilon)\bigr)^2}\right)\\&\le  2\exp \left(-\dfrac{3N A_2}{8B^2(\lambda+1)}\cdot \frac{192\,\epsilon(1+\epsilon)\,16(\lambda+1)B^2\alpha}
{\bigl(\epsilon(1-\epsilon)+48(1+\epsilon)\bigr)^2}\right) \\&=2 
\exp\left(
 -\frac{3N}{8B^2(\lambda+1)}
 \cdot\frac{\epsilon(1-\epsilon)}{64(\lambda+1)B^2(1+\epsilon)}
\cdot\frac{192\,\epsilon(1+\epsilon)\,16(\lambda+1)B^2 \alpha}
           {\bigl(\epsilon(1-\epsilon)+48(1+\epsilon)\bigr)^2}
\right)  \\&=2\exp\left(
 -\frac{18N\,\epsilon^{2}(1-\epsilon)\alpha}
       {(\lambda+1)B^{2}\,\bigl(\epsilon(1-\epsilon)+48(1+\epsilon)\bigr)^{2}}
\right)\le 2\exp\left(
 -\frac{N\epsilon^{2}(1-\epsilon)\alpha}
       {512(\lambda+1)B^{2}}
\right),  
\end{align*}
where the last inequality follows using
$\max_{\,0\le \epsilon\le 1} \bigl(\epsilon(1-\epsilon)+48(1+\epsilon)\bigr)^2=9216.
$ $(\dag)$ follows by noting that for any $u > 0$, $b>2$,  we have $$\frac{(a+u)^2}{a + b\,u}
\;\ge\;
\frac{\bigl(a + \tfrac{b-2}{b}\,a\bigr)^2}{a + b\,\tfrac{b-2}{b}\,a}
\;=\;
4a\,\frac{b-1}{b^{2}}\,,$$
which yields that, 
\begin{align*}
\frac{\bigl(A_1/A_2+\sigma^2\bigr)^2}
{\left(\dfrac{3}{4B^2A_2(\lambda+1)}+1\right)\sigma^2+\dfrac{A_1}{A_2}}  \ge 
\frac{\dfrac{192A_1(1+\epsilon)}{\epsilon(1-\epsilon)}}
{A_2\left(1+\dfrac{48(1+\epsilon)}{\epsilon(1-\epsilon)}\right)^2}
&=
\frac{192\,\alpha\,\epsilon(1+\epsilon)(\lambda+1)B^2(16+15\epsilon)}
{\bigl(\epsilon(1-\epsilon)+48(1+\epsilon)\bigr)^2} \\&\ge \frac{192\,\epsilon(1+\epsilon)\,16(\lambda+1)B^2 \alpha}
{\bigl(\epsilon(1-\epsilon)+48(1+\epsilon)\bigr)^2}.
\end{align*} 
Therefore \eqref{eq: bernstein bound} holds.

Averaging \eqref{eq: covering numbers} over $\cD$ and using
\eqref{eq: bernstein bound}, we obtain
\begin{align}
&\mathbb P\left\{\exists(f,g)\in\mathcal F\times\mathcal G:
\frac1N\left|\sum_iU_i k_i+\sum_jU_jp_j\right|
\ge\frac\epsilon4(\alpha+\beta)
+c_\epsilon\left(\widehat Q(f,g)-\frac\epsilon{1-\epsilon}\tau\right)
\right\}\nonumber\\
&\quad\le
2\,\E_{\cD}\left|\Gamma_{\epsilon\beta/5}(\cD)\right|
\exp\left(-\frac{N\epsilon^2(1-\epsilon)\alpha}
{512(\lambda+1)B^2}\right).
\label{eq:averaged-cover-bound}
\end{align}

\bigskip
\noindent\textbf{Step 6. Final bound.}
\raggedbottom
Observe that if we have   
\begin{align*}
&\frac{1}{N}\left(\sum\limits_{i} \left| f(X_i)-\tilde f(X_i) \right|+\sum\limits_{j} \left| f(X_j)-\tilde f(X_j) \right| \right) \le \dfrac{\epsilon\beta}{40B}, \\
&\frac{1}{N}\left(\sum\limits_{i} \left| g(Z_i)-\tilde g(Z_i) \right|+\sum\limits_{j} \left| g(Z_j)-\tilde g(Z_j) \right| \right)\le \dfrac{\epsilon\beta}{40B(\lambda+1)},
\end{align*}
then we have    
\begin{align*}
    &\dfrac{1}{N }\left(\sum\limits_{i} \left| k_i(f,g)-k_i(\tilde f,\tilde g) \right| +\sum\limits_{j} \left| p_j(f,g)-p_j(\tilde f,\tilde g) \right|\right) \\&\le\dfrac{1}{N }\left(\sum\limits_{i} 4B\left| f(X_i)-\tilde f(X_i) \right| +4\lambda B \left| g(Z_i)-\tilde g(Z_i) \right| +\sum\limits_{j} 4B\left| f(X_j)-\tilde f(X_j) \right|+4B \left| g(Z_j)-\tilde g(Z_j) \right| \right) \le\dfrac{\epsilon\beta}{5},
\end{align*}
which implies that   
$
\left| \Gamma_{\frac{\epsilon\beta}{5}} \right| \le \mathcal{N}_1\left(\dfrac{\epsilon\beta}{40B},\mathcal{F},x^N\right)\cdot \mathcal{N}_1\left(\dfrac{\epsilon\beta}{40(\lambda+1) B},\mathcal{G},z^N\right)
$.
Thus $\E_{\cD}|\Gamma_{\epsilon\beta/5}(\cD)|\le\mathcal M_1$,
with $\mathcal M_1$ defined below.

Combining the steps, we have the following.
\begin{align*}
    &\mathbb{P} \left\{\exists (f,g)\in \mathcal{F}\times \mathcal{G}:\E[A_{f,g}]-\frac{1}{N }\sum_{i} k_i - \frac{1}{N }\sum_{j} p_j\ge \epsilon \left(\alpha+\beta+\E[A_{f,g}]\right)   \right\}
    \\&\le  \frac{32}{7}\mathcal{M}_1\exp\left(
 -\frac{N\epsilon^{2}(1-\epsilon)\alpha} {512(\lambda+1)B^{2}}\right)+\frac{64}{7}\mathcal{M}_2\exp \left(-\dfrac{\epsilon^2\tau N}{480B^4(\lambda+1)^2}\right)
\end{align*}
 where $\mathcal{M}_1=\max_{z^N}\mathcal{N}_1\left(\dfrac{\epsilon\beta}{40B},\mathcal{F},x^N\right)\cdot \mathcal{N}_1\left(\dfrac{\epsilon\beta}{40(\lambda+1) B},\mathcal{G},z^N\right)$, and 
 \begin{align*}
   \mathcal{M}_2={}&\max_{z^n}\mathcal{N}_1 \left(\dfrac{\epsilon\tau}{80aB^2(\lambda+1)},\mathcal{G}_k,z^n\right)+\max_{z^m}\mathcal{N}_1 \left(\dfrac{\epsilon\tau}{80bB^2(\lambda+1)},\mathcal{G}_p,z^m\right)\\
   \le{}&\max_{z^n}\mathcal{N}_1 \left(\dfrac{\epsilon(\alpha+\beta)}{640Ba},\mathcal{F},x^n\right)\cdot\mathcal{N}_1 \left(\dfrac{\epsilon(\alpha+\beta)}{640B(\lambda+1)a},\mathcal{G},z^n\right)\\
   &+\max_{z^m}\mathcal{N}_1 \left(\dfrac{\epsilon(\alpha+\beta)}{640Bb},\mathcal{F},x^m\right)\cdot\mathcal{N}_1 \left(\dfrac{\epsilon(\alpha+\beta)}{640Bb},\mathcal{G},z^m\right).
\end{align*} 
Thus,
\begin{align*}
    &\mathbb{P} \left\{\exists (f,g)\in \mathcal{F}\times \mathcal{G}:\E[A_{f,g}]-\frac{1}{N }\sum_{i} k_i - \frac{1}{N }\sum_{j} p_j\ge \epsilon \left(\alpha+\beta+\E[A_{f,g}]\right)   \right\}
    \\&\le  \frac{32}{7}\max_{z^N}\mathcal{N}_1\left(\dfrac{\epsilon\beta}{40B},\mathcal{F},x^N\right)\cdot \mathcal{N}_1\left(\dfrac{\epsilon\beta}{40(\lambda+1) B},\mathcal{G},z^N\right) \cdot \exp\left(
 -\frac{N\epsilon^{2}(1-\epsilon)\alpha}{512(\lambda+1)B^{2}}\right)\\&+\frac{64}{7}\max_{z^n}\mathcal{N}_1 \left(\dfrac{\epsilon(\alpha+\beta)}{640Ba},\mathcal{F},x^n\right)\cdot\mathcal{N}_1 \left(\dfrac{\epsilon(\alpha+\beta)}{640B(\lambda+1)a},\mathcal{G},z^n\right) \cdot \exp \left(-\dfrac{\epsilon^2(\alpha+\beta)N}{480B^2(\lambda+1)}\right) \\ &+ \frac{64}{7}\max_{z^m}\mathcal{N}_1 \left(\dfrac{\epsilon(\alpha+\beta)}{640Bb},\mathcal{F},x^m\right)\cdot\mathcal{N}_1 \left(\dfrac{\epsilon(\alpha+\beta)}{640Bb},\mathcal{G},z^m\right) \cdot \exp \left(-\dfrac{\epsilon^2(\alpha+\beta)N}{480B^2(\lambda+1)}\right)\\
    &\le10\mathfrak M_{\mathcal{F},\mathcal{G}}^ {n,m,\lambda}\exp \left(-\dfrac{\epsilon^2(1-\epsilon)\alpha N}{512B^2(\lambda+1)}\right).
\end{align*}
This completes the proof of Theorem~\ref{thm:modified_thm11.4}.
\end{proof}
\raggedbottom
\subsection{Proof of Theorem~\ref{thm:rate-fstar}}
Write
\[
D(f,g)\coloneqq\mathcal L(f,g;\lambda)-\mathcal L(f^\star,g^\star;\lambda),
\qquad
\widehat D(f,g)\coloneqq
\hat{\mathcal L}(f,g;\lambda)-\hat{\mathcal L}(f^\star,g^\star;\lambda).
\]
In Theorem~\ref{thm:modified_thm11.4}, set
$\epsilon=1/2$, $\alpha=\beta=t/2$, and $C_0=1/(8192B^2)$.
Since
\[
D-\widehat D\ge\tfrac12(t+D)
\quad\Longleftrightarrow\quad
D-2\widehat D\ge t,
\]
that theorem gives
\begin{align*}
&\mathbb P\left\{\exists(f,g)\in\mathcal F\times\mathcal G:
D(f,g)-2\widehat D(f,g)\ge t\right\}\\
&\quad\le 10\bigl(M_n(t)+M_m(t)+M_N(t)\bigr)
\exp\left(-\frac{C_0Nt}{\lambda+1}\right),
\end{align*}
where
\begin{align*}
M_n(t)&\coloneqq\max_{z^n}
\mathcal N_1\left(\frac{t}{1280Ba},\mathcal F,x^n\right)
\mathcal N_1\left(\frac{t}{1280B(\lambda+1)a},\mathcal G,z^n\right),\\
M_m(t)&\coloneqq\max_{z^m}
\mathcal N_1\left(\frac{t}{1280Bb},\mathcal F,x^m\right)
\mathcal N_1\left(\frac{t}{1280Bb},\mathcal G,z^m\right),\\
M_N(t)&\coloneqq\max_{z^N}
\mathcal N_1\left(\frac{t}{160B},\mathcal F,x^N\right)
\mathcal N_1\left(\frac{t}{160B(\lambda+1)},\mathcal G,z^N\right).
\end{align*}

Let
\[
M\coloneqq\sup_{z^N}\mathcal N_{1,\infty}
\left(\epsilon_N,\mathcal F\times\mathcal G,z^N\right),
\qquad E_0\coloneqq\log M,
\]
so that $\mathfrak E_{\mathcal F\times\mathcal G}^{N,\lambda}=1+E_0$.
Restricting a full-sample $\epsilon_N$-cover to the labeled and unlabeled
blocks gives covers of radii $\epsilon_N/a$ and $\epsilon_N/b$,
respectively. For $t\ge1/N$, the radii defining $M_n(t)$, $M_m(t)$,
and $M_N(t)$ are at least the corresponding restricted-cover radii.
Projecting a joint max-norm cover onto its two coordinates therefore gives
\[
M_n(t)\le M^2,
\qquad M_m(t)\le M^2,
\qquad M_N(t)\le M^2.
\]
It follows that, for the data-dependent pair $(\hat f,\hat g)$ and $t\ge1/N$,
\[
\mathbb P\left\{D(\hat f,\hat g)-2\widehat D(\hat f,\hat g)\ge t\right\}
\le30M^2\exp\left(-\frac{C_0Nt}{\lambda+1}\right).
\]

Set
\[
\delta\coloneqq\frac{\lambda+1}{C_0N}\bigl(\log30+2E_0\bigr).
\]
Because $B\ge1$, we have $\delta\ge1/N$. Integrating the preceding tail bound yields
\begin{align*}
&\E_{\cD}\left[D(\hat f,\hat g)-2\widehat D(\hat f,\hat g)\right]\\
&\quad\le\delta+30M^2\int_\delta^\infty
\exp\left(-\frac{C_0Nt}{\lambda+1}\right)dt\\
&\quad=\frac{\lambda+1}{C_0N}\bigl(1+\log30+2E_0\bigr)\\
&\quad\le
\frac{C_1B^2(\lambda+1)\mathfrak E_{\mathcal F\times\mathcal G}^{N,\lambda}}{N},
\end{align*}
for a universal constant $C_1>0$.
Since $\widehat D(\hat f,\hat g)=\hat\Delta(\hat f,\hat g)$, we obtain
\begin{align}
&\E_{\cD}[\mathcal L(\hat f,\hat g;\lambda)]-\mathcal L(f^\star,g^\star;\lambda)\nonumber\\
&\quad\le
\frac{C_1B^2(\lambda+1)\mathfrak E_{\mathcal F\times\mathcal G}^{N,\lambda}}{N}
+2\E_{\cD}[\hat\Delta(\hat f,\hat g)].
\label{eq:upper-bound-star}
\end{align}
Combining Corollary~\ref{cor:gamma-star} and \eqref{eq:upper-bound-star} completes the proof.

\subsection{Proof of Remark \ref{rem: correlation score} (ii)} \label{proof: rem: correlation score}

Note that for any square-integrable random functions $k_1$ and $k_2$, we have \[
    1-\frac{\E[\ip{k_1}{k_2}    ]^2 }{\E[\|k_1\|^2]\E[\|k_2\|^2]}=\frac{\E[\|k_1-ck_2\|^2]}{\E[\|k_1\|^2]},
    \] where $$c=\frac{\E[\ip{k_1}{k_2}]}{\E[\|k_2\|^2]}.$$
Applying this identity with
$k_1 = g^\star(Z)-\hat g(Z)$
and
$k_2 = f^\star(X)-\hat f(X)$,
we obtain
\begin{align*}
1-\rho_\star^2
&=
\frac{\E_{\cD,Z}\left[\big(g^\star(Z)-\hat g(Z)-c(f^\star(X)-\hat f(X))\big)^2\right]}
{\E_{\cD,Z}\left[\big(g^\star(Z)-\hat g(Z)\big)^2\right]}
\\
&\ge
\frac{\E_{\cD,X}\left[\Var\big(g^\star(Z)-\hat g(Z)-c(f^\star(X)-\hat f(X))\mid \cD,X\big)\right]}
{\E_{\cD,Z}\left[\big(g^\star(Z)-\hat g(Z)\big)^2\right]}
\\
&=
\frac{\E_{\cD,X}\left[\Var\big(g^\star(Z)-\hat g(Z)\mid \cD,X\big)\right]}
{\E_{\cD,Z}\left[\big(g^\star(Z)-\hat g(Z)\big)^2\right]},
\end{align*}

where the inequality follows from the law of total variance.

Rearranging yields
\[
\rho_\star^2
\le
\frac{\E_{\cD,X}\left[\Big(\E\big[g^\star(Z)-\hat g(Z)\mid \cD,\,X\big]\Big)^2\right]}
{\E_{\cD,Z}\left[\big(g^\star(Z)-\hat g(Z)\big)^2\right]}.
\]
\subsection{Proof of Theorem~\ref{thm:oga-main-1over5}}

For comparison element $h = (f, g)$ with atomic norms.
\[
C_f \coloneq  \|f\|_{L^1(\mathcal{D}_f)}, \quad
C_g \coloneq  \|g\|_{L^1(\mathcal{D}_g)}, \quad
C_h \coloneq  C_f + C_g.
\]
Define the excess energy.
\begin{equation*}
a_k \coloneq  \|r_k\|_\star^2 - \|(Y,Y)-h\|_\star^2
\end{equation*}
and the norms
\begin{equation*}
    \alpha_k=\|\Pi_{S_k^f} r_k\|_\star\,,\quad\beta_k=\|\Pi_{S_k^g} r_k^{(g)}\|_\star.
\end{equation*}
By orthogonal projection properties,
\begin{equation*}\label{eq:residual-recursion}
\|r_{k+1}\|_\star^2 = \|r_k^{(g)}\|_\star^2 - \beta_k^2 = \|r_k\|_\star^2 - \alpha_k^2 - \beta_k^2 = \|r_k\|_\star^2-c^2_k,
\end{equation*}
where $c^2_k = \alpha_k^2 + \beta_k^2 = a_k - a_{k+1}$. To prove Theorem~\ref{thm:oga-main-1over5}, we first prove the following lemmas.

\begin{lemma}\label{lem:cross-talk}
For every $k\ge1$,
\begin{equation*}
A_k = \sup_{a \in \mathcal{D}_f^\star} \ip{r_k}{a}_\star \leq R_\lambda\alpha_k,
\qquad
B_k = \sup_{b \in \mathcal{D}_g^\star} \ip{r_k}{b}_\star \leq R_\lambda\alpha_k+R_\lambda\beta_k.
\end{equation*}
\end{lemma}

\begin{proof}
Write $r_k=r_k^{(g)}+\Pi_{S_k^f}r_k$ with $r_k^{(g)}\perp S_k^f$ by Step~1. For any $b\in\mathcal{D}_g^\star$,
\[
\langle r_k,b\rangle_\star
= \langle r_k^{(g)},b\rangle_\star + \langle \Pi_{S_k^f}r_k,b\rangle_\star
\leq \sup_{b\in\mathcal{D}_g^\star}\langle r_k^{(g)},b\rangle_\star + R_\lambda\|\Pi_{S_k^f}r_k\|_\star,
\]
where the last inequality uses Cauchy--Schwarz and $\|b\|_\star\le R_\lambda$. For $a\in\mathcal{D}_f^\star$, by definition of $\psi_k^\star$ we have
\[
\sup_{a\in\mathcal{D}_f^\star}\langle r_k,a\rangle_\star\leq R_\lambda\sup_{a\in\mathcal{D}_f^\star}\|\Pi_{S_{k-1}^f\cup\{a\}}r_k\|_\star=R_\lambda\|\Pi_{S_{k}^f} r_k\|_\star =R_\lambda\alpha_k.
\]
 Likewise,
\[
\sup_{b\in\mathcal{D}_g^\star}\langle r_k^{(g)},b\rangle_\star
\leq R_\lambda\|\Pi_{S_k^g}r_k^{(g)}\|_\star=R_\lambda\beta_k,
\]
hence $\sup_{b}\langle r_k,b\rangle_\star\leq R_\lambda\beta_k+R_\lambda\alpha_k$.  This gives the second bound.
\end{proof}

\begin{lemma}\label{lem:A-identity}
We have
\begin{equation}\label{eq:A}
\big\langle r_k,(Y,Y)\big\rangle_\star
=\|r_k\|_\star^2+\big\langle \Pi_{S_k^f}r_k,(f_{k-1},0)\big\rangle_\star, \quad (f_k,0)=(f_{k-1},0)+\Pi_{S_k^f}r_k,
\end{equation}
and, by polarization of the update,
\begin{equation}\label{eq:A-polar}
\big\langle \Pi_{S_k^f}r_k,(f_{k-1},0)\big\rangle_\star
=\frac{1}{2}\left(\big\|(f_k,0)\big\|_\star^2-\big\|(f_{k-1},0)\big\|_\star^2-\big\|\Pi_{S_k^f}r_k\big\|_\star^2\right).
\end{equation}
Equivalently,
\begin{equation}\label{eq:A-combined}
\big\langle r_k,(Y,Y)\big\rangle_\star
=\|r_k\|_\star^2+\frac{1}{2}\left(\big\|(f_k,0)\big\|_\star^2-\big\|(f_{k-1},0)\big\|_\star^2-\big\|\Pi_{S_k^f}r_k\big\|_\star^2\right).
\end{equation}
\end{lemma}

\begin{proof}
By definition $r_k=(Y,Y)-(f_{k-1},g_{k-1})$, and the Step~1 update is the orthogonal projection in the $\star$–inner product, hence
\[
(f_k,g_{k-1})=(f_{k-1},g_{k-1})+\Pi_{S_k^f}r_k,
\quad\text{so}\quad
(f_k,0)=(f_{k-1},0)+\Pi_{S_k^f}r_k .
\]
For \eqref{eq:A}, write
\[
\big\langle r_k,(Y,Y)\big\rangle_\star
=\big\langle r_k,(f_{k-1},g_{k-1})+r_k\big\rangle_\star
=\|r_k\|_\star^2+\big\langle r_k,(f_{k-1},g_{k-1})\big\rangle_\star.
\]
Since $(f_{k-1},0)\in \text{span}(S_k^f)$ and $\Pi_{S_k^f}r_k$ is the orthogonal projection onto $S_k^f$,
\[
\big\langle r_k,(f_{k-1},0)\big\rangle_\star
=\big\langle \Pi_{S_k^f}r_k,(f_{k-1},0)\big\rangle_\star.
\]
Moreover the previous step ends with $r_k\perp S_{k-1}^g$ and $(0,g_{k-1})\in S_{k-1}^g$, hence
$\big\langle r_k,(0,g_{k-1})\big\rangle_\star=0$, proving \eqref{eq:A}.

For \eqref{eq:A-polar}, expand the square using $(f_k,0)=(f_{k-1},0)+\Pi_{S_k^f}r_k$.
\[
\big\|(f_k,0)\big\|_\star^2
=\big\|(f_{k-1},0)\big\|_\star^2
+2\big\langle \Pi_{S_k^f}r_k,(f_{k-1},0)\big\rangle_\star
+\big\|\Pi_{S_k^f}r_k\big\|_\star^2,
\]
and rearrange to obtain \eqref{eq:A-polar}. Substituting \eqref{eq:A-polar} into \eqref{eq:A} yields \eqref{eq:A-combined}.
\end{proof}

\begin{lemma}\label{lem:Two-Stage}
For every $k\ge1$,
\[
\sum_{i=2}^{k+1} a_i \leq 2\sqrt{2}C_h R_\lambda\sum_{i=1}^k c_i,
\qquad C_h= C_f+C_g,\quad c_i = \sqrt{a_i - a_{i+1}}.
\]
\end{lemma}

\begin{proof}
Let $e\coloneq (Y,Y)-h$, $r_k\coloneq (Y,Y)-(f_{k-1},g_{k-1})$, and recall
\[
a_k\coloneq \|r_k\|_\star^2-\|e\|_\star^2,\qquad
c_k^2 = a_k-a_{k+1}=\alpha_k^2+\beta_k^2 .
\]

By Lemma \ref{lem:A-identity},
$$
\|r_k\|_\star^2 = \big\langle r_k,(Y,Y)\big\rangle_\star - \big\langle \Pi_{S_k^f}r_k,(f_{k-1},0)\big\rangle_\star.
$$
Then, by arithmetic--geometric mean inequality
\[
\begin{aligned}
a_k&=\|r_k\|_\star^2-\|e\|_\star^2
=\big\langle r_k,e\big\rangle_\star+\big\langle r_k,(f,0)\big\rangle_\star+\big\langle r_k,(0,g)\big\rangle_\star-\big\langle \Pi_{S_k^f}r_k,(f_{k-1},0)\big\rangle_\star-\|e\|_\star^2\\
&\leq \frac{1}{2} a_k + \big\langle r_k,(f,0)\big\rangle_\star - \big\langle \Pi_{S_k^f}r_k,(f_{k-1},0)\big\rangle_\star + \big\langle r_k,(0,g)\big\rangle_\star,
\end{aligned}
\]
and we use the bounds $\big\langle r_k,(f,0)\big\rangle_\star \leq C_f A_k$ and $\big\langle r_k,(0,g)\big\rangle_\star \leq C_g B_k$ to control the two correlation terms, 
hence
\[
\frac{a_k}{2}\leq C_fA_k + C_g B_k - \big\langle \Pi_{S_k^f}r_k,(f_{k-1},0)\big\rangle_\star.
\]

Apply the polarization identity (Lemma~\ref{lem:A-identity}) to the $f$-history:
\[
\langle \Pi_{S_k^f}r_k,(f_{k-1},0)\rangle_\star
=\frac{1}{2}\left(\|(f_k,0)\|_\star^2-\|(f_{k-1},0)\|_\star^2-\alpha_k^2\right).
\]
Therefore,
\[
\frac{a_k}{2}
\leq C_fA_k + C_g B_k
+\frac12\left(\|(f_{k-1},0)\|_\star^2-\|(f_k,0)\|_\star^2\right)
+\frac12\alpha_k^2 .
\]

Summing over $i=1,\dots,k$ (with $f_0=0$) yields
\begin{equation}\label{eq:sum-f-only}
\sum_{i=1}^k \frac{a_i}{2}
\leq C_f\sum_{i=1}^k A_i + C_g\sum_{i=1}^k B_i
-\frac12\|(f_k,0)\|_\star^2 + \frac12\sum_{i=1}^k \alpha_i^2 .
\end{equation}
Since $\sum_{i=1}^k \alpha_i^2 \leq \sum_{i=1}^k c_i^2
=\sum_{i=1}^k (a_i-a_{i+1}) = a_1-a_{k+1}$, $A_i \leq R_\lambda\alpha_i $, and  $B_i\leq R_\lambda\alpha_i+R_\lambda\beta_i\leq \sqrt{2}R_\lambda c_i$, we get
\[
C_f\sum_{i=1}^k A_i + C_g\sum_{i=1}^k B_i
\leq C_fR_\lambda\sum_{i=1}^k c_i + \sqrt{2}C_gR_\lambda\sum_{i=1}^k c_i
\leq \sqrt{2}(C_f+C_g)R_\lambda\sum_{i=1}^k c_i = \sqrt{2}C_hR_\lambda\sum_{i=1}^k c_i .
\]
Dropping the nonpositive term $-\frac{1}{2}\|(f_k,0)\|_\star^2$ in \eqref{eq:sum-f-only} gives
\[
\sum_{i=1}^k \frac{a_i}{2}
\leq \sqrt{2}C_hR_\lambda\sum_{i=1}^k c_i + \frac{a_1-a_{k+1}}{2},
\]
which implies
\[
\sum_{i=2}^{k+1} a_i \leq 2\sqrt{2}C_hR_\lambda\sum_{i=1}^k c_i.
\]
\end{proof}

We are ready to prove Theorem~\ref{thm:oga-main-1over5}.
\begin{proof}
Set
\[
S_k\coloneq \sum_{i=1}^k a_i,\qquad
H_k\coloneq \sum_{i=1}^k \frac{1}{i}.
\]
We have that for every $k\ge1$,
\begin{equation}\label{eq:oga-hyp}
\sum_{i=2}^{k+1} a_i \leq 2\sqrt{2}C_hR_\lambda \sum_{i=1}^k c_i,
\qquad c_i = \sqrt{a_i - a_{i+1}}.
\end{equation}

\medskip\noindent\textit{Summation by parts (discrete Abel transform).}
For every $k\ge1$,
\begin{equation}\label{eq:sbp}
\sum_{i=1}^k i(a_i-a_{i+1})
= \sum_{i=1}^k ic_i^2
= S_k - ka_{k+1}.
\end{equation}
Indeed,
\[
\sum_{i=1}^k i(a_i-a_{i+1})
=\sum_{i=1}^k\sum_{j=1}^i (a_i-a_{i+1})
=\sum_{j=1}^k\sum_{i=j}^k (a_i-a_{i+1})
=\sum_{j=1}^k (a_j-a_{k+1})
=S_k-ka_{k+1}.
\]

\medskip\noindent\textit{Cauchy--Schwarz with weights and \eqref{eq:sbp}.}
By Cauchy--Schwarz,
\[
\Big(\sum_{i=1}^k c_i\Big)^2
=\Big(\sum_{i=1}^k \frac{1}{\sqrt{i}}\sqrt{i}c_i\Big)^2
\leq \Big(\sum_{i=1}^k \frac{1}{i}\Big)\Big(\sum_{i=1}^k ic_i^2\Big)
= H_k\big(S_k-ka_{k+1}\big)
\leq H_kS_k.
\]
Using \eqref{eq:oga-hyp} and $S_k\leq S_{k+1}$,
\[
S_{k+1}-a_1
=\sum_{i=2}^{k+1} a_i
\leq 2\sqrt{2}C_h R_\lambda\sum_{i=1}^k c_i
\leq 2\sqrt{2}C_h R_\lambda\sqrt{H_kS_{k+1}}.
\]
Then
\[
S_{k+1} \leq a_1 + 2\sqrt{2}C_hR_\lambda\sqrt{H_k S_{k+1}}
\leq a_1 + S_{k+1}/2 + 4 C^2_hR_\lambda^2 H_k,
\]
so $S_{k+1} \leq 2a_1 + 8 C^2_hR_\lambda^2 H_k$. Hence for all $k\ge1$,
\begin{equation}\label{eq:Sk-bound}
S_k \leq 2a_1 + 8C_h^2R_\lambda^2H_k.
\end{equation}

\medskip\noindent\textit{Pointwise rate.}
Since $(a_k)$ is decreasing, $ka_k\leq S_k$. Combining with \eqref{eq:Sk-bound},
\[
a_k \leq \frac{2a_1}{k} + \frac{8C_h^2R_\lambda^2H_k}{k}
\leq C\frac{C_h^2R_\lambda^2\log(k+1)}{k},
\]
for a universal constant $C>0$, completing the proof.
\end{proof}